\documentclass[11pt]{article}

\usepackage{amsmath, amssymb, amsfonts, amsthm}
\usepackage{graphicx}
\usepackage{geometry}
\usepackage[hidelinks]{hyperref}
\usepackage{enumitem}
\usepackage{bm}
\usepackage{algorithm}
\usepackage{algorithmic}
\usepackage{booktabs}

\usepackage{tikz}
\usetikzlibrary{arrows.meta,positioning,calc,shapes}

\usepackage{pifont}
\newcommand{\xmark}{\ding{55}}%

\usepackage{fancyhdr}
\title{\textbf{VJEPA: Variational Joint Embedding Predictive Architectures as Probabilistic World Models}}
\author{Yongchao Huang \footnote{Author email: yongchao.huang@abdn.ac.uk}}
\date{01/01/2026}

\newtheorem{theorem}{Theorem}
\newtheorem{proposition}{Proposition}
\newtheorem{lemma}{Lemma}

\newtheorem{definition}{Definition}
\theoremstyle{remark}
\newtheorem{remark}{Remark}

\begin{document}
\maketitle

\begin{abstract}
Joint Embedding Predictive Architectures (JEPA) offer a scalable paradigm for self-supervised learning by predicting latent representations rather than reconstructing high-entropy observations. However, existing formulations rely on \textit{deterministic} regression objectives, which mask probabilistic semantics and limit its applicability in stochastic control.
In this work, we introduce \emph{Variational JEPA (VJEPA)}, a \textit{probabilistic} generalization that learns a predictive distribution over future latent states via a variational objective.
We show that VJEPA unifies representation learning with Predictive State Representations (PSRs) and Bayesian filtering, establishing that sequential modeling does not require autoregressive observation likelihoods.
Theoretically, we prove that VJEPA representations can serve as sufficient information states for optimal control without pixel reconstruction, while providing formal guarantees for collapse avoidance.
We further propose \emph{Bayesian JEPA (BJEPA)}, an extension that factorizes the predictive belief into a learned dynamics expert and a modular prior expert, enabling zero-shot task transfer and constraint (e.g. goal, physics) satisfaction via a Product of Experts.
Empirically, through a noisy environment experiment, we demonstrate that VJEPA and BJEPA successfully filter out high-variance nuisance distractors that cause representation collapse in generative baselines.
By enabling principled uncertainty estimation (e.g. constructing credible intervals via sampling) while remaining likelihood-free regarding observations, VJEPA provides a foundational framework for scalable, robust, uncertainty-aware planning in high-dimensional, noisy environments.
\end{abstract}

\tableofcontents
\newpage

\section{Introduction} \label{sec:intro}

Self-supervised representation learning has traditionally followed two dominant paradigms:
(i) generative modeling, which optimizes likelihoods over high-dimensional observations such as pixels or tokens
\cite{kingma2014vae, rezende2014stochastic, oord2016pixelrnn, raffel2020t5}, and
(ii) contrastive learning, which discriminates between positive and negative pairs
\cite{hadsell2006dimensionality, oord2018cpc, chen2020simclr, he2020moco}.
Both approaches suffer from fundamental limitations: generative models must model nuisance variability and high-entropy details irrelevant for downstream tasks
\cite{alemi2017ib, higgins2017beta, locatello2019challenging}, while contrastive methods require carefully engineered negative sampling schemes and can introduce representational bias
\cite{tschannen2019mutual, saunshi2019theoretical, chen2020simclr}.

Joint Embedding Predictive Architectures (JEPA) offer a third paradigm
\cite{lejepa2022ijepa, bardes2024vjepa}. Rather than reconstructing observations or contrasting samples, JEPA learns by predicting \emph{representations} of missing or future data \cite{lejepa2022ijepa}. This removes the need for observation-level likelihoods and explicitly biases learning toward predictable, task-relevant structure \cite{bardes2024vjepa, larsen2016autoencoding, Singh2004}.
Recent instantiations such as I-JEPA \cite{lejepa2022ijepa} and V-JEPA \cite{bardes2024vjepa,bardes2024vjepa_openReview} have demonstrated strong empirical performance on vision and video tasks. Early on, LeCun \cite{LeCun2022HJepa} proposed \emph{Hierarchical JEPA} (H-JEPA), a conceptual framework for multi-level,
multi-timescale predictive models that emphasizes learning latent dynamics without observation reconstruction as a basis for model-predictive control and long-horizon reasoning.

More recently, several works have begun to instantiate JEPA-style predictors as \emph{world models} for planning and control. Bardes et al.~\cite{Bardes2025VJEPA2} show that large-scale JEPA pretraining on internet video, combined with lightweight
action-conditioned prediction, enables zero-shot robotic planning via latent-space model-predictive control. Terver et al.~\cite{terver2025jepaworldmodels} systematically study action-conditioned \emph{JEPA World Models} (JEPA-WMs), demonstrating that planning directly in representation space can solve navigation and
manipulation tasks without observation reconstruction and without training on reward-labeled data. Destrade et al.~\cite{Destrade2026ValueGuidedJEPA} further study \emph{value-guided JEPA planning}, which shapes the representation space so that distances approximate goal-conditioned value, and improves optimization-based planning.

Despite these advances, existing JEPA-based world models remain largely \emph{deterministic} and are trained using regression-based objectives \cite{lejepa2022ijepa, bardes2024vjepa}. As a result, the underlying probabilistic semantics of JEPA predictions, the representation of uncertainty over future latent states, and the conditions under which the learned representation constitutes a sufficient information state for planning and control remain unformalized. In particular, current approaches implicitly assume predictive sufficiency, but do not characterize when or why this assumption should hold.

In this work, we ask and answer several foundational questions about JEPA:
\begin{itemize}
    \item \textit{Implicit Probabilistic Model:} what probabilistic objective does deterministic JEPA implicitly optimize, and how can it be generalized to handle uncertainty?
    \item \textit{Dynamics and Control:} can JEPA be formalized as a latent dynamical system where the learned representation serves as a \emph{sufficient information state} for optimal control, without reconstructing observations?
    \item \textit{Bayesian Connection:} how does JEPA relate to Bayesian filtering and Predictive State Representations (PSRs), and can this connection enable the injection of structural priors or constraints into latent planning?
    \item \textit{Nuisance Invariance:} does introducing temporal structure force JEPA to adopt autoregressive observation likelihoods, or can it remain likelihood-free to avoid modeling high-variance nuisance distractors?
\end{itemize}
We address these questions by formulating JEPA as a \emph{probabilistic predictive state-space model}, and showing that sequential predictive modeling does not require autoregressive likelihood factorization and that JEPA representations can serve as sufficient information states for control and planning.

\paragraph{Contributions.}
We propose \emph{Variational JEPA (VJEPA)}, a probabilistic formulation that explicates the implicit generative model of JEPA. We theoretically establish VJEPA as a predictive state-space model, proving that its representations serve as \emph{sufficient information states} for optimal control and providing formal guarantees for collapse avoidance. We further extend this to \emph{Bayesian JEPA (BJEPA)}, a modular framework that factorizes predictive belief into learned dynamics and structural priors, enabling the insertion of constraints and physical laws, as well as zero-shot task transfer via Product of Experts. Finally, we empirically demonstrate in a noisy environment that VJEPA and BJEPA successfully filter out high-variance distractors, validating that principled uncertainty-aware world modeling is achievable without autoregressive observation reconstruction.

\section{Related Work} \label{sec:related_work}

\subsection{Predictive State Representations}

Predictive State Representations (PSRs) were introduced by Sutton et al. \cite{Sutton2001} and independently formalized by Singh et al. \cite{Singh2004}. They replace latent states with sufficient statistics of future observations, representing the state as
\[
s_t = p(o_{t+1:t+H} \mid o_{\le t}),
\]
where $H$ denotes the prediction horizon. 
Sutton et al. \cite{Sutton2001} emphasized the theoretical connection between PSRs and history-based models, showing how state can be represented purely in terms of observable predictions. 
Singh et al. \cite{Singh2004} developed the system-dynamics matrix formalism and demonstrated that PSRs generalize $n$-th order Markov models and hidden-state models such as HMMs or POMDPs. 
Later, we will see that VJEPA can be viewed as a neural, amortized PSR that compresses predictive distributions into a latent embedding $Z_t$.

\subsection{World Models and Latent Dynamics}

Model-based latent dynamics have been widely explored in reinforcement learning. PlaNet \cite{hafner2018planet} introduced a probabilistic latent dynamics model that predicts future latent states $z_{t+1}$ conditioned on actions $a_t$ and reconstructs observations $x_t$, optimizing an evidence lower bound (ELBO) over pixels.
Dreamer \cite{hafner2019dreamer} extended PlaNet by integrating actor-critic reinforcement learning in the learned latent space, enabling policy optimization without decoding full trajectories to pixels. 
Formally, these methods learn
\[
p(z_{t+1} \mid z_t, a_t), \quad p(x_t \mid z_t),
\]
where the observation likelihood $p(x_t \mid z_t)$ is explicitly modeled.

In contrast, JEPA \cite{lejepa2022ijepa}, V-JEPA (video-JEPA) \cite{bardes2024vjepa} and VJEPA (this work) discard observation reconstruction, learning predictive latent representations that compress future information into $Z_t$, emphasizing predictive sufficiency rather than pixel fidelity.

MuZero \cite{Schrittwieser2020muzero} also learns latent dynamics without reconstructing observations, but it still relies on value and policy heads for planning. 
VJEPA differs by learning latent dynamics purely through representation prediction, without access to rewards, policies, or observation likelihoods, making it a self-supervised model suitable for both control and general representation learning.

\subsection{Active Inference}

Active Inference is a theoretical framework proposed by Friston and colleagues in which perception, action, and learning are cast as variational inference in a latent variable model of the world
\cite{friston2010free, friston2012perceptions, friston2017active}.
The central idea is that biological agents act to minimize \emph{variational free energy}, which provides an upper bound on sensory surprise (negative log-evidence). Formally, variational free energy is defined as
\[
\mathcal{F}(q)
=
\mathbb{E}_{q(z)} \left[\log q(z) - \log p(x, z)\right],
\]
where $z$ denotes latent (hidden) states of the environment, $x$ denotes sensory observations, and $q(z)$ is an approximate posterior belief.
Minimizing $\mathcal{F}(q)$ drives $q(z)$ towards the true Bayesian posterior $p(z \mid x)$, thereby implementing perception as approximate Bayesian inference.
This formulation was introduced and developed in early work by Friston \cite{friston2010free} and extended to perception-action loops in subsequent studies \cite{friston2012perceptions}.

Within Active Inference, belief updating corresponds to approximate Bayesian filtering over latent states, often implemented as gradient descent on free energy.
More recent work has emphasized a distinction between perception and action: while perception minimizes current free energy, action selection is driven by the minimization of \emph{expected free energy} over future trajectories, enabling agents to plan by imagining future outcomes and selecting policies that reduce expected surprise and uncertainty \cite{friston2017active}.

VJEPA aligns with Active Inference at the level of \emph{latent belief propagation}: both frameworks aim to learn compact latent representations that summarize past observations and support prediction of future states.
However, VJEPA deliberately omits modeling sensory likelihoods $p(x \mid z)$, preferences, or policy evaluation.
Instead, it focuses exclusively on learning predictive latent dynamics through representation prediction, yielding a self-supervised model of belief dynamics rather than a full Active Inference agent.

\subsection{Joint Embedding Predictive Architectures (JEPAs)}

JEPA was introduced to self-supervised learning by Assran et al. \cite{lejepa2022ijepa} in the Image-based JEPA (I-JEPA) model, where representations of held-out image regions are predicted from encoded context patches. I-JEPA demonstrated that prediction in a latent representation space can produce highly semantic features without pixel reconstruction or handcrafted augmentations, and scales to large vision models with competitive performance \cite{lejepa2022ijepa}.

Subsequent work extended the JEPA paradigm to video. Bardes et al. proposed Video JEPA (V-JEPA) for self-supervised learning from video by predicting masked spatio-temporal regions in a learned latent space \cite{bardes2024vjepa}. While V-JEPA improves temporal representation learning relative to image-only JEPA, it retains the deterministic regression objective.

Beyond these instantiations, the JEPA framework has inspired a variety of architectural and domain adaptations. For example, MC-JEPA \cite{bardes2023mcJEPA} incorporates motion and content prediction objectives for richer video representation \cite{bardes2023mcJEPA}, S-JEPA applies the JEPA principle to skeletal action recognition \cite{abdelfattah2024sjepa}, and Point-JEPA adapts JEPA to point cloud self-supervised learning \cite{saito2025pointjepa}. Other variants explore JEPA in graph domains \cite{skenderi2025graphJEPA} and in joint vision-language models \cite{chen2025vlJEPA}. Additional work (e.g. C-JEPA \cite{mo2024cJEPA}, DSeq-JEPA \cite{he2025dseqJEPA}) proposes contrastive or sequential extensions that address specific limitations such as representation collapse or prediction ordering.

Despite this growing body of JEPA-based methods and applications, to our knowledge there has been no formal variational or Bayesian JEPA that explicitly models predictive uncertainty and the underlying probabilistic generative process. Some work incorporates auxiliary losses to stabilize training (e.g. VICReg-style regularization \cite{bardes2022VICReg} in C-JEPA \cite{mo2024cJEPA}) or implicitly addresses uncertainty through architectural choices, but these do not derive from a principled likelihood-based objective. In contrast, our VJEPA formulates JEPA as a variational predictive model, learning an explicit predictive distribution over future latent states with likelihood and KL regularization, enabling a formal probabilistic interpretation and control-theoretic extensions not addressed by prior JEPA work.

\subsection{JEPA World Models and Latent-Space Planning}

The use of JEPA as a foundation for world modeling and planning was first articulated at a conceptual level by LeCun~\cite{LeCun2022HJepa}. In his proposal of \emph{Hierarchical JEPA} (H-JEPA), LeCun argues that autonomous agents should learn multi-level, multi-timescale predictive representations of the world that support planning and model-predictive control without relying on observation reconstruction or explicit reward modeling. H-JEPA provides a normative blueprint for predictive world modeling and hierarchical control, but does not specify a concrete training objective, probabilistic semantics, or learning algorithm.

Building on this vision, more recent work has developed concrete, action-conditioned JEPA-based world models that enable planning directly in representation space. Terver et al.~\cite{terver2025jepaworldmodels} introduce \emph{JEPA World Models} (JEPA-WMs), which combine pretrained visual encoders with action-conditioned JEPA predictors trained via a deterministic representation prediction loss. Planning is performed by sampling or optimizing action sequences whose predicted latent trajectories minimize a goal-conditioned cost in embedding space. Their results demonstrate that latent-space planning without observation reconstruction can match or outperform reconstruction-based world models, particularly when reward signals are unavailable or sparse.

Several subsequent works explore extensions of JEPA-based planning. Destrade et al.~\cite{Destrade2026ValueGuidedJEPA} propose \emph{value-guided JEPA planning}, in which the latent representation space is shaped so that distances approximate goal-conditioned value functions, improving the optimization landscape for planning.
More recently, Bardes et al.~\cite{Bardes2025VJEPA2} demonstrate that large-scale JEPA pretraining on internet video, followed by lightweight action-conditioned prediction, enables zero-shot robotic manipulation via latent-space model-predictive control.

While these works provide strong empirical evidence that JEPA representations can support planning and control, they are uniformly based on \emph{deterministic} predictors trained via regression objectives. As a result, uncertainty over future latent states, belief propagation, and the conditions under which the learned representation constitutes a sufficient information state for optimal control remain implicit and unformalized. In particular, these approaches assume predictive sufficiency of the latent state without providing a probabilistic model or formal guarantees.

Our work complements JEPA-based world models by providing a \emph{variational} formulation of JEPA that makes the underlying predictive distribution explicit. By modeling uncertainty over future latent states, VJEPA enables belief propagation, distributional planning, and formal connections to predictive state representations and Bayesian filtering, which are capabilities that not addressed by existing deterministic JEPA world models.

\vspace{0.6cm}
Taken together, the above lines of work leave open a gap that is not fully addressed by existing approaches.
Predictive State Representations provide a principled account of state as predictive statistics, but are not designed as scalable, amortized representation learners for high-dimensional sensory data.
Latent world models such as PlaNet, Dreamer, and MuZero learn predictive dynamics suitable for control, but rely on observation reconstruction, reward supervision, or policy optimization objectives that entangle representation learning with task-specific signals.
Active Inference offers a unifying theoretical framework for perception and action, but requires explicit modeling of sensory likelihoods, preferences, and policies, and is not typically instantiated as a self-supervised representation learning algorithm.
Finally, existing JEPA formulations learn predictive representations using deterministic regression objectives, which obscure the underlying probabilistic structure and preclude explicit modeling of uncertainty over future latent states.

VJEPA occupies a complementary point in this design space by extending JEPA with an explicit probabilistic formulation.
It learns predictive latent states as amortized representations while modeling a full predictive distribution over future embeddings, rather than a single point estimate.
This enables principled uncertainty estimation, belief propagation, and compatibility with Bayesian filtering and control, while retaining JEPA’s core advantages: avoiding observation reconstruction, reward supervision, and autoregressive likelihood factorization.
By focusing exclusively on predictive sufficiency through variational representation prediction, VJEPA bridges predictive state representations, latent world models, and belief propagation frameworks, while remaining a purely self-supervised model that is directly applicable to downstream control and planning.

\section{JEPA: Background and Deterministic Formulation} \label{sec:jepa}

Joint Embedding Predictive Architectures (JEPA) were introduced as a self-supervised learning framework that learns representations by predicting embeddings of missing or future data, rather than reconstructing observations or discriminating between positive and negative samples \cite{lejepa2022ijepa, bardes2024vjepa}. The central design principle of JEPA is to encourage representations that capture \emph{predictable, task-relevant structure} while discarding nuisance variability that is difficult or unnecessary to predict.

Given an input $x$ (e.g. an image or a video clip), JEPA first partitions the input into two disjoint subsets: a \emph{context} $x_C$ and a \emph{target} $x_T$. The context contains the information available to the model, while the target represents missing or held-out content that must be predicted at the representation level. This partitioning can correspond to spatial masking (as in I-JEPA \cite{lejepa2022ijepa}) or temporal masking (as in video JEPA \cite{bardes2024vjepa}), but the formulation itself is agnostic to the specific modality. 
The context is mapped to a latent representation using a trainable encoder
\begin{equation} \label{eq:context_encoder}
    Z_C = f_\theta(x_C),
\end{equation}
which produces a compact embedding summarizing the available information.
The target is encoded using a separate encoder
\begin{equation} \label{eq:target_encoder}
Z_T = f_{\theta'}(x_T),
\end{equation}
whose parameters $\theta'$ are updated as an exponential moving average (EMA) of the context encoder parameters\cite{mnih2015human,mnih2016asynchronousmethodsdeepreinforcement,lillicrap2019continuouscontroldeepreinforcement,grill2020byol}:
\begin{equation} \label{eq:target_encoder_weights_EMA}
\theta' \leftarrow \tau \theta' + (1-\tau)\theta.
\end{equation}
This asymmetric design, inspired by prior work \cite{grill2020byol}, stabilizes training and prevents representational collapse by ensuring that the target embeddings evolve more slowly than the predictor.

To predict the target representation, JEPA employs a predictor network
\begin{equation} \label{eq:context_predictor}
\hat{Z}_T = g_\phi(Z_C, \xi_T),
\end{equation}
where $\xi_T$ encodes structural (side) information about the target, such as spatial location or temporal index. The predictor is tasked with mapping the context representation to the target embedding without direct access to the target observations themselves.

With these settings, training is then performed by minimizing a regression loss between the predicted and target embeddings:
\begin{equation} \label{eq:JEPA_squared_loss}
\begin{aligned}
\mathcal{L}_{\text{JEPA}}
&= \left\| \hat{Z}_T - Z_T \right\|^2 \\
&= \left\| g_\phi \left(Z_C, \xi_T \right) - Z_T \right\|^2 \\
&= \left\| g_\phi \left(f_\theta(x_C), \xi_T \right)
      - f_{\theta'}(x_T) \right\|^2 \\
&= \left\| g_\phi \left(f_\theta(x_C), \xi_T \right)
      - f_{\tau \theta' + (1-\tau)\theta}(x_T) \right\|^2 ,
\end{aligned}
\end{equation}
where the target encoder parameters are updated as per Eq.~\eqref{eq:target_encoder_weights_EMA}. Unlike generative models, this objective does not require reconstructing observations or modeling pixel-level likelihoods; instead, it enforces consistency between representations, encouraging the latent space to organize itself around predictable relationships in the data.

Although the JEPA loss is typically presented as a mean-squared error, it admits a natural probabilistic interpretation. Specifically, minimizing $\mathcal{L}_{\text{JEPA}}$ is equivalent to maximizing the log-likelihood of $Z_T$ under an isotropic Gaussian distribution centered at $\hat{Z}_T$ with fixed variance (see Appendix.\ref{app:lsq_equivalence_with_MLE} a brief derivation):
\begin{equation} \label{eq:JEPA_loss_prob_explain}
\mathcal{L}_{\text{JEPA}}
 \propto 
- \log p(Z_T \mid Z_C),
\quad
p(Z_T \mid Z_C)
=
\mathcal{N} \left(
Z_T  \middle| 
g_\phi(Z_C, \xi_T),  \sigma^2 I
\right),
\end{equation}
This implicit assumption of a simple predictive distribution motivates the variational extension of JEPA introduced in the next section, in which predictive uncertainty is modeled explicitly.

\section{VJEPA: Variational Formulation of JEPA}
\label{sec:vjepa}

Eq.~\eqref{eq:JEPA_squared_loss} shows that the deterministic JEPA objective is equivalent to maximum likelihood under an isotropic Gaussian predictive model with fixed variance (Eq.~\eqref{eq:JEPA_loss_prob_explain}). This suggests a natural generalization: instead of predicting a single target embedding $\hat Z_T$ (i.e. a point estimate), we can model an explicit \emph{predictive distribution} over possible target embeddings, enabling uncertainty-aware prediction and multi-modal futures. Importantly, as in deterministic JEPA, prediction is conditioned not only on the context representation but also on structural (side) information describing the target.

\paragraph{From deterministic JEPA to VJEPA.}
VJEPA can be understood as a minimal probabilistic extension of deterministic JEPA. Rather than altering the architectural principles of JEPA, VJEPA replaces point prediction in representation space with distributional prediction while preserving the same context-target structure, conditioning mechanism, and training asymmetry. As a preview, we summarise the key differences in Table.\ref{tab:jepa_vs_vjepa}; these features will be progressively introduced in later sections.

\begin{table}[H]
\centering
\footnotesize
\begin{tabular}{lll}
\hline
\textbf{Feature} & \textbf{JEPA} & \textbf{VJEPA} \\
\hline
Target prediction
&
$\hat Z_T = g_\phi(Z_C,\xi_T)$
&
$p_\phi(Z_T \mid Z_C,\xi_T)$
\\[0.4em]

Prediction type
&
Point estimate
&
Predictive distribution
\\[0.4em]

Training objective
&
MSE (equally Gaussian likelihood)
&
Negative log-likelihood $+$ KL
\\[0.4em]

Uncertainty modeling
&
Implicit (fixed variance)
&
Explicit (learned variance / distribution)
\\[0.4em]

Multi-modal futures
&
Not supported
&
Supported
\\[0.4em]

Collapse avoidance
&
Architectural (EMA, asymmetry)
&
Objective-level (predictive mismatch + KL reg)
\\[0.4em]

Probabilistic interpretation
&
Implicit Gaussian
&
Explicit variational model
\\[0.4em]

Compatibility with filtering / control
&
Limited
&
Natural (belief propagation)
\\
\hline
\end{tabular}
\caption{Comparison between deterministic JEPA and its variational extension VJEPA. VJEPA preserves the structural design of JEPA while endowing it with an explicit probabilistic semantics.}
\label{tab:jepa_vs_vjepa}
\end{table}

\subsection{Probabilistic Predictive Model}

Let $x$ be an input sample and $(x_C, x_T)$ be a context/target partition.
As in JEPA, the context is encoded as deterministic $Z_C = f_\theta(x_C)$, see Eq.~\eqref{eq:context_encoder}. In addition, let $\xi_T$ denote side information specifying the structure of the target (e.g. spatial location, temporal index, or masking pattern, etc), as used in Eq.~\eqref{eq:context_predictor}.
VJEPA introduces a probabilistic predictive model over the target representation:
\begin{equation}
\label{eq:vjepa_predictive_model}
p_\phi(Z_T \mid Z_C, \xi_T),
\end{equation}
which generalizes the deterministic predictor $g_\phi(Z_C,\xi_T)$ by allowing the predicted target representation to be stochastic, e.g. $p_\phi(Z_T \mid Z_C, \xi_T)=\mathcal{N}\!\big(\mu_\phi(Z_C, \xi_T),\,\Sigma_\phi(Z_C)\big)$; if $p_\phi(Z_T \mid Z_C, \xi_T)=g_\phi(Z_C,\xi_T)$ is chosen, VJEPA reduces to the deterministic JEPA. The probabilistic, predictive model Eq.~\eqref{eq:vjepa_predictive_model} can be further decomposed, e.g. formulated as a Bayesian posterior (see later Section.\ref{sec:bjepa}).

\textit{Optionally}, one may also posit an observation model
\begin{equation}
\label{eq:vjepa_obs_model_optional}
p_\psi(x_T \mid Z_T),
\end{equation}
yielding the joint factorization
\begin{equation}
\label{eq:vjepa_joint}
p_{\phi,\psi}(Z_T, x_T \mid x_C, \xi_T)
=
p_\phi(Z_T \mid Z_C, \xi_T)\, p_\psi(x_T \mid Z_T).
\end{equation}
In this text, in the spirit of JEPA, we \emph{do not} optimize an observation-level reconstruction objective over $p_\psi(x_T \mid Z_T)$; the learning signal is provided solely through the predictive distribution in representation space, $p_\phi(Z_T \mid Z_C,\xi_T)$.

\subsection{Inference Model from the Target Encoder}

To train the predictive model without reconstructing $x_T$, we require a representation-space ``posterior'' for $Z_T$ given the target observation.
We define an amortized inference distribution using the (EMA) target encoder:
\begin{equation}
\label{eq:vjepa_inference_target}
q_{\theta'}(Z_T \mid x_T),
\end{equation}
which may be instantiated as a diagonal Gaussian whose mean and variance are parameterized by $f_{\theta'}(x_T)$.
The EMA update in Eq.~\eqref{eq:target_encoder_weights_EMA} ensures that the inference target evolves slowly, which stabilizes training and prevents degenerate fixed points.

\subsection{Variational Objective}

Given $(x_C, x_T, \xi_T)$, VJEPA trains the predictive model $p_\phi(Z_T\mid Z_C,\xi_T)$ to match the target-encoder distribution $q_{\theta'}(Z_T\mid x_T)$.
A simple and effective objective is the \emph{regularized negative log-likelihood}:
\begin{equation}
\label{eq:vjepa_objective}
\mathcal{L}_{\text{VJEPA}}
=
\mathbb{E}_{x} 
\mathbb{E}_{Z_T \sim q_{\theta'}(\cdot \mid x_T)}
\Big[-\log p_\phi(Z_T \mid Z_C, \xi_T)\Big]
+
\beta\,
\mathbb{E}_{x} 
\mathrm{KL} \left(q_{\theta'}(Z_T \mid x_T)\,\|\,p(Z_T)\right),
\end{equation}
where $p(Z_T)$ is a fixed prior (e.g. $\mathcal{N}(0,I)$) and $\beta>0$ controls the strength of regularization.

The first term trains a probabilistic predictor in \emph{representation space}, conditioned on both the available context and the structural specification of the target.
The second term prevents trivial or degenerate solutions by encouraging the inferred target representation distribution to remain non-collapsed and well-calibrated relative to the prior.
Together, this objective (i) introduces explicit predictive uncertainty, (ii) permits multi-modal futures through expressive $p_\phi(\cdot\mid Z_C,\xi_T)$, and (iii) yields collapse-avoidance guarantees under mild conditions (see later Section~\ref{sec:properties_of_VJEPA}).

\paragraph{Deterministic JEPA as a special case.}
If we set
\[
q_{\theta'}(Z_T\mid x_T)=\delta(Z_T-f_{\theta'}(x_T))
\]
and
\[
p_\phi(Z_T\mid Z_C,\xi_T)
=
\mathcal{N}(Z_T\mid g_\phi(Z_C,\xi_T), \sigma^2 I)
\]
with fixed $\sigma^2$, then minimizing Eq.~\eqref{eq:vjepa_objective} reduces (up to additive constants) to the squared-loss JEPA objective in Eq.~\eqref{eq:JEPA_squared_loss}.

\subsection{Schematic Overview}

The complete VJEPA architecture, illustrating the probabilistic predictor, target sampling, and the variational loss components, is schematically shown in Fig.\ref{fig:vjepa_architecture}.

\begin{figure}[H]
\centering
\begin{tikzpicture}[
    font=\small\sffamily,
    >=Latex,
    node distance=1.2cm and 1.6cm,
    box/.style={
        draw, rounded corners=2pt,
        minimum width=3.8cm, minimum height=1.05cm,
        text width=3.6cm, align=center
    },
    sbox/.style={
        draw, rounded corners=2pt,
        minimum width=4.2cm, minimum height=1.15cm,
        text width=4.0cm, align=center
    },
    op/.style={
        draw, rounded corners=2pt,
        minimum width=3.0cm, minimum height=0.95cm,
        align=center
    },
    circ/.style={
        draw, circle, minimum size=8mm, inner sep=0pt
    },
    sumcirc/.style={
        draw, circle, minimum size=7mm, inner sep=0pt
    },
    redbox/.style={
        draw=red!80!black, very thick, rounded corners=2pt,
        minimum width=4.6cm, minimum height=1.05cm,
        text width=4.4cm, align=center
    },
    bluebox/.style={
        draw=blue!80!black, very thick, rounded corners=2pt,
        minimum width=3.8cm, minimum height=1.05cm,
        text width=3.6cm, align=center
    },
    arr/.style={->, thick},
    redarr/.style={->, very thick, draw=red!80!black},
    bluearr/.style={->, very thick, draw=blue!80!black}
]

\node[circ] (xC) at (0,0) {$x_C$};
\node[circ] (xT) at (9,0) {$x_T$};

\node[below=0.25cm of xC, font=\footnotesize] {History Context};
\node[below=0.25cm of xT, font=\footnotesize] {Training Target};

\node[box, above=0.9cm of xC] (ctxenc)
{$f_\theta$\\ \footnotesize Context Encoder};

\node[box, above=0.9cm of xT] (tgtenc)
{$f_{\theta'}$\\ \footnotesize Target Encoder};

\node[circ, above=0.9cm of ctxenc] (ZC) {$Z_C$};
\node[circ, above=0.9cm of tgtenc] (qT) {$q_{\theta'}$};

\node[sbox, right=2.4cm of ZC] (pred)
{Probabilistic Predictor\\ \footnotesize $p_\phi(Z_T \mid Z_C,\xi_T)$};

\node[op, above=0.9cm of qT] (sample)
{Sample\\ \footnotesize $Z_T \sim q_{\theta'}$};

\node[redbox, above=2.0cm of pred] (nll)
{NLL Loss \quad $-\log p_\phi(Z_T \mid Z_C,\xi_T)$};

\node[bluebox, right=1.8cm of nll] (kl)
{KL Reg \quad $\mathrm{KL}(q_{\theta'} \,\|\, p_{\text{prior}})$};

\node[sumcirc] (plus) at ($(nll.east)!0.5!(kl.west)$) {$+$};

\draw[arr] (xC) -- (ctxenc);
\draw[arr] (ctxenc) -- (ZC);
\draw[arr] (ZC) -- (pred);

\draw[arr] (xT) -- (tgtenc);
\draw[arr] (tgtenc) -- (qT);
\draw[arr] (qT) -- (sample);

\draw[redarr] (pred.north) -- (nll.south);

\path (pred.north) -- (nll.south) coordinate[pos=0.55] (redmid);
\draw[redarr] (sample.west) -- (redmid);

\draw[bluearr] (qT.east) -- ++(1.6cm,0) -- ++(0,2.2cm) -- (kl.south);

\end{tikzpicture}

\caption{\textbf{Variational JEPA (VJEPA) architecture.}
The context $x_C$ is encoded into $Z_C$, conditioning a probabilistic predictor $p_\phi(Z_T \mid Z_C,\xi_T)$. The target $x_T$ is encoded into a distribution $q_{\theta'}$ (typically via EMA), from which a sample $Z_T$ is drawn. Training minimizes the sum of a negative log-likelihood term (red box) and a KL regularization term (blue box). $f_{\theta'}$ can output sufficient statistics (e.g. mean and covariance for Gaussians) for parameterizing $q_{\theta'}$.}
\label{fig:vjepa_architecture}
\end{figure}
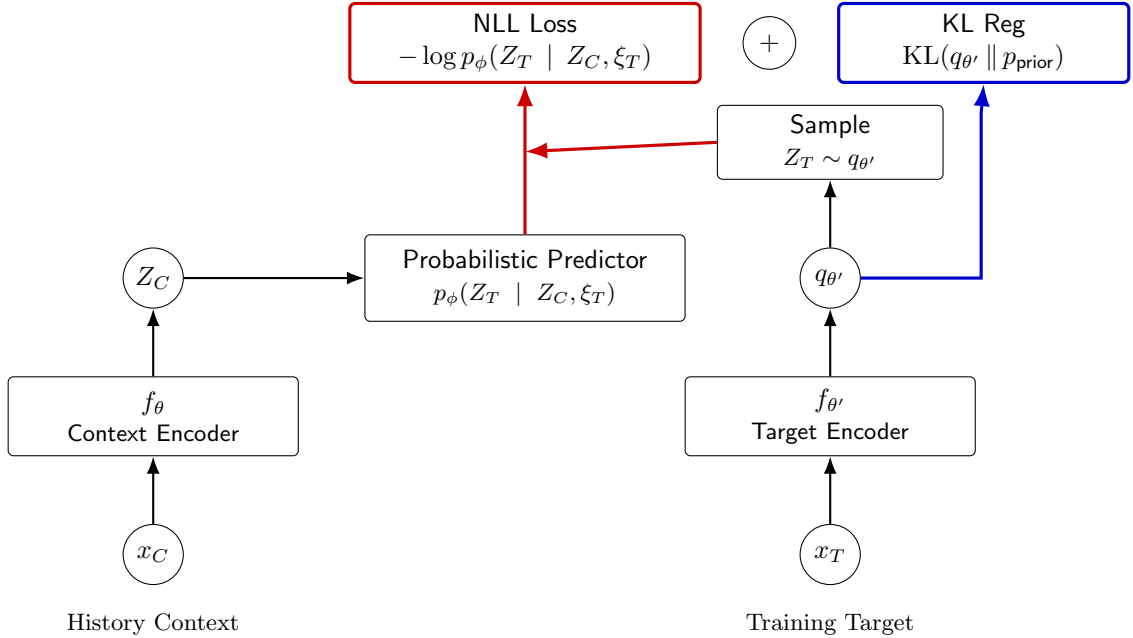

\subsection{Training Procedure}

Algorithm~\ref{alg:vjepa} summarizes VJEPA training.
In practice, the reparameterization trick may be used when $q_{\theta'}$ is continuous, and $p_\phi$ may be implemented as a Gaussian with learned mean and covariance (diagonal or structured), a mixture model, or a normalizing flow, depending on the desired expressivity. A concrete example is given in Section.\ref{sec:toy_experiment}.

\begin{algorithm}[H]
\caption{VJEPA Training}
\label{alg:vjepa}
\begin{algorithmic}[1]
\REQUIRE Dataset $\{x^{(i)}\}_{i=1}^N$, context encoder $f_\theta$, target encoder $f_{\theta'}$, predictive model $p_\phi$, prior $p(Z)$, EMA rate $\tau$, regularization weight $\beta$
\FOR{each minibatch $\{x\}$}
    \STATE Sample a context-target partition $(x_C, x_T)$ and corresponding target structure $\xi_T$
    \STATE Encode context:
    $
    Z_C \leftarrow f_\theta(x_C)
    $
    \STATE Compute inference distribution from target:
    $
    q_{\theta'}(Z_T \mid x_T)
    $
    \STATE Sample target representation (e.g. using reparameterization):
    $
    Z_T \sim q_{\theta'}(Z_T \mid x_T)
    $
    \STATE Evaluate predictive distribution:
    $
    p_\phi(Z_T \mid Z_C, \xi_T)
    $
    \STATE Compute variational loss:
    $
    \mathcal{L}
    =
    -\log p_\phi(Z_T \mid Z_C, \xi_T)
    + \beta \,\mathrm{KL} \left(q_{\theta'}(Z_T \mid x_T)\,\|\,p(Z_T)\right)
    $
    \STATE Update $\theta$ and $\phi$ via gradient descent on $\mathcal{L}$
    \STATE Update target encoder via EMA:
    $
    \theta' \leftarrow \tau \theta' + (1-\tau)\theta
    $
\ENDFOR
\end{algorithmic}
\end{algorithm}

\paragraph{Remark.}
Algorithm~\ref{alg:vjepa} reduces to standard JEPA training when the inference distribution
\(
q_{\theta'}(Z_T \mid x_T)
\)
collapses to a Dirac delta,
\(
q_{\theta'}(Z_T \mid x_T)=\delta \left(Z_T-f_{\theta'}(x_T)\right),
\)
and the predictive distribution is chosen as an isotropic Gaussian with fixed variance,
\(
p_\phi(Z_T \mid Z_C,\xi_T)=\mathcal{N} \left(Z_T \mid g_\phi(Z_C,\xi_T), \sigma^2 I\right).
\)
In this case, minimizing the VJEPA objective in Eq.~\eqref{eq:vjepa_objective} is equivalent (up to constants) to minimizing the squared-loss JEPA objective in Eq.~\eqref{eq:JEPA_squared_loss}.

\subsection{Prediction and Uncertainty Propagation}

After training, VJEPA can be used as a predictive model in representation space.
Given a new input $x_C$ and target specification $\xi_T$, we first compute the context representation (i.e. the context encoder Eq.~\eqref{eq:context_encoder})
\[
Z_C = f_\theta(x_C),
\]
and then form the predictive distribution (i.e. the predictor Eq.~\eqref{eq:vjepa_predictive_model})
\[
p_\phi(Z_T \mid Z_C, \xi_T).
\]

\paragraph{Point prediction.}
A point estimate of the target representation may be obtained from the predictive distribution
$p_\phi(Z_T \mid Z_C, \xi_T)$ in several standard ways.
A common choice is the predictive mean,
\[
\hat Z_T^{\text{mean}}
=
\mathbb{E}_{p_\phi(Z_T \mid Z_C, \xi_T)}[Z_T],
\]
which reduces to the deterministic JEPA predictor when $p_\phi$ is Gaussian with fixed variance.

Alternatively, a maximum \emph{a posteriori} (MAP) estimate may be used:
\[
\hat Z_T^{\text{MAP}}
=
\arg\max_{Z_T}  p_\phi(Z_T \mid Z_C, \xi_T),
\]
which coincides with the predictive mean for unimodal Gaussian models but differs for
multi-modal or heavy-tailed distributions.

More generally, one may consider a \emph{mode} prediction,
\[
\hat Z_T^{\text{mode}}
=
\operatorname*{mode} \left[p_\phi(Z_T \mid Z_C, \xi_T)\right],
\]
which selects a single most probable latent realization. For expressive predictive models (e.g. mixture models or normalizing flows), MAP or mode predictions allow VJEPA to select a representative future consistent with the learned multi-modal predictive structure.

\paragraph{Distributional prediction and uncertainty representation.}
Rather than using a single predicted embedding, VJEPA represents future uncertainty
explicitly through the predictive distribution $p_\phi(Z_T \mid Z_C, \xi_T)$. Downstream tasks may operate directly on this distribution, either by propagating
summary statistics (e.g. moments) when $p_\phi$ admits a closed form, or by approximating expectations using samples. In this sense, VJEPA produces a \emph{belief} over future representations rather than a point forecast. A general and flexible way to work with the predictive belief is via Monte Carlo sampling:
\[
Z_T^{(m)} \sim p_\phi(Z_T \mid Z_C, \xi_T),
\qquad m = 1,\dots,M.
\]
where $M$ is the total number of samples. These samples represent alternatively plausible future representations consistent with the observed context and target specification. Expectations of downstream quantities can then be approximated as
\[
\mathbb{E}[f(Z_T)]
 \approx 
\frac{1}{M}\sum_{m=1}^M f \left(Z_T^{(k)}\right),
\]
allowing uncertainty-aware prediction even when $p_\phi$ is complex (e.g. mixtures or normalizing flows).

\paragraph{Multi-step prediction and uncertainty propagation.}
For sequential prediction, VJEPA propagates uncertainty by integrating the predictive model in latent space. Given a belief $p(Z_t)$, the predictive belief at a future step $t+\Delta$ is
\begin{equation} \label{eq:VJEPA_multi_step_prediction}
    p(Z_{t+\Delta})
    =
    \int p_\phi(Z_{t+\Delta} \mid Z_t, \xi_{t+\Delta})\, p(Z_t)\, dZ_t,
\end{equation}
which defines a belief-state evolution analogous to \textit{Bayesian filtering}.
In practice, this integral may be approximated by moment propagation (when available) or by Monte Carlo rollouts:
\[
Z_{t+1}^{(m)} \sim p_\phi(Z_{t+1} \mid Z_t^{(m)}, \xi_{t+1}),
\qquad m=1,\dots,M,
\]
which yields a distribution over latent trajectories rather than a single deterministic path. This form of distributional rollout underlies the planning and control algorithms introduced in later sections, e.g. model predictive control (MPC).

\paragraph{Connection to ensemble world models.}
Sampling-based prediction in VJEPA is closely related to ensemble world models commonly used in model-based reinforcement learning.
In ensemble methods, multiple deterministic dynamics models
$\{f_{\phi^{(m)}}\}_{m=1}^M$
are trained independently, and uncertainty is approximated by disagreement across model predictions.
In contrast, VJEPA represents uncertainty \emph{within a single probabilistic model}
$p_\phi(Z_T \mid Z_C,\xi_T)$, where stochasticity arises from an explicit predictive distribution rather than parameter diversity.
Monte Carlo sampling from VJEPA therefore plays a role analogous to sampling from an ensemble, while avoiding the computational overhead and calibration issues associated with maintaining multiple models.
This perspective highlights VJEPA as a distributional generalization of ensemble world models operating directly in representation space.

\subsection{Collapse Avoidance} \label{sec:properties_of_VJEPA}
Here we formalize the claim that VJEPA discourages representational collapse.

\begin{definition} [Collapse]
    We say the context representation \emph{collapses} if there exists a constant $c\in\mathbb{R}^d$ such that
    \[
    f_\theta(x_C) = c \qquad \forall x_C.
    \]
\end{definition}

\begin{theorem}[No Collapsed Global Optimum under Target Diversity]
\label{thm:collapse_avoidance}
Consider the VJEPA objective in Eq.~\eqref{eq:vjepa_objective}. Assume:
\begin{enumerate}[label=(\roman*)]
    \item (\textbf{Target diversity}) There exist target inputs $x_T$ and $x_T'$ such that $q_{\theta'}(\cdot\mid x_T) \neq q_{\theta'}(\cdot\mid x_T')$ (as distributions).
    \item (\textbf{Nontrivial conditioning}) The predictive family $\{p_\phi(\cdot\mid Z_C,\xi_T)\}$ is such that, for some $Z_C\neq Z_C'$ and the same $\xi_T$, the induced distributions can differ: $p_\phi(\cdot\mid Z_C,\xi_T) \neq p_\phi(\cdot\mid Z_C',\xi_T)$ for some $\phi$.
\end{enumerate}
Then any globally optimal solution of the VJEPA training objective Eq.~\eqref{eq:vjepa_objective} is non-collapsed, i.e. $f_\theta(x_C)$ cannot be constant over all contexts at a global minimum.
\end{theorem}

\paragraph{Proof sketch.}
If $f_\theta(x_C)\equiv c$, then the predictor cannot use the context: for each $\xi_T$ it reduces to an \emph{unconditional} distribution $p_\phi(Z_T\mid c,\xi_T)$. The best possible choice of such an unconditional predictor (for fixed $\xi_T$) minimizes the cross-entropy with the aggregated target distribution
$q_{\theta'}(Z_T\mid \xi_T):=\mathbb{E}_{x_T\mid \xi_T}[q_{\theta'}(Z_T\mid x_T)]$, hence achieves
\[
\inf_{p} \ \mathbb{E}_{x_T\mid \xi_T}\,\mathbb{E}_{Z_T\sim q_{\theta'}(\cdot\mid x_T)}
[-\log p(Z_T)]
=
H_{q}(Z_T\mid \xi_T),
\]
which equals $H_q(Z_T\mid X_T,\xi_T) + I_q(Z_T;X_T\mid \xi_T)$, where $I(\cdot;\cdot)$ denotes mutual information (see Appendix.\ref{app:info_theory_identities}). By Assumption (i), we have $I_q(Z_T;X_T\mid \xi_T)>0$ (for at least one $\xi_T$ with nonzero probability), so the collapsed predictor incurs a strictly positive irreducible gap relative to a predictor that can condition on informative $Z_C$.
Assumption (ii) ensures the model class can exploit non-collapsed $Z_C$ to reduce the first term in Eq.~\eqref{eq:vjepa_objective}. The KL regularizer further discourages degenerate target encodings by penalizing pathological $q_{\theta'}(\cdot\mid x_T)$, but is not the main source of the strict gap. 
A more detailed proof can be found in Appendix.\ref{app:collapse_avoidance_proof}.
\hfill $\square$

\begin{remark}
    Theorem~\ref{thm:collapse_avoidance} shows that collapse is not merely discouraged but is fundamentally incompatible with global optimality of the VJEPA objective (Eq.~\eqref{eq:vjepa_objective}) under target diversity. A collapsed context representation forces the predictor to ignore informative conditioning, yielding an irreducible excess loss relative to non-collapsed solutions. Therefore, collapse avoidance in VJEPA arises from the structure of the predictive objective itself, rather than relying solely on architectural asymmetry, stop-gradient operations, or EMA updates.
\end{remark}

\section{Time-indexed VJEPA as a Latent Space Dynamical System}
\label{sec:JEPA_as_dynamics}

VJEPA always learns a conditional predictive relation in representation space of the form (Eq.~\eqref{eq:vjepa_predictive_model}):
\[
Z_T  \sim  p_\phi(Z_T \mid Z_C, \xi_T),
\]
where $Z_C$ is a context representation, $Z_T$ is a target representation, and $\xi_T$ specifies structural information about the target (e.g. spatial location, temporal index, or masking pattern). However, \emph{whether this predictive relation constitutes a latent dynamical system} depends on how the context-target pair is constructed and how $\xi_T$ is interpreted. JEPA is therefore not inherently a sequential or observation-likelihood model; in particular, JEPA does \emph{not} inherently impose sequential or Markovian structure. This distinction is important for understanding both the capabilities and the limitations of existing JEPA variants, including Video JEPA \cite{bardes2024vjepa,bardes2024vjepa_openReview}.

Here we formalize a \emph{conditional} dynamical-system perspective on VJEPA, introducing a \emph{latent dynamical system} in representation space when the context and target are indexed by time. Under this time-indexed formulation, VJEPA defines a latent state, a predictive transition model, and a belief-propagation mechanism, enabling filtering-style updates and control directly in latent space. Table.\ref{tab:jepa_dynamics_cases} summarizes the different regimes in which JEPA-style objectives \footnote{The temporal formulation in this section also applied to JEPA.} are commonly used.

\begin{table}[H]
\centering
\footnotesize
\begin{tabular}{lccc}
\hline
\textbf{} & \textbf{What is Predicted} & \textbf{Structure Learned} & \textbf{Dynamical System?} \\
\hline
Image JEPA (I-JEPA \cite{lejepa2022ijepa})
& Masked image regions
& Spatial / semantic relations
& \texttimes \\[0.4em]

Clip-level V-JEPA \cite{bardes2024vjepa}
& Masked video blocks
& Intra-clip spatiotemporal correlations
& \texttimes\ (not compositional) \\[0.4em]

\textbf{Time-indexed JEPA}
& $Z_{t+\Delta}$ from $Z_t$
& Latent state transitions across time
& \checkmark \\
\hline
\end{tabular}
\caption{JEPA always learns conditional predictive relations, but it induces a latent
dynamical system only when contexts and targets are explicitly indexed by time and predictions are reused across steps. Clip-level V-JEPA captures short-range correlations but does not, by itself, define a compositional state transition model.}
\label{tab:jepa_dynamics_cases}
\end{table}

In this section, we focus on the \emph{time-indexed} formulation of JEPA/VJEPA, under which the learned predictive model can be interpreted as a latent state-space model. This perspective will later justify belief propagation, filtering-style updates, and planning directly in representation space, while remaining fully consistent with the original JEPA training objective.

\subsection{Background: State-Space Models}

A (controlled) state-space model posits latent states $\{s_t\}_{t=1}^T$ and observations $\{x_t\}_{t=1}^T$ with a transition model and an observation model:
\begin{equation}
\label{eq:ssm_factorization}
p(s_{1:T},x_{1:T}\mid u_{1:T-1})
=
p(s_1)\prod_{t=1}^{T-1} p(s_{t+1}\mid s_t,u_t)\prod_{t=1}^T p(x_t\mid s_t),
\end{equation}
where $u_t$ is a control/action input. Classical HMMs (see Appendix.\ref{app:hmmlink}) are the special case where $s_t$ is discrete and transitions/emissions belong to parametric families; modern nonlinear state-space models allow continuous $s_t$ and often use approximate inference.

The defining feature of \eqref{eq:ssm_factorization} is that prediction and control can be performed by propagating a \emph{belief} over latent states $p(s_t\mid x_{\le t},u_{<t})$ forward through the transition model (also sometimes termed the \textit{transition kernel}).

\subsection{From JEPA Masking to Temporal Dynamics}

JEPA is defined by choosing, for each sample $x$, a context-target partition $(x_C,x_T)$ and a target specification $\xi_T$ (e.g. the spatial location, temporal index, or masking pattern, etc). In sequential settings, $x$ can represent a window of a video or a trajectory segment, and the target can correspond to a future time index. Concretely, for a temporal horizon $\Delta$, one may take
\[
x_C \equiv x_{\le t} \quad\text{and}\quad x_T \equiv x_{t+\Delta},
\qquad \xi_T \equiv (t+\Delta,\Delta),
\]
so that predicting the target embedding corresponds to predicting a (arbitrary) future latent state \footnote{We term the time-indexed JEPA a \emph{sequential} (or \emph{temporal}) JEPA. A characteristic feature of JEPA is that it can be trained to predict representations at a \emph{chosen} future horizon $\Delta$ directly (i.e. skip-time prediction), rather than only learning one-step transitions that must be iterated. This capability is complementary to (not incompatible with) Markovian and diffusion models, and it does not require an autoregressive likelihood factorization over observations.}. Let the JEPA encoder define the latent state
\begin{equation}
\label{eq:latent_state_def}
Z_t  :=  f_\theta(x_{\le t}),
\end{equation}
where $x_{\le t}$ denotes the available history (or the context crop/patch/window).
Similarly, the target encoder produces a representation of the target observation
\[
Z_{t+\Delta}  :=  f_{\theta'}(x_{t+\Delta}),
\]
used as the training target (with EMA parameters $\theta'$ as in Eq.~\eqref{eq:target_encoder_weights_EMA}).

\subsection{Latent Transition Model Induced by JEPA/VJEPA}

In deterministic JEPA, the predictor $g_\phi$ induces a latent transition (prediction)
\begin{equation}
\label{eq:jepa_dynamics_deterministic}
\hat Z_{t+\Delta}
=
g_\phi(Z_t,\xi_{t+\Delta}),
\end{equation}
which can be read as a \emph{deterministic} latent dynamics model in representation space. In VJEPA, this becomes an explicit \emph{stochastic} transition:
\begin{equation}
\label{eq:vjepa_dynamics}
Z_{t+\Delta} \sim p_\phi(Z_{t+\Delta}\mid Z_t,\xi_{t+\Delta}),
\end{equation}
which is exactly the state-transition component of a (latent) state-space model, except that it is learned \emph{without} requiring an observation likelihood. Note that, JEPA/VJEPA allows $\Delta > 1$, i.e. predicting/transitioning to an arbitrary future (or masked) latent state.

To accommodate control, we let $u_{t:t+\Delta-1}$ denote control inputs between $t$ and $t+\Delta$. Then the most general controlled version of \eqref{eq:vjepa_dynamics} is
\begin{equation}
\label{eq:vjepa_controlled_dynamics}
Z_{t+\Delta} \sim p_\phi \big(Z_{t+\Delta}\mid Z_t,\xi_{t+\Delta},u_{t:t+\Delta-1}\big),
\end{equation}
where $\xi_{t+\Delta}$ can include the time index and horizon $\Delta$.

\paragraph{Optional observation model.}
For completeness, one may optionally posit a decoder (as in Eq.~\eqref{eq:vjepa_obs_model_optional})
\begin{equation}
\label{eq:ssm_obs_optional}
x_{t+\Delta} \sim p_\psi(x_{t+\Delta}\mid Z_{t+\Delta}),
\end{equation}
which would turn the latent process in Eq.~\eqref{eq:vjepa_dynamics} into a fully generative
state-space model. In this case, the implied generative structure is
\[
Z_t  \longrightarrow  Z_{t+\Delta}  \longrightarrow  x_{t+\Delta},
\]
i.e. future observations are generated from future latent states predicted from the current
latent state. This contrasts with autoregressive observation models, which directly factorize
\[
Z_t  \longrightarrow  x_t  \longrightarrow  x_{t+1},
\]
and therefore require modeling the full conditional distribution
$p(x_{t+1}\mid x_{\le t})$.

Importantly, the central JEPA/VJEPA design principle is that learning proceeds
\emph{without} optimizing the observation-level likelihood term
$\log p_\psi(x_{t+\Delta}\mid Z_{t+\Delta})$.
The decoder is indexed at $t+\Delta$ to mirror the skip-time prediction structure of JEPA
(Eq.~\eqref{eq:vjepa_dynamics}); when $\Delta=1$, this formulation reduces to the standard
one-step state-space model.

\subsection{Sequential but Not Autoregressive}
\label{sec:non_autoregressive}

We clarify here a common confusion of equating ``sequential modeling'' with ``autoregressive modeling''. Autoregression refers specifically to a likelihood factorization over observations:
\begin{equation}
\label{eq:autoregressive_factorization}
p(x_{1:T})
=
\prod_{t=1}^T p(x_t\mid x_{<t}),
\end{equation}
which underlies a wide class of generative sequence models. This includes classical probabilistic models such as Hidden Markov Models (HMMs, see Appendix.\ref{app:hmmlink}) \cite{rabiner1989hmm} and linear Gaussian state-space models (Kalman filters) \cite{kalman1960new}, as well as modern deep generative models such as diffusion models \cite{ho2020DDPM}. Although HMMs and Kalman filters introduce latent states, marginalizing these states still yields an autoregressive likelihood over observations \footnote{
Many classical sequence models are \emph{autoregressive at the observation level}, even when latent states are introduced. For example, a HMM defines
$
p(s_{1:T},x_{1:T})
=
p(s_1)\prod_{t=2}^T p(s_t\mid s_{t-1})\prod_{t=1}^T p(x_t\mid s_t),
$
which implies, after marginalizing latent states $\{s_t\}$,
$
p(x_{1:T})=\prod_{t=1}^T p(x_t\mid x_{<t}).
$
Linear Gaussian state-space models (Kalman filters) similarly assume latent dynamics
$
z_{t+1}=Az_t+\epsilon_t,  x_t=Cz_t+\delta_t,
$
and also induce an autoregressive observation likelihood
$
p(x_{1:T})=\prod_t p(x_t\mid x_{<t}).
$
Modern generative models such as diffusion models \cite{ho2020DDPM} and autoregressive transformers \cite{vaswani2023attentionneed} make this factorization explicit. Thus, autoregression is a property of \emph{observation likelihood modeling}, not of temporal dynamics. JEPA introduces temporal structure by predicting latent representations directly, without defining or factorizing an observation-level likelihood.
}. Optimizing Eq.~\eqref{eq:autoregressive_factorization} forces the model to account for the full entropy of observations (e.g. textures, noise, syntax, etc), including aspects irrelevant to prediction for downstream tasks such as control.

Instead, JEPA/VJEPA defines a predictive model in \emph{latent space} (Eq.~\eqref{eq:vjepa_controlled_dynamics}):
\begin{equation}
\label{eq:latent_prediction_not_ar}
p(Z_{t+\Delta}\mid Z_t,\xi_{t+\Delta},u_{t:t+\Delta-1}),
\end{equation}
and learns representations by matching predicted latent states to target-encoder latents. Importantly, this does \emph{not} require specifying or factorizing an observation likelihood $p(x_{t+1}\mid x_{\le t})$. Thus, JEPA can be sequential (it supports multi-step prediction and belief propagation in $Z$) without being autoregressive in $x$.

\paragraph{Belief propagation in latent space.}
Given a belief $p(Z_t)$, sequential prediction (and decision) proceeds by pushing the belief through the latent transition:
\begin{equation}
\label{eq:latent_belief_propagation}
p(Z_{t+\Delta})
=
\int p_\phi \big(Z_{t+\Delta}\mid Z_t,\xi_{t+\Delta},u_{t:t+\Delta-1}\big)\, p(Z_t)\, dZ_t,
\end{equation}
matching Eq.~\eqref{eq:VJEPA_multi_step_prediction} and making explicit the analogy to filtering and planning.

\subsection{Clip-Level Prediction (V-JEPA) \textit{vs.} Sequential Dynamics}

It is important to distinguish between \emph{latent prediction} and \emph{latent dynamics}. In its current instantiation, V-JEPA \cite{bardes2024vjepa,bardes2024vjepa_openReview} is trained to predict masked spatio-temporal regions within a fixed video clip, using a large masked context that may already contain information from both past and future frames \cite{bardes2024vjepa}. As a result, the learned predictor is not explicitly constrained to represent a compositional, step-by-step temporal transition model.

This design choice has been noted in recent discussions  \cite{bardes2024vjepa_openReview} of V-JEPA: while V-JEPA captures rich temporal cues and yields strong representations for downstream video understanding tasks, it does not enforce sequential consistency across multiple prediction steps, nor does it learn a Markovian latent transition operator that can be iterated forward in time. In particular, predictions at different horizons are made independently, and the training objective does not require that
\[
Z_{t+2} \approx p_\phi(\,\cdot \mid Z_{t+1}, \xi_{t+2})
\quad\text{with}\quad
Z_{t+1} \sim p_\phi(\,\cdot \mid Z_t, \xi_{t+1}).
\]

The dynamical-system perspective adopted here makes this distinction explicit.
By parameterizing JEPA/VJEPA with a time-indexed context $Z_t = f_\theta(x_{\le t})$ and a horizon-dependent target specification $\xi_{t+\Delta}$,
the predictor (Eq.~\eqref{eq:vjepa_dynamics})
$
p_\phi(Z_{t+\Delta}\mid Z_t,\xi_{t+\Delta})
$
can be interpreted as a latent transition model.
When this transition is applied recursively, it induces a belief-state evolution analogous
to classical state-space models, enabling filtering, multi-step prediction, and planning
directly in representation space.

From this perspective, clip-level V-JEPA can be seen as learning a powerful
\emph{one-shot latent predictor}, whereas the sequential formulation studied here extends JEPA/VJEPA into a latent dynamical system. This extension does not contradict the original V-JEPA design, but rather clarifies the additional structural assumptions required to support temporal composition, uncertainty propagation, and control.

\paragraph{Summary.}
Time-indexed JEPA/VJEPA can be viewed as learning a latent state-space model in which:
(i) the encoder defines the latent state $Z_t$,
(ii) the predictor defines a latent transition model,
and (iii) optional observation modeling is decoupled from representation learning.
This is the mathematical sense in which JEPA is sequential but not autoregressive.

\section{JEPA/VJEPA for Control}
\label{sec:jepa_control}

Having interpreted JEPA/VJEPA as defining a latent dynamical system in representation space, we now show how such latent predictive states can be
used for \textit{planning and control}. The key message is classical in partially observable control: \emph{optimal control does not require reconstructing observations, but only a sufficient belief (information state) for predicting future consequences of actions}.

\subsection{Background: POMDPs and Belief-State Control}

A partially\footnote{
POMDP \cite{POMDP_lovejoy,kaelbling1998planning} generalizes an MDP to settings where the agent cannot directly observe the true system state. The environment evolves according to Markovian dynamics in a latent state space, but the agent only receives observations generated probabilistically from the underlying state. As a result, the agent must maintain a \emph{belief state}, which is a probability distribution over latent states, based on the history of observations and actions. Policies in a POMDP therefore map from belief states (or equivalently observation histories) to actions, rather than from fully observed states as in an MDP.} observable Markov decision process (POMDP) is given by
\[
\mathcal{M} = (\mathcal{S},\mathcal{U},\mathcal{X},P,R,\gamma),
\]
with latent state $s_t\in\mathcal{S}$, action $u_t\in\mathcal{U}$, observation $x_t\in\mathcal{X}$, transition kernel $P(s_{t+1}\mid s_t,u_t)$, reward (utility, or negative cost) $R(s_t,u_t)$, and discount $\gamma\in(0,1)$. The agent does not observe $s_t$ directly and must act based on history
\[
h_t := (x_{1:t},u_{1:t-1}).
\]
A standard result is that an \emph{information state} (a Bayesian belief state) is sufficient for optimal control \cite{kaelbling1998planning,smallwood1973optimal,sondik1978optimal,astrom1965optimal}:
\[
b_t(s) := p(s_t=s \mid h_t),
\]
then there exists an optimal policy that depends on history only through $b_t$:
\begin{equation}
\label{eq:pomdp_belief_opt_policy}
\pi^\star(u_t\mid h_t) = \pi^\star(u_t\mid b_t).
\end{equation}
Thus, control reduces to propagating and planning in belief space rather than observation space.

The sufficiency of the \textit{belief state} (or \textit{information state}) for optimal control is a foundational result established across the control theory and operations research literature. Åström (1965) \cite{astrom1965optimal} originally demonstrated that the conditional probability distribution of the latent state serves as a sufficient statistic for the complete history of actions and observations, thereby transforming the partially observable problem into a completely observable Markovian problem over the space of distributions. This theoretical basis was operationally expanded by Smallwood and Sondik (1973) \cite{smallwood1973optimal} and Sondik (1978) \cite{sondik1978optimal}, who proved the existence of optimal policies that depend exclusively on this "information vector" for both finite and infinite horizon settings, deriving the associated dynamic programming equations. Synthesizing these classical results for artificial intelligence planning, Kaelbling et al. (1998) \cite{kaelbling1998planning} re-affirm that because the belief state encapsulates all necessary historical information, the control problem formally reduces to solving a "belief MDP" which is a fully observable Markov decision process defined over the continuous space of probability distributions. These collective work confirm that the optimal policy need only map the current belief to an action $\pi^\star(u_t\mid b_t)$, and justify the reduction of general history-based planning to belief-space planning.

\subsection{JEPA/VJEPA Latent State as a Predictive Information State}

JEPA/VJEPA does not explicitly model $p(x_t\mid s_t)$ and therefore does not construct the Bayesian belief $b_t$ in $\mathcal{S}$. Instead, it learns a latent representation $Z_t$ from available context/history and equips it with a predictive transition model (Section~\ref{sec:JEPA_as_dynamics}):
\[
Z_t = f_\theta(x_{\le t}), 
\qquad
Z_{t+\Delta} \sim p_\phi(Z_{t+\Delta}\mid Z_t,\xi_{t+\Delta},u_{t:t+\Delta-1}).
\]
The control-relevant requirement is not that $Z_t$ reconstructs $x_t$, but that it is \emph{predictively sufficient} for evaluating future outcomes under candidate actions.

\paragraph{Predictive sufficiency (latent information state).}
We call $Z_t$ a \textit{predictive information state} for control (over horizon $H$) if for any action sequence $u_{t:t+H-1}$, the conditional distribution of future task variables depends on history only via $Z_t$. A convenient instantiation is to require that $Z_t$ predicts future latent representations:
\begin{equation}
\label{eq:predictive_sufficiency_latent}
p(Z_{t+1:t+H}\mid h_t,u_{t:t+H-1})
=
p(Z_{t+1:t+H}\mid Z_t,u_{t:t+H-1}),
\end{equation}
at least approximately under the learned model. This factorization is precisely the notion of predictive sufficiency formalized in the next subsection, where we show that such a representation is sufficient for optimal control.

\paragraph{Why reconstruction is unnecessary.}
In model-based control, planning means using a predictive model to compare candidate actions by their anticipated consequences \footnote{Concretely, a planner simulates candidate futures under different action sequences and selects the action (or sequence) that optimizes an objective such as expected reward or cost.}. If the task objective can be expressed in terms of predictive features (for example via a learned cost head on top of $Z$), then planning can be performed entirely in latent space. An observation decoder $p_\psi(x_t\mid Z_t)$ is therefore optional: it may aid visualization or auxiliary tasks, but it is not required for the control computation itself. While JEPA/VJEPA does not \emph{guarantee} that $Z_t$ is a sufficient information state without additional assumptions, it is designed to \emph{encourage} predictive sufficiency by learning an amortized predictive state together with a transition model $p_\phi(Z_{t+\Delta}\mid Z_t,\xi_{t+\Delta},u_{t:t+\Delta-1})$ across horizons. Under standard realizability conditions, namely that the existence of a true predictive information state $\tilde Z_t$ (e.g. a belief state or predictive state representation) that renders future task-relevant variables conditionally independent of history, and that $q_{\theta'}$ and $p_\phi$ are sufficiently expressive and correctly specified, a near-optimal VJEPA solution can make $Z_t=f_\theta(x_{\le t})$ an information-preserving transformation of $\tilde Z_t$ (up to approximation error). Consequently, $Z_t$ inherits the control-relevant predictive sufficiency of $\tilde Z_t$ \cite{Sutton2001,boots2011psr}. These assumptions are standard in predictive-state and representation-learning analyses, so JEPA/VJEPA can be viewed as replacing explicit belief-state estimation with representation learning while preserving the information needed for optimal control.

\subsection{Predictive Sufficiency for Control (Formalization)}
\label{sec:predictive_sufficiency_control}

We now formalize the notion of predictive sufficiency introduced above and state the standard control-theoretic consequence for optimal policies. Consider a POMDP with latent state $s_t$, observation $x_t$, and control $u_t$ (Section~\ref{sec:jepa_control}). Let the agent’s history be $h_t=(x_{1:t},u_{1:t-1})$, and let the (post-action) stage cost be $c_{t+1}=c(s_{t+1},u_t)$ (or any task cost measurable with respect to the environment’s trajectory).

\paragraph{From latent predictive sufficiency to control-relevant sufficiency.}
In control, we care about predicting \emph{cost-relevant} consequences of actions. If the stage cost can be written as a (possibly learned) function of latent states, e.g. $c_{t+k+1} = \ell(Z_{t+k+1},u_{t+k})$ (or $\ell(h(Z_{t+k+1}),u_{t+k})$ for a cost head $h$), then the latent factorization in Eq.~\eqref{eq:predictive_sufficiency_latent} implies that the conditional law of any cumulative cost over horizon $H$ depends on history only through $Z_t$. This motivates the following definition:

\begin{definition}[Control-Relevant Predictive Sufficiency]
\label{def:control_sufficiency}
A representation $Z_t=f_\theta(x_{\le t})$ is \emph{(control-)predictively sufficient over horizon $H$} if for any
action sequence $u_{t:t+H-1}$, the conditional distribution of future cumulative cost depends on history only through
$Z_t$:
\[
p\!\left(\sum_{k=0}^{H-1} c(s_{t+k+1},u_{t+k}) \,\middle|\, h_t,\,u_{t:t+H-1}\right)
=
p\!\left(\sum_{k=0}^{H-1} c(s_{t+k+1},u_{t+k}) \,\middle|\, Z_t,\,u_{t:t+H-1}\right).
\]
\end{definition}

\begin{lemma}[Latent sufficiency implies cost sufficiency]
\label{lem:latent_to_cost_sufficiency}
Assume the stage cost is measurable with respect to the latent trajectory.
If Eq.~\eqref{eq:predictive_sufficiency_latent} holds for horizon $H$, then $Z_t$ satisfies
Definition~\ref{def:control_sufficiency}.
\end{lemma}
This lemma shows that latent predictive sufficiency implies control-relevant predictive sufficiency whenever the task cost is measurable with respect to the learned latent state.

\paragraph{Proof sketch.}
Eq.~\eqref{eq:predictive_sufficiency_latent} implies the conditional distribution of the latent trajectory
$Z_{t+1:t+H}$ given $(h_t,u_{t:t+H-1})$ depends on history only via $Z_t$.
Any measurable functional of that trajectory, including the cumulative cost, therefore has the same property.
\hfill $\square$

\begin{theorem}[Sufficiency for Optimal Control from a Predictive Information State]
\label{thm:predictive_sufficiency_control}
If $Z_t$ is control-predictively sufficient in the sense of Definition~\ref{def:control_sufficiency} for all horizons
$H$ (or for the planning horizon used by the controller), then there exists an optimal policy that depends on history
only through $Z_t$:
\[
\pi^\star(u_t\mid h_t) = \pi^\star(u_t\mid Z_t).
\]
\end{theorem}

\paragraph{Proof sketch.}
The statement follows the standard POMDP information-state argument: if $Z_t$ renders the conditional law of all
future (cost-relevant) consequences of action sequences independent of the full history, then dynamic programming (or
finite-horizon rollout evaluation) can be carried out using $Z_t$ as the state variable. Hence an optimal policy can be
chosen Markov in $Z_t$, i.e. $\pi^\star(u_t\mid h_t)=\pi^\star(u_t\mid Z_t)$. \hfill $\square$

\paragraph{Relation to VJEPA training.}
VJEPA does not, by itself, \emph{guarantee} the predictive sufficiency property in Definition~\ref{def:control_sufficiency} without additional realizability and optimization assumptions. Rather, it is designed to \emph{encourage} such sufficiency by learning a latent predictive state together with a predictive model
$p_\phi(Z_{t+\Delta}\mid Z_t,\xi_{t+\Delta},u_{t:t+\Delta-1})$
that supports latent-space planning over multiple horizons. Under standard realizability assumptions, namely: (i) that there exists a true latent predictive information state $\tilde Z_t$ (e.g. a Bayesian belief state or a predictive state representation in the sense of PSRs \cite{Sutton2001,boots2011psr}) rendering future task-relevant variables conditionally independent of history; (ii) that the target encoder family $q_{\theta'}(Z_T\mid x_T)$ is expressive enough to represent the true conditional distribution of future predictive states given observations; and (iii) that the predictive model family $p_\phi$ is well specified and contains the true latent transition kernel induced by $\tilde Z_t$, then if VJEPA training reaches a (near-)global optimum over a sufficiently rich set of prediction horizons $\Delta$, the learned representation $Z_t=f_\theta(x_{\le t})$ becomes an information-preserving (invertible up to approximation error) transformation of $\tilde Z_t$. Consequently, $Z_t$ inherits the predictive sufficiency of $\tilde Z_t$ for evaluating future costs under candidate action sequences, and Theorem~\ref{thm:predictive_sufficiency_control} applies. In this sense, VJEPA can be viewed as learning an \emph{amortized predictive information state} whose sufficiency for control follows from correct specification and optimization, rather than from observation reconstruction or explicit belief-state estimation. As in PSR and belief-state analyses, these assumptions characterize an idealized regime; in practice VJEPA is expected to approximate predictive sufficiency to the extent allowed by data, capacity, and optimization.

\subsection{Latent-Space Planning Objective}

\paragraph{Planning in POMDPs.}
In a POMDP, optimal control cannot be expressed directly as a function of the unobserved latent state $s_t$. Instead, planning must be carried out in terms of the \emph{belief state}
\[
b_t(s) := p(s_t=s \mid x_{1:t}, u_{1:t-1}),
\]
which summarizes all information available to the agent. Given a transition model $P(s_{t+1}\mid s_t,u_t)$, an observation model $p(x_t\mid s_t)$, and a reward (or cost) function, the POMDP induces a \emph{belief-space MDP}, in which actions are selected to maximize expected cumulative reward under the evolution of beliefs.
Planning in a POMDP therefore amounts to simulating how candidate action sequences propagate the belief distribution forward in time and selecting the action that optimizes expected return \cite{smallwood1973optimal,kaelbling1998planning,ross2008planning}:
\[
\pi^\star(u_t \mid h_t) = \pi^\star(u_t \mid b_t).
\]
This belief-space formulation is exact but typically intractable in high-dimensional or continuous settings \cite{Papadimitriou1987,Madani2003,Kurniawati2009,sunberg2018}.

\paragraph{Latent belief approximation via VJEPA.}
VJEPA replaces the intractable Bayesian belief $b_t$ with a learned latent representation $Z_t=f_\theta(x_{\le t})$ that serves as an amortized, predictive belief state \footnote{Here ‘belief’ refers to a predictive distribution over latent representations rather than a posterior over true environment states.}. The predictive model $p_\phi(Z_{t+1}\mid Z_t,\xi_{t+1},u_t)$ then plays the role of a belief transition operator, enabling planning directly in latent space without requiring an explicit observation likelihood.

\paragraph{Latent-space planning objective.}
Let $c(Z_{t+1},u_t)$ be a stage cost defined on latent states (or via a cost head applied to $Z_{t+1}$). Given the current latent state $Z_t$, the $H$-step planning problem is
\begin{equation}
\label{eq:latent_mpc_objective}
u_{t:t+H-1}^\star
\in
\arg\min_{u_{t:t+H-1}}
\mathbb{E} \left[
\sum_{k=0}^{H-1} c(Z_{t+k+1},u_{t+k})
 \middle| 
Z_t,  u_{t:t+H-1}
\right],
\end{equation}
where the expectation is taken under the latent dynamics induced by the VJEPA predictive model\footnote{This is \textit{learning-based latent-space MPC}, see e.g. \cite{Koller2018,Watter2015,hafner2019dreamer}.}:
\[
Z_{t+k+1} \sim p_\phi(Z_{t+k+1}\mid Z_{t+k},\xi_{t+k+1},u_{t+k}).
\]
This objective is formally equivalent to stochastic model predictive control (MPC \cite{camacho2013MPC,rawlings2017MPC,mesbah2016stochasticmpc}), but instantiated entirely in representation space.

\paragraph{Distributional planning.}
When $p_\phi$ is stochastic, as in VJEPA, the objective \eqref{eq:latent_mpc_objective} involves an expectation over latent trajectories.
This enables \emph{distributional planning} \cite{mesbah2016stochasticmpc,Schwarm1999}: risk-neutral control arises naturally, and risk-sensitive criteria \cite{Borkar2002} (e.g. variance penalties \cite{Lai_2011,mesbah2016stochasticmpc} or CVaR \cite{rockafellar2000cvar,Chow2015}) can be incorporated by modifying the functional applied to the trajectory distribution.

\subsection{Sampling \textit{vs.} Point Prediction Planning}
\label{sec:sampling_vs_map_control}

The choice of how to use the predictive distribution is central in practice.

\paragraph{MAP/mean (trajectory) planning.}
A computationally cheap approximation is to plan using a single representative trajectory, e.g. the predictive mean (or MAP point when well-defined):
\[
\hat Z_{t+1} = \mu_\phi(Z_t,u_t)
\quad\text{or}\quad
\hat Z_{t+1} = \arg\max_{Z_{t+1}} p_\phi(Z_{t+1}\mid Z_t,u_t).
\]
This reduces planning to deterministic rollout and is often effective when the predictive distribution is unimodal and uncertainty is weak or approximately action-independent.

\paragraph{Sampling-based (distributional) planning.}
Alternatively, one can approximate \eqref{eq:latent_mpc_objective} by Monte Carlo rollouts:
\[
Z_{t+k+1}^{(m)} \sim p_\phi(\,\cdot \mid Z_{t+k}^{(m)},u_{t+k}), 
\qquad m=1,\dots,M,
\]
and estimate the expected cost by sample averaging. $M$ is the number of samples drawn at each time step. This is useful in multi-modal dynamics or partial observability, where committing to a single MAP/mean trajectory can ignore low-probability but potentially high-cost outcomes. Sampling also naturally supports robust or risk-sensitive planning by changing how samples are aggregated.

\subsection{Algorithm: VJEPA Sampling-Based Latent MPC}
\label{sec:vjepa_mpc_algo}

We now formalize the sampling-based planning procedure discussed in Section~\ref{sec:sampling_vs_map_control}. This procedure, commonly referred to as \textit{random shooting} or \textit{Model Predictive Path Integral} (MPPI) control in the literature \cite{williams2015MPPI,williams2016MPPI}, approximates the minimization of the expected cost (Eq.~\eqref{eq:latent_mpc_objective}) by simulating $N$ parallel latent trajectories using the VJEPA predictive model.

\begin{algorithm}[H]
\caption{VJEPA-Based Model Predictive Control (VJEPA-MPC)}
\label{alg:jmpc}
\begin{algorithmic}[1]
\REQUIRE History $x_{\le t}$, Horizon $H$, Number of samples $M$, Cost function $c$
\STATE \textbf{Encode:} Estimate current predictive state:
\[ Z_t = f_\theta(x_{\le t}) \]
\FOR{sample $i = 1$ to $M$}
    \STATE Sample candidate action sequence $u^{(i)}_{t:t+H-1} \sim p(u)$
    \STATE $Z^{(i)}_t \leftarrow Z_t$
    \FOR{step $k=0$ to $H-1$}
        \STATE Sample next latent state (dynamics rollout):
        \[ Z^{(i)}_{t+k+1} \sim p_\phi(Z_{t+k+1} \mid Z^{(i)}_{t+k}, u^{(i)}_{t+k}) \]
    \ENDFOR
    \STATE Compute cumulative cost for trajectory $i$:
    \[ J^{(i)} = \sum_{k=0}^{H-1} c(Z^{(i)}_{t+k+1}, u^{(i)}_{t+k}) \]
\ENDFOR
\STATE \textbf{Select:} Optimal action sequence index $i^\star = \arg\min_i J^{(i)}$
\STATE \textbf{Execute:} Apply first action $u^{(i^\star)}_t$ to environment
\end{algorithmic}
\end{algorithm}

\subsection{When is MAP Control Optimal?}
\label{sec:map_optimal}

We formalize a simple sufficient condition under which planning on a point prediction (mean or MAP) is equivalent to planning under the full predictive distribution. Following standard practice in stochastic MPC and belief-space control, we index the stage cost at the \emph{post-action} state: applying action $u_t$ induces a transition $Z_t \to Z_{t+1}$, and the cost evaluates the consequence of this action on the resulting state
\cite{mesbah2016stochasticmpc,rawlings2017MPC,bertsekas2012dp}.

\begin{theorem}[Optimality of MAP (Mean) Control under Quadratic Costs]
\label{thm:map_optimal}
Consider a one-step control problem with latent predictive model $p_\phi(Z_{t+1}\mid Z_t,u_t)$ and stage cost $c(Z_{t+1},u_t)$. Assume that for each $(Z_t,u_t)$,
\[
p_\phi(Z_{t+1}\mid Z_t,u_t)
=
\mathcal{N} \big(Z_{t+1}\mid \mu_\phi(Z_t,u_t),\,\Sigma\big),
\]
where the covariance $\Sigma$ does \emph{not} depend on $u_t$. Let the stage cost be quadratic in the next latent state:
\[
c(Z_{t+1},u_t)
=
(Z_{t+1}-z^\star)^\top Q_c (Z_{t+1}-z^\star) + r(u_t),
\qquad Q_c \succeq 0,
\]
with arbitrary $r(\cdot)$. Then the action minimizing expected cost,
\[
u_t^\star
\in
\arg\min_{u_t} 
\mathbb{E} \left[c(Z_{t+1},u_t)\mid Z_t,u_t\right],
\]
is equivalently obtained by minimizing the cost at the predictive mean:
\[
u_t^\star
\in
\arg\min_{u_t} 
c(\mu_\phi(Z_t,u_t),u_t).
\]
Moreover, since a Gaussian distribution has its unique MAP point at the mean, this is equivalently \emph{MAP control}.
\end{theorem}

\begin{proof}
Using the quadratic form and standard Gaussian moment identities,
\[
\mathbb{E} \left[(Z_{t+1}-z^\star)^\top Q_c (Z_{t+1}-z^\star)\right]
=
(\mu_\phi-z^\star)^\top Q_c (\mu_\phi-z^\star)
+
\mathrm{tr}(Q_c\Sigma),
\]
where the trace term $\mathrm{tr}(Q_c\Sigma)$ does not depend on $u_t$ by assumption. Hence
\[
\arg\min_{u_t}  \mathbb{E}[c(Z_{t+1},u_t)]
=
\arg\min_{u_t} 
\Big(
(\mu_\phi-z^\star)^\top Q_c (\mu_\phi-z^\star) + r(u_t)
\Big)
=
\arg\min_{u_t}  c(\mu_\phi(Z_t,u_t),u_t).
\]
Finally, for a Gaussian predictive distribution, the MAP point coincides with the mean $\mu_\phi(Z_t,u_t)$. A more detailed proof is provided in Appendix~\ref{app:map_optimal_proof}.
\end{proof}

\paragraph{Discussion.}
Theorem~\ref{thm:map_optimal} captures a common regime in which point-prediction planning is justified: predictive uncertainty contributes only an additive constant to the objective and therefore does not affect the optimizer. When the predictive covariance depends on the action (risk-sensitive settings), when costs are non-quadratic, or when $p_\phi$ is multi-modal (non-Gaussian), sampling-based rollouts become important, as optimizing a single MAP/mean trajectory may ignore uncertainty and rare but costly outcomes.

\paragraph{Corollary (Multi-step mean/MAP planning under action-independent covariance).}
\label{cor:map_optimal_multistep}
Consider an $H$-step planning problem in latent space with stochastic predictive dynamics
\[
Z_{t+k+1} \mid Z_{t+k},u_{t+k}
 \sim 
\mathcal{N} \big(\mu_\phi(Z_{t+k},u_{t+k}),\,\Sigma_k\big),
\qquad k=0,\dots,H-1,
\]
where each covariance $\Sigma_k$ is positive semidefinite and does \emph{not} depend on the action sequence $u_{t:t+H-1}$. Assume the stage cost is quadratic in the \emph{next} latent state:
\[
c(Z_{t+k+1},u_{t+k})
=
(Z_{t+k+1}-z^\star)^\top Q_k (Z_{t+k+1}-z^\star) + r_k(u_{t+k}),
\qquad Q_k\succeq 0.
\]

Define the risk-neutral MPC objective\footnote{\emph{Risk-neutral} control refers to optimizing the expected cumulative cost $\mathbb{E}[\sum_t c(Z_t,u_t)]$ without explicit preference for or against uncertainty. Two action sequences with equal expected cost are treated as equally desirable, even if they differ in variance or tail risk. By contrast, \emph{risk-sensitive} control incorporates uncertainty directly into the objective, for example via variance penalties,
exponential utility functions, or tail-risk criteria such as CVaR \cite{rockafellar2000cvar,Chow2015}.}
\[
J(u_{t:t+H-1})
:=
\mathbb{E} \left[
\sum_{k=0}^{H-1} c(Z_{t+k+1},u_{t+k})
 \middle| 
Z_t,  u_{t:t+H-1}
\right].
\]

Assume additionally that the induced state covariances $\mathrm{Cov}(Z_{t+k})$ under the rollout are independent of the chosen action sequence (e.g. in linear-Gaussian dynamics with additive, action-independent noise). Then the optimal action sequence 
\[
u_{t:t+H-1}^\star
 \in 
\arg\min_{u_{t:t+H-1}} 
J(u_{t:t+H-1})
\]
is equivalently obtained by deterministic planning on the mean (or MAP) trajectory:
\[
u_{t:t+H-1}^\star
\in
\arg\min_{u_{t:t+H-1}}
\sum_{k=0}^{H-1} c(\bar Z_{t+k+1},u_{t+k}),
\]
where the mean trajectory $\{\bar Z_{t+k}\}_{k=0}^{H}$ is defined recursively by
\[
\bar Z_t := Z_t,
\qquad
\bar Z_{t+k+1} := \mu_\phi(\bar Z_{t+k},u_{t+k}).
\]
Since each Gaussian predictive distribution has its unique MAP point at the mean, this procedure is equivalently \emph{MAP planning} at every step.

\begin{proof}
Write $Z_{t+k+1} = \bar Z_{t+k+1} + \varepsilon_{t+k+1}$, where $\bar Z_{t+k+1} := \mathbb{E}[Z_{t+k+1}]$ under the chosen action sequence and $\varepsilon_{t+k+1}$ is a zero-mean deviation with covariance $\mathrm{Cov}(Z_{t+k+1})$. For each $k$, expand the quadratic term:
\[
(Z_{t+k+1}-z^\star)^\top Q_k (Z_{t+k+1}-z^\star)
=
(\bar Z_{t+k+1}-z^\star)^\top Q_k (\bar Z_{t+k+1}-z^\star)
+
2\,\varepsilon_{t+k+1}^\top Q_k (\bar Z_{t+k+1}-z^\star)
+
\varepsilon_{t+k+1}^\top Q_k \varepsilon_{t+k+1}.
\]
Taking expectations, the cross term vanishes because $\mathbb{E}[\varepsilon_{t+k+1}]=0$, and
\[
\mathbb{E}[\varepsilon_{t+k+1}^\top Q_k \varepsilon_{t+k+1}]
=
\mathrm{tr} \big(Q_k\,\mathrm{Cov}(Z_{t+k+1})\big).
\]
Therefore,
\[
\mathbb{E} \left[(Z_{t+k+1}-z^\star)^\top Q_k (Z_{t+k+1}-z^\star)\right]
=
(\bar Z_{t+k+1}-z^\star)^\top Q_k (\bar Z_{t+k+1}-z^\star)
+
\mathrm{tr} \big(Q_k\,\mathrm{Cov}(Z_{t+k+1})\big).
\]

Summing over $k$, the expected cost decomposes as
\[
J(u_{t:t+H-1})
=
\sum_{k=0}^{H-1}
\Big(
(\bar Z_{t+k+1}-z^\star)^\top Q_k (\bar Z_{t+k+1}-z^\star)
+
r_k(u_{t+k})
\Big)
+
\sum_{k=0}^{H-1}
\mathrm{tr} \big(Q_k\,\mathrm{Cov}(Z_{t+k+1})\big).
\]
Under the stated assumptions, the covariance terms are independent of the action sequence and thus constitute a constant offset with respect to optimization over
$u_{t:t+H-1}$. Dropping this constant yields deterministic planning on the mean trajectory. Finally, since each conditional Gaussian $\mathcal{N}(\mu_\phi(\cdot),\Sigma_k)$ has its MAP point at the mean, this is equivalently MAP planning, i.e. 
\[
\arg\min_{u_{t:t+H-1}} 
J(u_{t:t+H-1})
=
\arg\min_{u_{t:t+H-1}} 
\sum_{k=0}^{H-1}
\Big(
(\bar Z_{t+k+1}-z^\star)^\top Q_k (\bar Z_{t+k+1}-z^\star)
+
r_k(u_{t+k})
\Big).
\]
\end{proof}

\section{Information-Theoretic Analysis}
\label{sec:information_theoretic}

In this section, we analyze VJEPA through the lens of information theory, specifically relating the variational objective to the maximization of predictive mutual information and the Predictive Information Bottleneck (PIB) principle \cite{tishby2000information, bialek2001predictability,alemi2017ib,alemi2019variationalpredictiveinformationbottleneck,wang2019pastfuture,Wang_2021}.

\subsection{Variational Maximization of Predictive Information}

We establish that minimizing the VJEPA loss is equivalent to maximizing a lower bound on the mutual information\footnote{In probability theory and information theory, mutual information (MI) measures the mutual dependence between two random variables, quantifying the information obtained about one by observing the other. Let $(X, Y)$ be a pair of random variables over space $\mathcal{X} \times \mathcal{Y}$ with joint distribution $P_{(X,Y)}$ and marginals $P_X, P_Y$. The mutual information is defined as the KL divergence between the joint and product distributions: $I(X; Y) = D_{\mathrm{KL}}(P_{(X, Y)} \| P_X \otimes P_Y)$. It can also be expressed via entropy $H(\cdot)$ and conditional entropy $H(\cdot|\cdot)$ as: $I(X; Y) = H(X) - H(X|Y) = H(Y) - H(Y|X)$. By the properties of KL divergence, $I(X;Y) \ge 0$, with equality if and only if $X$ and $Y$ are independent. See Appendix.\ref{app:info_theory_identities} for more details.} between the current latent state $Z_t$ and the future latent state $Z_{t+\Delta}$.

\begin{theorem}[Variational Mutual Information Lower Bound]
\label{thm:mi_bound}
Let $(Z_t, Z_{t+\Delta})$ be the joint distribution of context and target representations induced by the data and the encoder policies. The mutual information $I(Z_t; Z_{t+\Delta})$ is lower-bounded by the negative cross-entropy (or expected log-likelihood) of the predictive distribution:
\[
I(Z_t; Z_{t+\Delta})
\ge
\mathbb{E}_{p(Z_t, Z_{t+\Delta})} \big[ \log p_\phi(Z_{t+\Delta} \mid Z_t) \big]
+ H(Z_{t+\Delta}),
\]
where $H(Z_{t+\Delta})$ is the marginal entropy of the target representations.
\end{theorem}

\begin{proof}
By definition, $I(Z_t; Z_{t+\Delta}) = H(Z_{t+\Delta}) - H(Z_{t+\Delta} \mid Z_t)$. The conditional entropy is defined as $H(Z_{t+\Delta} \mid Z_t) = -\mathbb{E} [\log p(Z_{t+\Delta} \mid Z_t)]$. Using the non-negativity of KL divergence, $D_{\text{KL}}(p(\cdot|Z_t) \,\|\, p_\phi(\cdot|Z_t)) \ge 0$, we have:
\[
\mathbb{E} [\log p(Z_{t+\Delta} \mid Z_t)] \ge \mathbb{E} [\log p_\phi(Z_{t+\Delta} \mid Z_t)].
\]
Substituting this inequality yields the result. This is known as the Barber-Agakov bound \cite{barber2003algorithm,Poole2019}.
\end{proof}

\paragraph{Implication.}
The VJEPA objective (Eq.~\ref{eq:vjepa_objective}) minimizes $-\log p_\phi(Z_{t+\Delta} \mid Z_t)$. Since $H(Z_{t+\Delta})$ depends only on the target encoder (which evolves slowly via EMA), minimizing the VJEPA loss effectively maximizes the mutual information between the past and the future representations, $I(Z_t; Z_{t+\Delta})$.

\subsection{Predictive Information Bottleneck}

The \emph{Information Bottleneck (IB)} method \cite{tishby2000information} provides a general information-theoretic framework for finding a compressed representation $Z$ of an input source $X$ that retains the maximum possible information about a relevant target variable $Y$. Formally, it seeks to minimize the functional
\[
\mathcal{L}_{\text{IB}} = I(X; Z) - \beta I(Z; Y),
\]
where $I(\cdot;\cdot)$ denotes mutual information\footnote{The mutual information $I(\cdot;\cdot)$ is a measure of mutual dependence between two random variables defined as the \textit{Kullback-Leibler divergence} between their joint distribution and the product of their marginals: $I(X;Y) = D_{\mathrm{KL}}(P_{(X,Y)} \| P_X \otimes P_Y) = H(X) - H(X|Y)$. See Appendix.\ref{app:info_theory_identities}.} and $\beta$ is a Lagrange multiplier controlling the trade-off between complexity (compression of $X$) and accuracy (prediction of $Y$).
The \emph{Predictive Information Bottleneck (PIB)} \cite{bialek2001predictability} specializes this principle to temporal data: it aims to extract a summary of the \emph{past} ($X = x_{\le t}$) that is maximally predictive of the \emph{future} ($Y = x_{t+\Delta}$), while discarding irrelevant, noisy, or redundant details.

A core distinction of VJEPA is its adherence to this predictive principle. Unlike autoencoders, which effectively maximize $I(Z_t; x_{\le t})$ by reconstructing the full input history, VJEPA acts as a PIB. It seeks to capture information about the future while remaining invariant to non-predictive details of the past.

Let $x_{\le t}$ be the observation history. We posit a generative process where observations consist of a predictable signal $S$ and nuisance noise $N$: $x = S + N$. The future $x_{t+\Delta}$ depends on history only through $S$.

\begin{proposition}[Invariance to Nuisance Variability]
\label{prop:nuisance_invariance}
Let the input $x$ decompose into a predictive signal $S$ and a nuisance variable $N$ (where $N$ is independent of the future target $Z_{t+\Delta}$).
Generative models maximizing $\log p(x \mid Z)$ must encode both $S$ and $N$ to minimize reconstruction error (since the entropy $H(x)$ includes $H(N)$).
In contrast, the VJEPA objective is invariant to representations that discard $N$ provided they preserve the mutual information with the future, $I(Z_t; Z_{t+\Delta})$. Consequently, VJEPA admits \emph{minimal sufficient statistics} that filter out nuisance variability while maintaining optimal prediction loss. (See Proof in Appendix.\ref{app:proof_nuisance}).
\end{proposition}

This proposition implies that VJEPA \emph{allows} the representation to be minimal (compressing away noise), whereas reconstruction-based objectives \emph{force} the representation to be maximal (retaining noise). This efficiency is critical for control, where acting on noise can lead to instability.

\subsection{Contrast with Generative Objectives}

We can now formally distinguish VJEPA from autoregressive (AR) world models \cite{ha2018RNN, hafner2019dreamer} by analyzing their respective training objectives through an information-theoretic lens.

\paragraph{Derivation of the Generative Penalty.}
Standard generative world models are trained to maximize the log-likelihood of future observations $x_{t+\Delta}$ given the current latent state $Z_t$. This corresponds to minimizing the Negative Log-Likelihood (NLL):
\[
\mathcal{L}_{\text{AR}}
=
\mathbb{E}_{data} \left[ -\log p_\psi(x_{t+\Delta} \mid Z_t) \right].
\]
In the limit of infinite data and a sufficiently expressive model, minimizing the NLL is equivalent to minimizing the \emph{conditional entropy} of the targets given the representation:
\[
\mathcal{L}_{\text{AR}}
\cong
H(x_{t+\Delta} \mid Z_t).
\]
Using the fundamental identity relating entropy and mutual information, $H(X \mid Y) = H(X) - I(X; Y)$, we can rewrite the objective as:
\[
\mathcal{L}_{\text{AR}}
\cong
\underbrace{H(x_{t+\Delta})}_{\text{Constant}} - I(x_{t+\Delta}; Z_t).
\]
Since the marginal entropy of the data $H(x_{t+\Delta})$ is a constant property of the environment (determined by the dataset), minimizing the generative loss is equivalent to maximizing the mutual information between the latent state and the \emph{raw pixels}:
\[
\min \mathcal{L}_{\text{AR}} \iff \max I(x_{t+\Delta}; Z_t).
\]

\paragraph{Decomposition into Signal and Noise.}
Consider the observation $x$ as a composition of task-relevant signal $S$ and nuisance noise $N$ (e.g., ``Noisy TV'' static, camera grain), such that $x_{t+\Delta} = (S_{t+\Delta}, N_{t+\Delta})$. By applying the chain rule for mutual information\footnote{The chain rule for mutual information states that, for any random variables $X, Y, Z$, the information provided by the pair $(X, Y)$ about $Z$ is: $I(X, Y; Z) = I(X; Z) + I(Y; Z \mid X)$.}, we can decompose the information the representation $Z_t$ holds about the full observation into two components:
\[
I(x_{t+\Delta}; Z_t)
=
I(S_{t+\Delta}, N_{t+\Delta}; Z_t)
=
I(S_{t+\Delta}; Z_t) + I(N_{t+\Delta}; Z_t \mid S_{t+\Delta}).
\]
The first term, $I(S_{t+\Delta}; Z_t)$, measures how much the representation captures the predictable, task-relevant signal. The second term, $I(N_{t+\Delta}; Z_t \mid S_{t+\Delta})$, measures the information the representation retains about the noise, given that the signal is already known. If we assume the noise $N_{t+\Delta}$ is independent of the signal $S_{t+\Delta}$ (a common property of sensor noise), this second term identifies the specific capacity the model must allocate to capturing unpredictable nuisance variability.

Substituting this back into the loss function, we obtain the effective objective for AR models:
\begin{equation}
\label{eq:ar_loss_decomposed}
\mathcal{L}_{\text{AR}}
\approx
- I(Z_t; S_{t+\Delta}) \underbrace{- I(Z_t; N_{t+\Delta})}_{\text{Penalty term}}.
\end{equation}
Importantly, the second term acts as a penalty: the loss \emph{increases} if the representation $Z_t$ fails to capture information about the noise $N_{t+\Delta}$. To minimize $\mathcal{L}_{\text{AR}}$, the model is mathematically forced to allocate capacity to predict high-entropy, task-irrelevant details.

\paragraph{The VJEPA Advantage.}
In contrast to generative models that operate on pixels, VJEPA defines the prediction target in an abstract representation space $Z_{t+\Delta}$, generated by an encoder $f_{\theta'}$ designed to capture the predictable signal $S$ while discarding nuisance noise $N$. We can formally derive the VJEPA objective by considering the minimization of the latent negative log-likelihood (NLL):
\begin{equation}
\mathcal{L}_{\text{VJEPA}} = \mathbb{E} \left[ -\log p_\phi(Z_{t+\Delta} \mid Z_t) \right] \cong H(Z_{t+\Delta} \mid Z_t),
\end{equation}
where the expected NLL corresponds to the conditional entropy of the future latent state. Using the fundamental identity $H(X \mid Y) = H(X) - I(X; Y)$, we rewrite the objective as:
\begin{equation}
\mathcal{L}_{\text{VJEPA}} \cong H(Z_{t+\Delta}) - I(Z_t; Z_{t+\Delta}).
\end{equation}
Under the assumption of a well-regularized encoder, maintained via the KL-regularization term shown in Fig.\ref{fig:vjepa_architecture}, the marginal entropy $H(Z_{t+\Delta})$ is encouraged to remain high to prevent representation collapse. Consequently, the optimization reduces to maximizing the mutual information in the latent space: $\min \mathcal{L}_{\text{VJEPA}} \iff \max I(Z_t; Z_{t+\Delta})$. 

According to the Predictive Information Bottleneck (PIB) principle, an ideal target encoder should filter out unpredictable nuisances such that $Z_{t+\Delta} \approx f_{\theta'}(S_{t+\Delta})$. Substituting this signal-centric target yields the decomposed VJEPA advantage:
\begin{equation}
\label{eq:vjepa_loss_decomposed}
\mathcal{L}_{\text{VJEPA}}
\approx
- I(Z_t; Z_{t+\Delta})
\cong
- I(Z_t; S_{t+\Delta}).
\end{equation}
Because the target $Z_{t+\Delta}$ has already discarded the pixel-level noise $N_{t+\Delta}$, the VJEPA objective contains no penalty term for ignoring $N$. This allows the world model to filter out high-frequency observation noise \emph{before} the prediction loss is computed, focusing its limited capacity directly on the causal dynamics relevant for planning and control.

\section{Bayesian JEPA (BJEPA)}
\label{sec:bjepa}

VJEPA models uncertainty directly in \emph{representation space} by learning a conditional distribution over future latent states. For example, the predictive model Eq.~\eqref{eq:vjepa_predictive_model} can be formulated as
\begin{equation}
\label{eq:vjepa_predictive}
p_\phi(Z_T \mid Z_C)
=
\mathcal{N}\!\big(\mu_\phi(Z_C),\,\Sigma_\phi(Z_C)\big),
\end{equation}
which scales naturally to high-dimensional embeddings and aligns with the interpretation of $Z_t$ as a predictive information state.

In this section, we introduce \emph{Bayesian JEPA (BJEPA)}, which extends VJEPA by explicitly factorizing the predictive model into independent factors: a learned dynamics expert and a constraint-based expert. This enables principled incorporation of goals, constraints, and task structure while preserving JEPA’s core design principles.

\subsection{Bayesian Factorization via Product of Experts}

BJEPA approximates the target posterior distribution using a \emph{Product of Experts} (PoE) mechanism\footnote{PoE is a machine learning framework that models a probability distribution by combining the output from several simpler probability models (experts), this is achieved by multiplying their probability distributions and renormalizing.}. Strictly speaking, PoE is the computational operation used to fuse multiple density functions; however, we structure the experts such that their product corresponds to a Bayesian update.

We treat the historical context and the auxiliary constraints as conditionally independent sources of information regarding the future state, combining them via the product of their respective probability densities:
\begin{equation}
\label{eq:bjepa_posterior}
p(Z_T \mid Z_C, \eta)
 \propto 
p_{\text{like}}(Z_T \mid Z_C) 
p_{\text{prior}}(Z_T \mid \eta).
\end{equation}
This formulation allows us to interpret the factors through a Bayesian lens:
\begin{itemize}
    \item $p_{\text{like}}(Z_T \mid Z_C)$ acts as the \emph{predictive likelihood} (or transition model). It is learned from data and encodes the system's typical latent dynamics conditioned on past context.
    \item $p_{\text{prior}}(Z_T \mid \eta)$ acts as a \emph{latent-space prior}. It injects auxiliary information $\eta$, such as goal regions, safety constraints, or physical manifolds, into the inference process.
\end{itemize}

Mathematically, the product formulation implements a logical AND \footnote{Unlike a mixture model, which effectively performs a logical OR.}: the predicted state must satisfy the dynamics \emph{and} the constraints simultaneously. Just as a Kalman Filter update intersects the prediction with a measurement, BJEPA intersects physical possibility with task necessity. This effectively carves the latent space to the intersection of physically probable and task-compliant states, i.e. we get a predictive distribution that only has high probability in regions that are both physically reachable (dynamics) and closer to the goal (prior). As shown in Appendix~\ref{app:poe_derivation}, this formulation \footnote{The full Eq.\eqref{eq:bjepa_posterior} is: $p(Z_T \mid Z_C, \eta) = \frac{1}{\mathcal{Z}} \left( p_{\text{like}}(Z_T \mid Z_C) \times p_{\text{prior}}(Z_T \mid \eta) \right)$. In the context of PoE, multiplying probabilities (element-wise multiplication of two PDFs) corresponds to a logical AND operation on constraints; the resulting distribution represents states that are consistent with both the dynamics (history) AND the auxiliary constraints (goals). After multiplication, the resulting function usually does not integrate to unity; a normalization step divides the product by the partition function $\mathcal{Z}$ to ensure the result is a valid probability distribution.} is exact under the assumption that the context $Z_C$ and auxiliary info $\eta$ are conditionally independent given the target $Z_T$, up to a normalization constant and marginal scaling. Thus, BJEPA utilizes the PoE mechanism to realize a Bayesian inference process entirely within the representation space. Compared to a standard autoregressive (AR) model, we would have to search the entire predicted pixel space to find a path to a goal; in BJEPA, the PoE constraints this search within the intersected space \footnote{If both $p_{like}$ and $p_{prior}$ are Gaussian, the posterior can be derived analytically (also Gaussian), by simply adding the precisions (inverse variances) of the two experts in the latent space.}.

\subsection{The BJEPA Architecture}

BJEPA builds on the standard JEPA/VJEPA architecture by retaining the
\emph{context/target encoder structure} used for representation learning,
and augmenting it with an explicit \emph{latent-space prior encoder}.
Concretely, BJEPA consists of the following components:
\begin{itemize}
    \item a \textbf{context (online) encoder} $f_\theta$ that maps past observations to a latent context representation,
    \[
        Z_C = f_\theta(x_C),
    \]
    this is identical to the context encoder used in JEPA/VJEPA
    (Eq.~\eqref{eq:context_encoder});

    \item a \textbf{target encoder} $f_{\theta'}$ that maps future observations
    to target latent representations,
    \[
        Z_T = f_{\theta'}(x_T),
    \]
    as in standard JEPA training (Eq.~\eqref{eq:target_encoder});
    the parameters $\theta'$ are typically updated via an exponential moving average (EMA) of $\theta$ (Eq.~\eqref{eq:target_encoder_weights_EMA}) and are not directly optimized by gradient descent;

    \item a \textbf{predictive likelihood network} parameterizing
    \begin{equation}  \label{eq:bjepa_likelihood}
        p_{\text{like}}(Z_T \mid Z_C),
    \end{equation}
    which generalizes the VJEPA probabilistic predictor
    (Eq.~\eqref{eq:vjepa_predictive_model}) and models uncertainty over future
    latent representations conditioned on context;

    \item a \textbf{prior encoder} that maps auxiliary information
    $\eta$, e.g. goals, constraints, demonstrations, physical structure,
    to a latent-space prior distribution:
    \begin{equation} \label{eq:bjepa_prior}
        p_{\text{prior}}(Z_T \mid \eta),
    \end{equation}
    which is \emph{new} relative to JEPA/VJEPA and enables Bayesian conditioning
    in representation space.
\end{itemize}
The full Bayesian JEPA (BJEPA) architecture can be summarized by the following latent-space mappings:
{\footnotesize
\begin{align*}
Z_C &= f_\theta(x_C),
&& \text{Context / online encoder: what has happened (history)} \\[0.3em]
Z_T &= f_{\theta'}(x_T),
&& \text{Target encoder (training only): what usually happens (learned dynamics)
} \\[0.3em]
p_{\text{like}}(Z_T \mid Z_C)
&= \text{Likelihood}_\phi(Z_C),
&& \text{Predictive likelihood} \\[0.3em]
p_{\text{prior}}(Z_T \mid \eta)
&= \text{Prior}_{\phi'}(\eta),
&& \text{Latent prior: what should or must or expect to happen (goals/constraints)
} \\[0.3em]
p(Z_T \mid Z_C, \eta)
&\propto
p_{\text{like}}(Z_T \mid Z_C) 
p_{\text{prior}}(Z_T \mid \eta),
&& \text{'Bayesian posterior'} 
\label{eq:bjepa_posterior_summary}
\end{align*}
}

These components are schematically sketched in Fig.\ref{fig:bjepa_architecture}, in which the \textit{combination operator} (``Combine'') fuses the likelihood and prior distributions. Mathematically, this implements PoE via pointwise multiplication of densities followed by normalization. In the common case where both experts are parameterized as Gaussians, i.e. $p_{\text{like}}=\mathcal{N}(\mu_{dyn}, \Sigma_{dyn})$ and $p_{\text{prior}}=\mathcal{N}(\mu_{aux}, \Sigma_{aux})$, this operation admits a closed-form solution equivalent to a Kalman update in latent space. The posterior mean $\mu_{post}$ is computed as the precision-weighted average of the expert means:
$\mu_{post} = (\Sigma_{dyn}^{-1} + \Sigma_{aux}^{-1})^{-1} (\Sigma_{dyn}^{-1}\mu_{dyn} + \Sigma_{aux}^{-1}\mu_{aux})$. This mechanism ensures that the predicted representation $\hat{Z}_T$ satisfies both the physical dynamics and the task constraints, effectively intersecting the two manifolds in latent space. Note that, all components operate entirely in representation space. Observation reconstruction, reward supervision, and autoregressive likelihood factorization are not required.

\begin{figure}[H]
\centering
\begin{tikzpicture}[
    font=\small\sffamily,
    >=Latex,
    node distance=1.2cm and 1.2cm,
    box/.style={
        draw, rounded corners=2pt, 
        minimum width=3.0cm, minimum height=1.0cm, 
        text width=2.8cm, align=center
    },
    sbox/.style={
        draw, rounded corners=2pt, 
        minimum width=3.2cm, minimum height=1.1cm, 
        text width=3.0cm, align=center
    },
    op/.style={
        draw, rounded corners=2pt, 
        minimum width=2.8cm, minimum height=0.9cm, 
        align=center
    },
    circ/.style={
        draw, circle, minimum size=8mm, inner sep=0pt, outer sep=0pt
    },
    redbox/.style={
        draw=red!80!black, very thick, rounded corners=2pt, 
        minimum width=3.0cm, minimum height=1.0cm, 
        text width=2.8cm, align=center
    },
    redarr/.style={->, very thick, draw=red!80!black},
    darr/.style={->, thick, dashed},
    arr/.style={->, thick}
]

\coordinate (c1) at (0,0);      
\coordinate (c2) at (4.0,0);    
\coordinate (c3) at (8.0,0);    

\node[circ] (xC) at (c1) {$x_C$};
\node[circ] (eta) at (c2) {$\eta$};
\node[circ] (xT) at (c3) {$x_T$};

\node[below=0.2cm of xC, align=center, font=\footnotesize] {History Context};
\node[below=0.2cm of eta, align=center, font=\footnotesize] {Goals / Constraints};
\node[below=0.2cm of xT, align=center, font=\footnotesize] {Training Only};

\node[box, above=0.8cm of xC] (ctxenc) {$f_\theta$\\ \footnotesize Context Encoder};
\node[box, above=0.8cm of eta] (priorenc) {$f_{\phi'}$\\ \footnotesize Prior Encoder};
\node[box, above=0.8cm of xT] (tgtenc) {$f_{\theta'}$\\ \footnotesize Target Encoder};

\node[circ, above=0.8cm of ctxenc] (ZC) {$Z_C$};
\node[circ, above=0.8cm of priorenc] (psieta) {$\psi(\eta)$};
\node[circ, above=0.8cm of tgtenc] (ZT) {$Z_T$};

\coordinate (mid_left) at ($(ZC)!0.5!(psieta)$);

\node[op, above=3.5cm of mid_left] (combine) {Product of Experts \\ \footnotesize Multiply \& Normalize};

\node[sbox, above=0.6cm of combine] (post) {Posterior Predictor \\ \footnotesize $p(Z_T \mid Z_C, \eta)$};

\node[box, above=0.6cm of post] (infer) {Inference \\ \footnotesize MAP or Sampling};

\node[circ, above=0.6cm of infer] (Zhat) {$\hat{Z}_T$};

\node[sbox, below left=0.8cm and -0.5cm of combine] (like) {Likelihood \\ \footnotesize $p_{\text{like}}(Z_T\mid Z_C)$};
\node[sbox, below right=0.8cm and -0.5cm of combine] (prior) {Latent Prior \\ \footnotesize $p_{\text{prior}}(Z_T\mid \eta)$};

\node[redbox] (match) at (Zhat -| xT) {Match \\ \footnotesize JEPA/VJEPA Loss};

\draw[arr] (xC) -- (ctxenc);
\draw[arr] (ctxenc) -- (ZC);
\draw[arr] (eta) -- (priorenc);
\draw[arr] (priorenc) -- (psieta);
\draw[arr] (xT) -- (tgtenc);
\draw[arr] (tgtenc) -- (ZT);
\draw[arr] (ZC) -- (like.south); 
\draw[arr] (psieta) -- (prior.south); 
\draw[arr] (like.north) -- ++(0,0.3) -| (combine.220);
\draw[arr] (prior.north) -- ++(0,0.3) -| (combine.320);
\draw[arr] (combine) -- (post);
\draw[arr] (post) -- (infer);
\draw[arr] (infer) -- (Zhat);
\draw[redarr] (Zhat) -- (match);
\draw[darr, draw=red!80!black, line width=1.2pt] (ZT) -- (match);

\end{tikzpicture}
\caption{\textbf{Bayesian JEPA (BJEPA) architecture.} 
A predictive likelihood models dynamics while a prior encoder maps auxiliary input $\eta$ to a latent constraint. These distributions are fused via a PoE operator to yield a posterior predictive distribution. During training, the prediction is matched against the target encoder output using a JEPA loss.}
\label{fig:bjepa_architecture}
\end{figure}
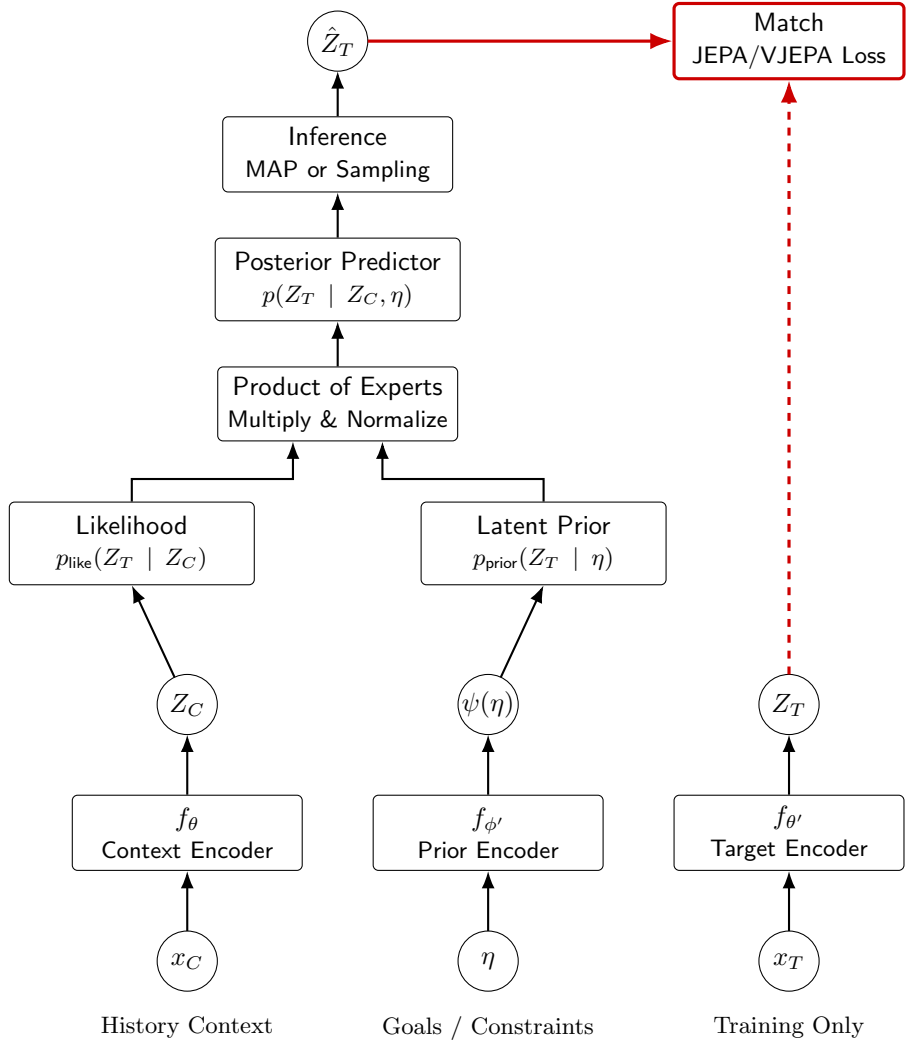

\paragraph{VJEPA as a special case of BJEPA}
VJEPA is recovered as a special case of BJEPA by choosing a uniform (uninformative / vague) prior:
\[
p_{\text{prior}}(Z_T \mid \eta) \propto 1.
\]
In this case, the Bayesian posterior predictor becomes
\[
p(Z_T \mid Z_C, \eta) = p_{\text{like}}(Z_T \mid Z_C),
\]
which reduces exactly to the VJEPA predictive distribution. Thus, BJEPA generalizes VJEPA without altering its training objective.

\subsection{Training and Inference}
\label{sec:bjepa_training_inference}
BJEPA allows for \emph{disentangled learning}: the dynamics (what is possible) can be learned from large-scale unlabeled data, while the prior (what is desired) can be specified or learned separately.

\paragraph{Representation Learning: Training with Structural Priors.}
The primary training phase focuses on the context encoder $f_\theta$ and the predictive likelihood expert $p_{\text{like}}(Z_T \mid Z_C)$. While we can assume an uninformative prior\footnote{We can assume an uninformative (uniform) prior, $p_{\text{prior}}(Z_T \mid \eta) \propto 1$, allowing the model to learn pure latent dynamics driven solely by observational data. Note that, in a modular architecture such as BJEPA, the training of the prior (the "constraint expert") is often distinct from the training of the dynamics (the "likelihood expert"). If they were to be trained jointly end-to-end, the dynamics model might overfit to specific tasks rather than learning general world physics.} to learn pure data-driven dynamics, BJEPA also allows us to inject \emph{structural priors} $p_{\text{struct}}(Z_T)$ during training. These priors can enforce desirable properties, such as stationarity, sparsity, or slowness, acting as an information bottleneck that filters out high-entropy nuisance variables (e.g. the "Noisy TV" distractors in later experiment).

The objective minimizes the variational negative log-likelihood of the target representations provided by the target encoder $f_{\theta'}$ (which serves as the ``teacher'' via EMA). The target encoder $f_{\theta'}$ provides stable regression targets (the ``ground truth'' latent states) derived from the data, exactly as in the standard VJEPA formulation. We optimize the parameters of the context encoder and the likelihood expert to maximize the predictive likelihood of the target representations while optionally regularizing the dynamics against the structural prior. Adapting Eq.~\eqref{eq:vjepa_objective}, the general BJEPA training objective is:
{\small
\begin{equation}
\label{eq:bjepa_training_objective}
\mathcal{L}_{\text{BJEPA}}
=
\mathbb{E}_{x} \left[
\underbrace{
\mathbb{E}_{Z_T \sim q_{\theta'}(\cdot \mid x_T)}
\Big[ -\log p_{\text{like}}(Z_T \mid Z_C) \Big]
}_{\text{Maximize Predictive Likelihood}}
+
\beta
\underbrace{
\mathrm{KL} \left(q_{\theta'}(Z_T \mid x_T)\,\|\,p_{\text{ref}}(Z)\right)
}_{\text{Regularize Information State}}
+
\gamma
\underbrace{
\mathrm{KL} \left(p_{\text{like}}(Z_T \mid Z_C)\,\|\,p_{\text{struct}}(Z_T)\right)
}_{\text{Enforce Structural Constraints}}
\right].
\end{equation}
}
\begin{itemize}
    \item \textbf{Likelihood ($\mathcal{L}_{\text{VJEPA}}$):} The first two terms (controlled by $\beta$) constitute the standard VJEPA objective (Eq.\ref{eq:vjepa_objective}). They ensure the learned representation captures the predictable mutual information between the past and future.
    \item \textbf{Structural Regularization ($\gamma$):} The third term penalizes the model if the "physically feasible" distribution ($p_{like}$) deviates too far from the "task compliant" distribution ($p_{prior}$), therefore forces the predicted dynamics to adhere to the structural prior. In our "Noisy TV" experiments (Section~\ref{sec:toy_experiment}), we use a \emph{static prior} ($p_{\text{struct}} \approx \text{const}$), which penalizes the model for tracking drifting, non-stationary noise, effectively filtering out distractors that violate the prior's structure.
\end{itemize}

The target encoder parameters are updated asymmetrically: the mean parameters $\mu_{\theta'}$ are updated via EMA as per Eq.~\eqref{eq:target_encoder_weights_EMA}, while the variance parameters $\sigma^2_{\theta'}$ are learned directly via gradient descent on the KL regularization term (the second term in Eq.~\eqref{eq:bjepa_training_objective}). This allows the model to learn appropriate uncertainty estimates for the target representations while maintaining stability through the EMA-updated mean. Alternatively, both the mean and variance parameters of the target encoder can be learned, or set to be constant. $p_{\text{like}}$ can capture the multi-modal distribution of valid future states (the ``physics'' of the latent space) without collapsing, thanks to the reference prior $p_{\text{ref}}$ (typically a unit Gaussian) and the asymmetry of the target encoder.

Note that we explicitly maximize the log-likelihood of the \emph{dynamics expert} $p_{\text{like}}(Z_T \mid Z_C)$ rather than the full posterior $p(Z_T \mid Z_C, \eta)$ (Eq.~\eqref{eq:bjepa_posterior}). This factorization is deliberate: the likelihood expert is responsible for learning how the world \emph{actually} evolves (data-driven physics), while the prior expert encodes how we \emph{expect} or \emph{want} the world to behave (constraints). Optimizing the full, fused posterior directly against the ground truth $Z_T$ can be risky, as it would force the model to pretend that the data satisfies our arbitrary constraints, for example, ignoring "Noisy TV" drift because the prior enforces stationarity (see later the experiment section). Instead, we require the model to fit the observational data (via the negative log-likelihood term) while being penalized for deviating from the structural expectations (via the KL regularization term). This ensures accurate physical modeling while filtering out nuisance variables that violate the structural prior.

\paragraph{Inference with Task-Specific Priors.}
Once the dynamics $p_{\text{like}}$ are trained, we can introduce new, task-specific priors $p_{\text{task}}(Z_T \mid \eta)$ at \textit{inference} time (or fine-tuning) to guide the agent, without retraining the underlying physics model. Unlike the likelihood, the prior expert $p_{\text{prior}}(Z_T \mid \eta)$ encodes task-specific constraints or goals. It is specified or trained in one of two ways, depending on the nature of the auxiliary information $\eta$:

\begin{itemize}
    \item \textbf{Learned Goal Priors (e.g. image goals):} if $\eta$ represents a goal observation (e.g. an image of the desired outcome), the prior encoder can be trained to map $\eta$ to a distribution in latent space. A common and efficient strategy is to reuse the \emph{frozen} target encoder $f_{\theta'}$ or train a separate lightweight mapper to predict the representation of the goal, e.g.
    \[
    p_{\text{prior}}(Z_T \mid \eta_{\text{img}}) \approx \mathcal{N}(Z_T \mid f_{\theta'}(\eta_{\text{img}}), \sigma^2_{\text{goal}} I).
    \]
    This allows the agent to pursue visual goals without retraining the dynamics model.

    \item \textbf{Analytic / Energy-Based Priors:} for logical constraints or safety regions (e.g. ``avoid region $Z_{\text{unsafe}}$''), the prior can be specified analytically as an energy function:
    \[
    p_{\text{prior}}(Z_T \mid \eta) \propto e^{-E(Z_T; \eta)}.
    \]
    For example, a ``stay close to reference trajectory'' prior might be defined as a quadratic energy well centered on a reference path.
\end{itemize}

\paragraph{Inference: Planning via Bayesian Fusion.}
During planning, the target encoder is inactive\footnote{The target encoder is not required at inference time unless new observations $x_T$ are incorporated.} (as the future is unknown). Instead, we fuse the learned dynamics with the task-specified prior using the Product of Experts mechanism to obtain the BJEPA posterior:
\begin{enumerate}
    \item \textit{Likelihood Step:} the learned expert $p_{\text{like}}(Z_T \mid Z_C)$ predicts where the system \emph{will} go based on history $Z_C$.
    \item \textit{Prior Step:} the prior expert $p_{\text{prior}}(Z_T \mid \eta)$ specifies where the system \emph{should} go to satisfy auxiliary information $\eta$ (e.g. goals or safety constraints).
    \item \textit{Fusion:} the planner samples or optimizes against the product distribution:
    \[
    p(Z_T \mid Z_C, \eta) \propto p_{\text{like}}(Z_T \mid Z_C) \times p_{\text{prior}}(Z_T \mid \eta).
    \]
\end{enumerate}
This effectively restricts the search space to the intersection of dynamically feasible and task-compliant trajectories, enabling efficient latent-space planning without requiring the model to have seen the specific task $\eta$ during the representation learning phase.

The planner then samples or optimizes latent states $Z_T$ that maximize the joint posterior:
\[
\log p(Z_T \mid Z_C, \eta)  \cong  \log p_{\text{like}}(Z_T \mid Z_C) + \log p_{\text{prior}}(Z_T \mid \eta) + \text{const}.
\]
This effectively intersects the manifold of physically feasible futures (from the Likelihood) with the manifold of task-compliant futures (from the Prior), enabling zero-shot generalization to new tasks defined by $\eta$.

\vspace{0.5cm}
The distinction between the training and inference phases in BJEPA represents a transition from latent alignment to explicit manifold intersection. During training, the system employs \textit{soft fusion}, where the task prior $\eta$ acts as a regularizer via a KL divergence term; this biases the dynamics predictor to learn a world model that naturally favors task-compliant regions without strictly forcing the prediction to originate from the fused distribution $p(Z_T \mid Z_C, \eta)$. In contrast, the inference phase utilizes \textit{hard fusion} through the PoE mechanism, explicitly intersecting the manifold of physically feasible trajectories with the task-specific prior. This ensures that the final latent plan is strictly confined to the intersection of what the world model deems possible and what the task requires, enabling efficient, zero-shot generalization to new constraints.

\paragraph{MAP \textit{vs.} Sampling in the Posterior.}
Given the fused posterior in Eq.~\eqref{eq:bjepa_posterior}, prediction and planning may formally proceed via:
\begin{itemize}
    \item \textbf{MAP prediction:}
    \[
    \hat Z_T^{\text{MAP}}
    =
    \arg\max_{Z_T}
    \big[
    \log p_{\text{like}}(Z_T \mid Z_C)
    +
    \log p_{\text{prior}}(Z_T \mid \eta)
    \big];
    \]
    \item \textbf{Sampling-based prediction:}
    \[
    Z_T^{(m)} \sim p(Z_T \mid Z_C, \eta),
    \qquad m=1,\dots,M.
    \]
    with $M$ being the total number of samples.
\end{itemize}
MAP prediction is computationally efficient and sufficient for unimodal posteriors, while sampling becomes essential in multi-modal or risk-sensitive settings where preserving the full uncertainty profile is necessary for robust control.

\subsection{Sequential BJEPA as Latent Bayesian Filtering}
\label{sec:bjepa_filtering}

When applied sequentially over time, BJEPA naturally induces a recursive filtering mechanism in representation space. By treating the context $Z_t$ as the current state summary ($Z_t=f_\theta(x_{\le t})$ summarizes history) and the auxiliary information $\eta_{t+1}$ as a ``virtual observation'' or constraint for the next step, the Product of Experts formulation implements a Bayesian filter update.

\paragraph{The BJEPA Filter Recursion.}
Standard Bayesian filtering consists of a \emph{prediction step} (time update) via a transition model and a \emph{correction step} (measurement update) via an observation likelihood. BJEPA implements this structure entirely in latent space:
\begin{enumerate}
    \item \textit{Prediction (Time Update):} The predictive likelihood $p_{\text{like}}(Z_{t+1} \mid Z_t)$ propagates the state forward based on system dynamics.
    \item \textit{Correction (Constraint Update):} The latent prior $p_{\text{prior}}(Z_{t+1} \mid \eta_{t+1})$ acts as a measurement likelihood for the auxiliary task information $\eta_{t+1}$, effectively ``correcting'' the predicted state to satisfy the goal or constraint.
\end{enumerate}
Combining these via the PoE mechanism yields the recursive update for the belief over the future trajectory:
\begin{equation}
\label{eq:bjepa_filter}
p(Z_{t+1} \mid Z_t, \eta_{t+1})
 \propto 
\underbrace{p_{\text{like}}(Z_{t+1} \mid Z_t)}_{\text{Dynamics (Prediction)}}
 \times 
\underbrace{p_{\text{prior}}(Z_{t+1} \mid \eta_{t+1})}_{\text{Constraint (Correction)}}.
\end{equation}
This formulation mirrors Bayesian filtering, but avoids explicit observation likelihoods. It also unifies planning and filtering: planning becomes the task of filtering the state distribution conditioned on the evidence of ``optimality'' or goal achievement encoded by $\eta$.

\paragraph{Comparison to Standard Filtering.}
In a standard Kalman Filter or HMM, the correction comes from sensory data $x_{t+1}$ via $p(x_{t+1} \mid Z_{t+1})$. In BJEPA, the correction comes from the \emph{intention} or \emph{constraint} $\eta_{t+1}$ via $p_{\text{prior}}$. This allows the agent to reason about the ``posterior'' future that is consistent with both physics and goals, without needing to generate or score high-dimensional observations.

\begin{remark}[Belief-Space MPC]
This perspective allows BJEPA planning to be interpreted as \emph{Stochastic Model Predictive Control (MPC)} performed via belief propagation. The likelihood expert propagates the physically feasible set (the learned belief), while the prior expert injects goals or constraints and carves out the task-relevant subset. The resulting posterior $p(Z_{t+1} \mid Z_t, \eta_{t+1})$ represents the optimal control distribution, generalizing the notion of ``Control as Inference'' to latent representation spaces and yielding a belief-space MPC formulation without explicit state estimation.
\end{remark}

\subsection{MAP Optimality under BJEPA}

We now show that the sufficiency of MAP planning (Theorem~\ref{thm:map_optimal} of VJEPA) extends to the BJEPA architecture, provided the ``experts'' are well-behaved.

\begin{theorem}[Optimality of MAP Planning under BJEPA]
\label{thm:bjepa_map_optimal}
Consider a one-step planning problem where the belief is formed by a BJEPA posterior:
\[
p(Z_{t+1} \mid Z_t, \eta, u_t) \propto p_{\text{like}}(Z_{t+1} \mid Z_t, u_t) \times p_{\text{prior}}(Z_{t+1} \mid \eta).
\]
Assume:
\begin{enumerate}
    \item The likelihood expert is Gaussian with action-independent covariance:
    $p_{\text{like}} = \mathcal{N}(\mu_{\text{dyn}}(Z_t, u_t), \Sigma_{\text{dyn}})$.
    \item The prior expert is Gaussian and independent of the current action:
    $p_{\text{prior}} = \mathcal{N}(\mu_{\text{aux}}(\eta), \Sigma_{\text{aux}})$.
    \item The stage cost $c(Z_{t+1}, u_t)$ is quadratic in $Z_{t+1}$.
\end{enumerate}
Then the action minimizing the expected cost under the full Bayesian posterior is equivalently obtained by minimizing the cost at the posterior MAP state:
\[
u_t^\star = \arg\min_{u_t} c(\hat{Z}_{t+1}^{\text{MAP}}, u_t).
\]
\end{theorem}

\begin{proof}
Since the posterior is formed by the product of two Gaussians, it is itself a Gaussian:
\[
p(Z_{t+1} \mid \cdot) = \mathcal{N}(Z_{t+1} \mid \mu_{\text{post}}, \Sigma_{\text{post}}).
\]
The posterior precision (inverse covariance) is the sum of the expert precisions:
\[
\Sigma_{\text{post}}^{-1} = \Sigma_{\text{dyn}}^{-1} + \Sigma_{\text{aux}}^{-1}.
\]
Since both $\Sigma_{\text{dyn}}$ and $\Sigma_{\text{aux}}$ are independent of $u_t$, the resulting posterior covariance $\Sigma_{\text{post}}$ is also constant with respect to $u_t$. The expected quadratic cost decomposes\footnote{Similar to the proof to Theorem~\ref{thm:map_optimal} of VJEPA, see Appendix~\ref{app:map_optimal_proof}.} into a mean-dependent term and a trace term $\mathrm{tr}(Q\Sigma_{\text{post}})$. Since the trace term is constant with respect to the action, the optimization depends only on the posterior mean. For a Gaussian, the mean coincides with the MAP estimate.
\end{proof}

\subsection{Energy-Based Priors and Latent Optimization}
\label{sec:bjepa_ebm}

A key advantage of the Product of Experts formulation is flexibility: the prior $p_{\text{prior}}$ need not be a normalized probability density. It can be specified as an unnormalized \emph{energy-based model} (EBM):
\begin{equation}
p_{\text{prior}}(Z_T \mid \eta) \propto \exp\!\big(-E_{\phi'}(Z_T;\eta)\big),
\end{equation}
where $E_{\phi'}$ is a scalar energy function representing the ``cost'' of a state $Z_T$ given task $\eta$. This allows for complex, non-convex constraints (e.g. obstacle avoidance potentials or logical constraints) that are difficult to normalize.

Combining a Gaussian predictive likelihood $p_{\text{like}} = \mathcal{N}(\mu_\phi, \Sigma)$ with an energy-based prior yields a posterior whose MAP estimate corresponds to a \emph{latent-space energy minimization} problem:
\[
Z_T^{\mathrm{MAP}}
=
\arg\min_{Z_T}
\Big[
\underbrace{
\tfrac{1}{2}\|Z_T-\mu_\phi(Z_C)\|_{\Sigma^{-1}}^2
}_{\text{Dynamics Consistency}}
 + 
\underbrace{
E_{\phi'}(Z_T;\eta)
}_{\text{Task Constraint}}
\Big].
\]
This formulation reveals a deep connection to classical trajectory optimization and control:
\begin{itemize}
    \item The first term acts as a regularizer anchoring the solution to physically plausible futures (the learned dynamics).
    \item The second term acts as a task cost or potential field pulling the solution towards the goal.
\end{itemize}
Since all components in VJEPA/BJEPA are differentiable, this optimization can be solved efficiently via \textit{gradient descent} directly in the representation space, effectively performing \emph{planning as inference} without creating a separate discrete graph or search tree.

\subsection{Algorithm: Gradient-Based Planning with BJEPA}
\label{sec:bjepa_algo}

We synthesize the previous 3 theoretical components, i.e. Bayesian filtering, MAP optimality, and energy-based priors, into a practical control algorithm. Because the BJEPA posterior is fully differentiable, we can implement Model Predictive Control (MPC) via \textit{gradient descent in latent space}, this procedure is detailed in Algorithm~\ref{alg:bjepa_mpc}. At every time step, we optimize a sequence of actions such that the resulting latent trajectory maximizes the joint posterior probability defined by the learned dynamics and the task prior.

\begin{algorithm}[H]
\caption{BJEPA Gradient-Based MPC}
\label{alg:bjepa_mpc}
\begin{algorithmic}[1]
\REQUIRE Current history $x_{\le t}$, Task $\eta$, Horizon $H$
\REQUIRE Differentiable Dynamics $p_{\text{like}}$, Energy Prior $E_{\phi'}$
\STATE \textbf{Initialize:} Estimate current state $Z_t = f_\theta(x_{\le t})$
\STATE \textbf{Initialize:} Action sequence $u_{0:H-1}$ (e.g. random or warm-start)
\WHILE{optimization budget not exceeded}
    \STATE \textit{1. Forward Rollout (Dynamics Expert)}
    \STATE $\hat{Z}_0 \leftarrow Z_t$
    \FOR{$k=0$ \TO $H-1$}
        \STATE $\hat{Z}_{k+1} \leftarrow \mu_{\text{dyn}}(\hat{Z}_k, u_k)$ \COMMENT{Predict future mean}
    \ENDFOR
    
    \STATE \textit{2. Evaluate BJEPA Posterior (Energy)}
    \STATE $J \leftarrow 0$
    \FOR{$k=1$ \TO $H$}
        \STATE $J_{\text{prior}} \leftarrow E_{\phi'}(\hat{Z}_k; \eta)$ \COMMENT{Task Constraint Energy}
        \STATE $J \leftarrow J + J_{\text{prior}} + \lambda \|u_{k-1}\|^2$ \COMMENT{Add control cost}
    \ENDFOR
    
    \STATE \textit{3. Backward Pass (Optimization)}
    \STATE $\nabla_{u} J \leftarrow \text{Backpropagate}(J)$
    \STATE $u_{0:H-1} \leftarrow u_{0:H-1} - \alpha \nabla_{u} J$ \COMMENT{Gradient update}
\ENDWHILE
\STATE \textbf{Execute:} Apply $u_0$ to environment
\STATE \textbf{Recurse:} Receive $x_{t+1}$, update $Z_{t+1}$, repeat.
\end{algorithmic}
\end{algorithm}

This algorithm contrasts with standard trajectory optimization in that the ``cost function'' is not handcrafted; it is the learned energy prior $E_{\phi'}$. Furthermore, the ``dynamics constraints'' are provided by the learned likelihood expert $p_{\text{like}}$.

\subsection{Modularity, Transfer, and Continual Learning}
\label{sec:bjepa_transfer}

A fundamental advantage of the BJEPA factorization (Eq.~\ref{eq:bjepa_posterior}) is the \emph{semantic disentanglement} of environmental dynamics from task objectives. This modularity unlocks capabilities that are difficult to achieve in monolithic world models:

\paragraph{Zero-Shot Transfer via Prior Swapping.}
In standard reinforcement learning or world modeling (e.g. Dreamer, MuZero), the reward function is often entangled with the representation or dynamics. To change tasks, one must typically retrain or fine-tune the model weights. In BJEPA, the predictive likelihood $p_{\text{like}}$ learns ``objective-agnostic physics.'' Once trained, it can be transferred zero-shot to arbitrarily many new tasks simply by swapping the prior expert $p_{\text{prior}}(Z \mid \eta)$. The agent understands how the world evolves (Likelihood) independently of what it is currently trying to achieve (Prior). This 'reuse dynamics, adapt priors' paradigm facilitates zero-shot task transfer.

\paragraph{Continual Learning without Catastrophic Forgetting.}
Because task-specific information is injected at inference time via the prior (or the energy function), the dynamics model does not need to be updated to accommodate new goals. This mitigates catastrophic forgetting: the agent does not overwrite its knowledge of physics to learn a new task. The shared world model accumulates general knowledge, while task specificity is handled by lightweight, modular priors.

\vspace{0.5cm}
\paragraph{Summary.}
BJEPA generalizes JEPA-based world models by introducing an explicit `Bayesian posterior' in representation space. This enables uncertainty-aware planning, structured priors, belief propagation, and principled transfer and continual learning, while retaining JEPA’s core advantages: the avoidance of observation reconstruction, reward supervision, and autoregressive likelihood factorization.

\paragraph{Comparison with Existing WM Frameworks.}
Table.\ref{tab:world_model_comparison} contextualizes BJEPA within the landscape of latent world models. While methods such as \textit{Dreamer} handle uncertainty, they rely on computationally expensive observation reconstruction. Conversely, recent JEPA-based world models avoid reconstruction but typically lack explicit uncertainty quantification and rely on monolithic, learned reward functions. BJEPA uniquely occupies the intersection of these capabilities: it is likelihood-free (like JEPA) yet probabilistic (like Dreamer), and it replaces monolithic rewards with modular, energy-based priors. This specific combination allows for flexible, zero-shot task specification without retraining the underlying dynamics.

\begin{table}[H]
\centering
\caption{Comparison of latent world modeling frameworks. BJEPA is unique in combining likelihood-free learning with explicit uncertainty and modular task specification.}
\label{tab:world_model_comparison}
\begin{minipage}{\linewidth} 
    \centering
    \resizebox{\linewidth}{!}{
        \begin{tabular}{lccccc}
        \toprule
        Method & Obs.\ Reconstruction & Latent Uncertainty & Planning Space & Task Specification & Reward/Cost \\
        \midrule
        Dreamer & \checkmark & \checkmark & Latent & Monolithic & Learned \\
        JEPA-WM & \xmark & \xmark & Latent & Monolithic & Learned \\
        VJEPA (Ours) & \xmark & \checkmark & Latent & Implicit & None \\
        \textbf{BJEPA (Ours)} & \xmark & \checkmark & \textbf{Latent} & \textbf{Modular Prior} & \textbf{Energy / Prior} \\
        \bottomrule
        \end{tabular}}
    \vspace{1ex} 
    \begin{flushleft} 
    \scriptsize \textbf{Note on ``likelihood-free'':} In generative world models (e.g. Dreamer, VAEs), the model maximizes the likelihood of observed data $\max \log p_\psi(x \mid z)$, requiring a decoder to reconstruct every detail (e.g. texture, noise) from the latent state. In contrast, JEPA architectures are ``likelihood-free'' regarding observations because they never reconstruct the input. The loss operates entirely in representation space ($Z$), optimizing prediction of abstract features rather than pixel-level reconstruction.
    \end{flushleft}
\end{minipage}
\end{table}

\section{Toy Experiment: a ``Noisy TV'' Linear System}
\label{sec:toy_experiment}

We present an analytically tractable experiment designed to stress-test the \emph{nuisance invariance} property (Proposition~\ref{prop:nuisance_invariance}) of the JEPA family (JEPA/VJEPA/BJEPA). We construct an environment where observations contain a low-variance signal embedded in high-variance, yet predictable, nuisance noise (``distractors''). This setup mimics the ``Noisy-TV'' problem\footnote{The ``Noisy-TV'' problem originated as a thought experiment in Burda et al. \cite{burda2018noisyTV} regarding curiosity-driven reinforcement learning. It describes a failure mode where an agent, rewarded for seeking novelty or prediction error, becomes fixated on a source of uncontrollable, unpredictable noise (like a TV displaying static). Because the noise is inherently unpredictable, the agent continuously extracts high intrinsic rewards without making meaningful progress, effectively becoming a ``couch potato''.}, where a model is presented with a dominant source of high-entropy, task-irrelevant variability that threatens to overwhelm the signal representation. Implementation details can be found in Appendix.\ref{app:toy_experiment_details}.

\subsection{Environment Setup}

We define a Linear-Gaussian System where the observation dimension $D_x=20$ is significantly larger than the true state dimension $D_s=4$. The observation $x_t$ is composed of a \textit{signal} $s_t$ and a scaled \textit{distractor} $d_t$ projected through fixed matrices $C$ and $D$:
\[
x_t = C s_t + D (\sigma d_t) + \epsilon_t.
\]
The latent processes evolve as follows:
\begin{itemize}
    \item \textit{The Signal ($s_t$):} evolves according to a stable rotation, $s_{t+1} = A_{\text{rot}} s_t + w_t$, where $w_t \sim \mathcal{N}(0, 0.1^2 I)$. The orthogonal transition ensures the signal remains bounded and maintains constant variance, fixing the theoretical recoverability throughout the sequence.
    \item \textit{The Distractor ($d_t$):} evolves as a ``sticky'' random walk, $d_{t+1} = 0.9 d_t + v_t$. We scale the distractor by $\sigma \in [0, 8]$, allowing its variance to grow up to $\approx 64\times$ the signal variance. At $\sigma=8.0$, the SNR drops to $-2.2$ dB, effectively burying the signal in noise.
\end{itemize}

\subsection{Model Architectures and Evaluation Alignment}
\label{subsec:architectures}

We compare five models using a latent dimension $D_z = D_s = 4$. All architectures use linear transformations to match the environment. To ensure a fair comparison, the evaluation respects the temporal structure of each model:

\begin{itemize}
    \item \textit{VAE (Static):} acts as a static model where the latent $z_t$ represents the current state. It is evaluated against the current signal $s_t$.
    \item \textit{AR (Pixel-Predictive):} the bottleneck $z_t$ is optimized to predict future pixels $x_{t+1}$. Consequently, $z_t$ is evaluated on its ability to recover the \textit{future} signal $s_{t+1}$.
    \item \textit{JEPA/VJEPA (Latent-Predictive):} we extract the \textit{predictor output} (the explicit prediction of $Z_{t+1}$) for the linear probe. These are evaluated against $s_{t+1}$.
    \item \textit{BJEPA (Bayesian Fusion):} at inference, BJEPA uses \emph{hard fusion} (Product of Experts) to combine dynamics with a task prior. The resulting fused mean $\mu_{post}$ is evaluated against $s_{t+1}$.
\end{itemize}

\subsection{Results and Analysis}

Performance is measured via the coefficient of determination ($R^2$) using a linear probe fit on training data and evaluated on generalization test data.

\paragraph{Failure of Generative Models.}
As shown in Table~\ref{tab:toyExperimentResults} and Fig.\ref{fig:results_grid}, generative models (VAE, AR) suffer catastrophic degradation as noise increases. At scale $\sigma=8.0$, VAE test signal recovery drops to $R^2 \approx 0.50$ while its noise recovery remains high ($R^2 = 0.62$), confirming it has prioritized high-variance distractors. Similarly, the AR model drops to $R^2 = 0.578$.

\paragraph{Robustness of Joint-Embedding Architectures.}
All joint-embedding architectures (JEPA, VJEPA, BJEPA) demonstrate robustness to the ``Noisy TV'' distractor, maintaining $R^2 > 0.84$ even at scale 8.0. Deterministic JEPA achieved the highest accuracy ($R^2=0.93$), though it exhibited momentary instability at scale 3.0 ($R^2=0.841$). Probabilistic VJEPA and BJEPA showed superior training stability across seeds.

\paragraph{Qualitative Filtering.}
Fig.\ref{fig:reconstructions} illustrates the filtering effect. At Scale 8.0, VAE and AR reconstructions (dashed lines) track high-frequency noise. In contrast, BJEPA and VJEPA (solid lines) successfully act as latent filters, tracking the underlying true signal $s_{t+1}$ with high fidelity.

\begin{table}[H]
\centering
\footnotesize
\setlength{\tabcolsep}{5pt}
\renewcommand{\arraystretch}{1.05}
\begin{tabular}{@{} l l c c c @{}}
\toprule
Scale (SNR) & Model & $\uparrow$Signal $R^2$ (Tr/Te) & $\downarrow$Noise $R^2$ (Tr/Te) & $\downarrow$Time (Tr/Te) \\
\midrule
0.0 (inf dB)  & VAE   & \underline{\textbf{1.000}} / \underline{\textbf{1.000}} & NA / NA & 12.6s / 0.01s \\
              & AR    & 0.999 / 0.999      & NA / NA    & \underline{\textbf{6.4s}} / 0.01s \\
              & JEPA  & 0.947 / 0.930      & NA / NA    & 16.9s / 0.01s \\
              & VJEPA & 0.999 / 0.999      & NA / NA    & 13.9s / 0.01s \\
              & BJEPA & 0.987 / 0.981      & NA / NA    & 23.5s / 0.01s \\
\midrule
4.0 (3.8 dB)  & VAE   & 0.822 / 0.730      & 0.512 / 0.458    & 12.5s / 0.01s \\
              & AR    & 0.839 / 0.756      & 0.394 / 0.338    & \underline{\textbf{6.5s}} / 0.01s \\
              & JEPA  & \underline{\textbf{0.999}} / \underline{\textbf{0.999}}      & \underline{\textbf{0.004}} / \underline{\textbf{-0.010}}   & 16.6s / 0.02s \\
              & VJEPA & 0.994 / 0.993      & 0.025 / -0.007   & 13.8s / 0.01s \\
              & BJEPA & 0.920 / 0.899      & 0.213 / 0.156    & 23.2s / 0.01s \\
\midrule
8.0 (-2.2 dB) & VAE   & 0.613 / 0.499      & 0.656 / 0.620    & 12.3s / 0.01s \\
              & AR    & 0.680 / 0.578      & 0.491 / 0.449    & \underline{\textbf{7.1s}} / 0.01s \\
              & JEPA  & \underline{\textbf{0.947}} / \underline{\textbf{0.930}} & \underline{\textbf{0.226}} / \underline{\textbf{0.183}}    & 16.1s / 0.01s \\
              & VJEPA & 0.905 / 0.870      & 0.299 / 0.251    & 13.4s / 0.01s \\
              & BJEPA & 0.896 / 0.841      & 0.292 / 0.238    & 23.0s / 0.01s \\
\bottomrule
\end{tabular}
\caption{Performance metrics ($R^2$) across models at representative noise scales. $Tr$ denotes training set, $Te$ denotes test set. Full table is listed in Appendix.\ref{app:toy_experiment_details}.}
\label{tab:toyExperimentResults}
\end{table}

\begin{figure}[H]
\centering
\includegraphics[width=0.65\linewidth]{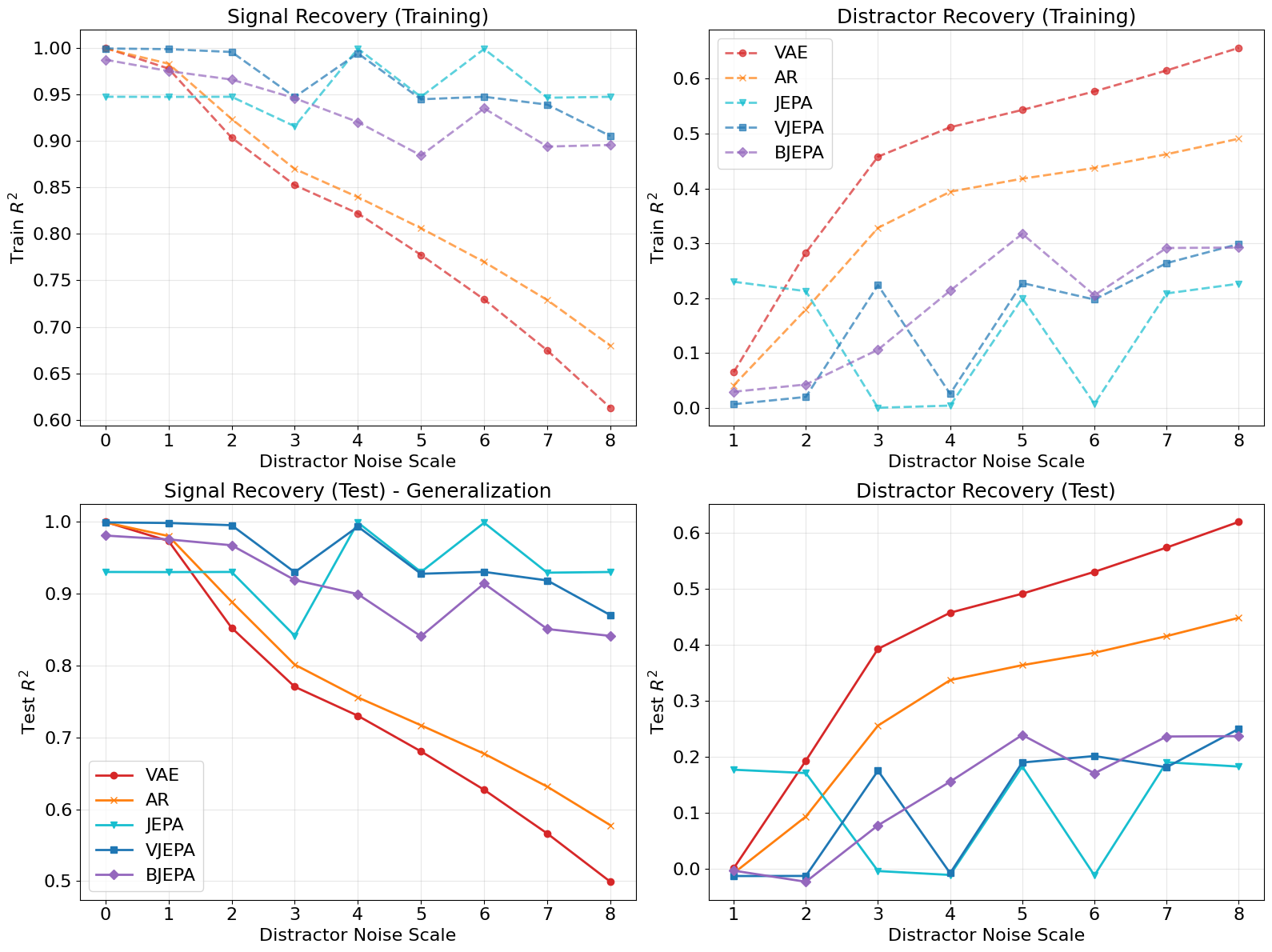} 
\caption{\textbf{Performance metrics across noise scales.} \textbf{Top Row:} Training set $R^2$. \textbf{Bottom Row:} Test set $R^2$ (Generalization). The generative models (VAE, AR) degrade linearly as noise increases, tracking the distractor (Bottom Right). The JEPA-based models (Blue/Cyan/Purple) maintain high signal recovery (Bottom Left) even at high noise scales, demonstrating invariance to nuisance variability.}
\label{fig:results_grid}
\end{figure}

\begin{figure}[H]
\centering
\includegraphics[width=1.0\linewidth]{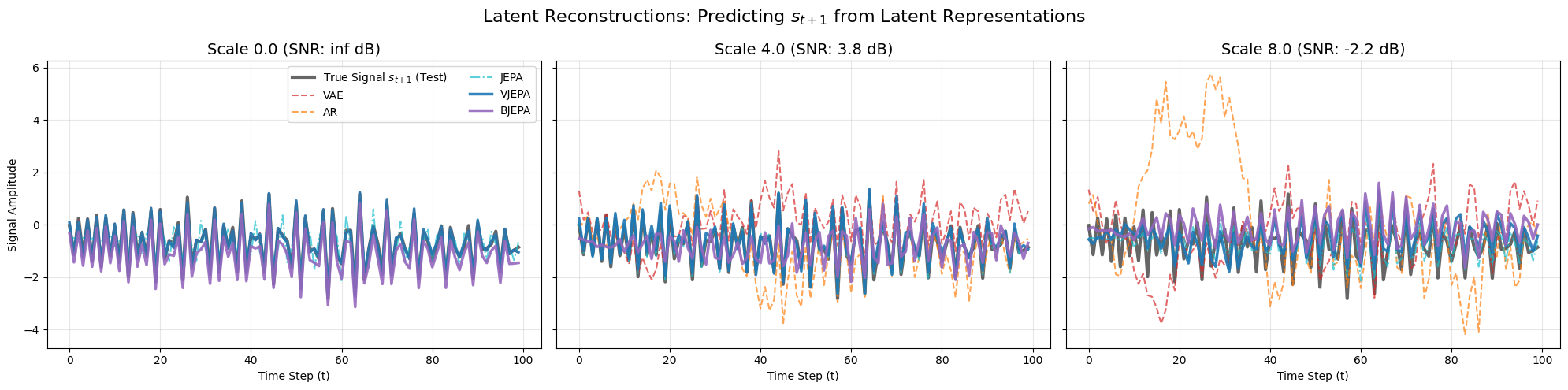} 
\caption{\textbf{Latent Reconstructions at varying noise scales.} At $\sigma=8.0$ (Right), the VAE and AR reconstructions (dashed lines) track the high-frequency noise. In contrast, BJEPA and VJEPA (solid lines) successfully filter the noise and track the underlying true signal (black line).}
\label{fig:reconstructions}
\end{figure}

\section{Discussion}
\label{sec:discussion}

In this work, we have formalized Joint Embedding Predictive Architectures (JEPA) not merely as a self-supervised training trick, but as a principled probabilistic framework for learning latent dynamical systems. By introducing \emph{Variational JEPA (VJEPA)} and its Bayesian extension \emph{BJEPA}, we bridged the gap between deterministic representation learning and probabilistic world models. Our analysis and experiments support several key conclusions regarding the nature of predictive learning.

\subsection{Unifying JEPA with Probabilistic State-Space Models}
JEPA is \textit{probably} best understood as a predictive state-space model trained by representation prediction rather than observation likelihood. A central contribution of this work is the decoupling of \emph{sequential modeling} from \emph{autoregressive observation likelihoods}.
Standard world models (e.g. Dreamer \cite{hafner2019dreamer}) and generative sequence models (e.g. Transformers \cite{vaswani2023attentionneed}) typically conflate the two, forcing the model to predict the next observation $x_{t+1}$ given history. This forces the latent state to account for high-entropy local noise.

VJEPA demonstrates that sequential structure does not harm representation learning; only autoregressive factorization over pixels does. By adopting a \emph{time-indexed} specialization (as in Section.\ref{sec:JEPA_as_dynamics} and Section.\ref{sec:jepa_control}), where contexts and targets are explicitly associated with temporal indices, VJEPA learns a mathematically rigorous \emph{latent} transition model $p_\phi(Z_{t+1}|Z_t)$. Notably, this transition kernel does not need to be autoregressive (as seen in Eq.\ref{eq:VJEPA_multi_step_prediction}) and is capable of belief propagation and filtering without ever estimating a density over pixels. This positions VJEPA as a scalable foundation for model-based control in high-dimensional, noisy environments.

\subsection{Nuisance Invariance: The ``PCA \textit{vs.} CCA'' Distinction}
Our "Noisy TV" experiment (Section~\ref{sec:toy_experiment}) provides empirical validation for the Information-Theoretic analysis in Section~\ref{sec:information_theoretic}. The failure of VAE and AR baselines on high-variance distractors highlights a fundamental distinction in representation learning objectives:
\begin{itemize}
    \item \textit{Generative Reconstruction (e.g. VAE):} These models implicitly perform a non-linear generalization\footnote{VAEs do not perform standard linear PCA by design, but they often implicitly discover the same principal directions (subspaces) as PCA while learning, particularly due to their regularization and training dynamics \cite{Michal2019vaePCA}.} of \emph{Principal Component Analysis (PCA)} \cite{Michal2019vaePCA}. They prioritize latent dimensions that explain the maximum variance in the input $x$. If nuisance noise has higher variance than the signal (as in our experiment), a reconstruction objective is mathematically compelled to model the noise.
    \item \textit{Predictive Association (VJEPA):} JEPA-based models implicitly perform a non-linear generalization of \emph{Canonical Correlation Analysis\footnote{CCA is a multivariate statistical method that finds the strongest linear relationships between two sets of variables, which creates new composite variables (canonical variates) from each set that are maximally correlated with each other \cite{weenink2003cca}.} (CCA \cite{weenink2003cca})} or the \emph{Information Bottleneck}. As established in Theorem~\ref{thm:mi_bound} (Section~\ref{sec:information_theoretic}), minimizing the VJEPA loss maximizes a lower bound on the mutual information between the past and the future ($I(Z_t; Z_{t+\Delta})$). Therefore, VJEPA (implicitly CCA) prioritizes dimensions with the highest predictive mutual information. It will ignore the noise regardless of its magnitude. This "nuisance invariance" property (Proposition~\ref{prop:nuisance_invariance}) makes VJEPA as a \emph{semantic noise filter} for downstream planning.
\end{itemize}

\subsection{Applications and Broader Impact}
\label{sec:applications}

While this work focuses on theoretical foundations, the VJEPA framework enables new capabilities in applied domains where state estimation and planning are required but observation reconstruction is costly or unnecessary.

\paragraph{Robotics and Embodied AI.}
In visual robotic control, $x_t$ represents high-dimensional sensory streams (e.g. camera feeds, lidar). Standard world models (e.g. Dreamer \cite{hafner2019dreamer}) often spend significant capacity modeling local textures and camera noise. VJEPA allows the agent to learn a \emph{task-specific physics engine} in latent space. For a manipulation task, the context $Z_t$ captures the geometry of objects, while the predictive model $p_\phi(Z_{t+1} \mid Z_t, u_t)$ predicts the consequences of motor commands. Planning can be performed via the \textit{VJEPA-MPC} algorithm (Algo.\ref{alg:jmpc}) to minimize a latent distance to a target goal embedding, enabling efficient visual control without pixel-level prediction. Furthermore, the BJEPA extension allows for \emph{zero-shot transfer}: a robot can learn the environmental dynamics (Likelihood Expert) once, and then swap in different goal constraints (Prior Expert) at inference time to solve varied tasks without retraining the underlying model.

\paragraph{Game Playing and Strategic Reasoning.}
In high-dimensional game environments (e.g. \textit{Atari}, \textit{Minecraft}), VJEPA functions as a self-supervised latent simulator. Unlike \textit{MuZero} \cite{Schrittwieser2020muzero}, which couples representation learning with a specific reward signal, VJEPA learns a reward-agnostic model of the game mechanics. The predictive model $p_\phi(Z_{t+1} \mid Z_t, u_t)$ captures the transition rules and physics of the game world. This enables planning via Monte Carlo Tree Search (MCTS \cite{MCTS_wiechowski_2022}) or trajectory optimization in the latent space, allowing the agent to reason about long-term strategies and effectively transfer the learned dynamics to new in-game tasks or modified rule sets without retraining the visual encoder.

\paragraph{Language-as-Action Planning.}
Current LLMs typically operate by minimizing the surface-level entropy of the next token. VJEPA offers an alternative view where language generation can be framed as \emph{planning in semantic space}. Here, the `world' is informational or symbolic, and the latent state $Z_t$ represents the current discourse status or reasoning state.
\begin{itemize}
    \item \textbf{Control Inputs:} The tokens (or chunks of text) act as discrete actions driving the system ($u_t \in \mathcal{V}$). $u_t$ can represent both generated content and user inputs such as prompts, which function as exogenous control signals that steer the initial semantic trajectory.
    \item \textbf{Latent Dynamics:} The model learns how choosing a word $u_t$ updates the semantic state:
    \[ Z_{t+1} \sim p_\phi(Z_{t+1} \mid Z_t, u_t). \]
\end{itemize}
In this framework, a response is not evaluated by its likelihood $\prod p(u_t \mid u_{<t})$, but by its \emph{effect on the latent world}. Generation becomes a trajectory optimization problem: identifying a sequence of tokens $u_{t:T}$ that drives the initial thought $Z_t$ to a target state $Z_{\text{goal}}$ (e.g. a state satisfying a logical constraint or answer condition). This allows text generation to leverage planning algorithms such as Beam Search \cite{meister2020BeamSearch}, Monte Carlo Tree Search (MCTS \cite{MCTS_wiechowski_2022}), or latent-space Cross-Entropy Method (CEM \cite{mannor2003CEM}), which decouples reasoning from surface-form statistics and potentially helps address hallucinations and long-horizon coherence.

\subsection{Limitations and Future Directions}
While VJEPA provides a rigorous probabilistic footing, our experiments revealed specific trade-offs:

\paragraph{Expressivity of the Predictive Distribution.}
In our experiments, the probabilistic VJEPA sometimes slightly underperformed the deterministic JEPA in peak accuracy, or required more careful hyperparameter tuning. We attribute this to the unimodality of the Gaussian assumption used in our implementation of $p_\phi$. In environments with complex, multi-modal bifurcations (e.g. a robot encountering an obstacle and choosing left or right), a unimodal Gaussian may average distinct modes, resulting in a blurry or incorrect belief. Future work should explore more expressive predictive heads, such as \textit{Gaussian Mixture Models (GMMs)} \cite{huang2025GMAsampling} or \textit{Latent Diffusion heads}, to capture complex aleatoric uncertainty while retaining the VJEPA objective.

\paragraph{Optimization Dynamics.}
We observed that the balance between the predictive loss and the KL regularization terms (controlled by $\beta$ in the VJEPA objective Eq.~\ref{eq:vjepa_objective}, and $\gamma$ in the BJEPA objective Eq.~\ref{eq:bjepa_training_objective}) is important. If regularization is too strong, the representation collapses; if too weak, the variance estimates become uncalibrated. Developing adaptive mechanisms for balancing these terms, similar to KL-balancing in latent overshooting, is a promising direction.

\section{Conclusion} \label{sec:conclusion}

This work introduces \emph{Variational JEPA (VJEPA)} and \emph{Bayesian JEPA (BJEPA)}, establishing the first formal probabilistic framework for Joint Embedding Predictive Architectures. By recasting JEPA as a variational inference problem in representation space, we resolve the ambiguity\footnote{Standard JEPA trains a deterministic predictor using Mean Squared Error. In stochastic environments with multimodal futures (e.g. a fork in the road), such a predictor converges to the conditional mean, which may be a physically invalid state (e.g. the average of `go left' and `go right' is `hit the wall'). This creates ambiguity regarding the semantics of the embedding: does it represent a specific future or an average? VJEPA resolves this by explicitly modeling the predictive distribution. Further, the variational derivation provides the theoretical guarantee that the learned state maximizes predictive mutual information, a necessary condition for it to serve as a sufficient information state for control.} regarding the probabilistic semantics of predictive embeddings and their suitability for control.

\paragraph{VJEPA: A Probabilistic Foundation.}
VJEPA generalizes standard deterministic JEPA by learning an explicit predictive distribution over future latent states. We showed that minimizing the VJEPA objective maximizes a variational lower bound on the predictive mutual information between the past and the future. Unlike generative world models, VJEPA achieves this without an autoregressive observation likelihood, effectively filtering out high-entropy nuisance variables (the "Noisy TV" problem) while retaining the information necessary for dynamics modeling. A key practical advantage of this formulation is the ability to perform principled uncertainty estimation (e.g. constructing credible intervals) by drawing multiple samples ($M \gg 1$) from the predictive distribution during inference \footnote{We distinguish between two sampling regimes using $K$ and $M$ to denote total number of samples. \textit{First. Training Sampling ($K=1$):} drawn from the target encoder distribution $q_{\theta'}(\cdot \mid x_T)$ during training to generate a concrete regression target. A single sample is sufficient because the stochasticity of mini-batch updates averages out the noise over time. \textit{Second. Inference Sampling ($M \gg 1$):} drawn from the predictor $p_\phi(\cdot \mid Z_C)$ during testing or planning (inference) to estimate uncertainty. Multiple samples are required here to accurately map the distribution's shape and construct credible intervals. This is also a key advantage of VJEPA/BJEPA over deterministic JEPA.}, a capability absent in deterministic JEPA.

\paragraph{BJEPA: Modular and Structural Control.}
We extended this VJEPA framework to BJEPA, which factorizes prediction into a learned \emph{Likelihood Expert} (dynamics) and a \emph{Prior Expert} (constraints). This modularity allows for the injection of structural priors, such as stationarity or goal-directedness, directly into the latent inference process via Product of Experts which bridges the gap between "what is physically possible" and "what is task-relevant". This separation enables zero-shot task transfer and robust planning in the presence of non-stationary noise.

\paragraph{Theoretical and Practical Unification.}
Our analysis unifies JEPA with classical ideas from Predictive State Representations (PSRs) and Bayesian filtering. We established that observation reconstruction is theoretically unnecessary for optimal control, provided the latent state captures sufficient predictive mutual information. Furthermore, we demonstrated that:
\begin{itemize}
    \item \textit{Sequential modeling does not imply autoregression:} one can learn rigorous temporal dynamics $p(Z_{t+1}|Z_t)$ without the computational burden of factorizing the high-dimensional observation density $p(x_{t+1}|x_{\le t})$.
    \item \textit{Predictive representations suffice for control:} we proved that under standard conditions, an optimal policy can be derived solely from the predictive latent state, rendering pixel reconstruction redundant for planning.
    \item \textit{Likelihood-free world models can be principled:} VJEPA breaks the historical conflation of probabilistic rigor with pixel-level likelihoods. It replaces heuristic regression losses with a formal variational framework operating entirely in latent space, ensuring principled uncertainty estimation and information maximization without modeling nuisance variability.
\end{itemize}

By combining the representational efficiency of JEPA with the rigor of probabilistic state-space models, VJEPA positions itself as a foundational framework for scalable, robust, uncertainty-aware world modeling and planning in high-dimensional, possibly noisy, real-world environments.

\bibliographystyle{plain}
\bibliography{reference}

\appendix
\section{Equivalence Between Squared Loss and Gaussian Likelihood}
\label{app:lsq_equivalence_with_MLE}

In this appendix, we show that minimizing the JEPA regression loss in Eq.~\eqref{eq:JEPA_squared_loss} is equivalent to maximizing the log-likelihood of the target embedding under an isotropic Gaussian predictive model (Eq.~\eqref{eq:JEPA_loss_prob_explain}).

Consider the probabilistic model assumed in Eq.~\eqref{eq:JEPA_loss_prob_explain}:
\[
p(Z_T \mid Z_C)
=
\mathcal{N} \left(
Z_T  \middle| 
\hat{Z}_T,  \sigma^2 I
\right),
\quad
\hat{Z}_T = g_\phi(Z_C, \xi_T).
\]
The corresponding negative log-likelihood is given by
\[
\begin{aligned}
- \log p(Z_T \mid Z_C)
&=
- \log \left[
\frac{1}{(2\pi\sigma^2)^{d/2}}
\exp \left(
- \frac{1}{2\sigma^2}
\| Z_T - \hat{Z}_T \|^2
\right)
\right] \\
&=
\frac{1}{2\sigma^2}
\| Z_T - \hat{Z}_T \|^2
+ \frac{d}{2} \log(2\pi\sigma^2),
\end{aligned}
\]
where \(d\) is the dimensionality of the embedding space.

Since the second term is constant with respect to the model parameters \((\theta, \phi)\), minimizing the negative log-likelihood is equivalent to minimizing
\[
\| Z_T - \hat{Z}_T \|^2.
\]
Thus, minimizing the deterministic JEPA objective 
$
\mathcal{L}_{\text{JEPA}} = \| \hat{Z}_T - Z_T \|^2
$
can be translated as maximum likelihood estimation under an isotropic Gaussian predictive distribution with fixed variance \(\sigma^2\).

This observation motivates the variational extension introduced in the main text, where the predictive distribution \(p(Z_T \mid Z_C)\) is modeled explicitly and the variance (or more general uncertainty structure) is learned rather than fixed.

\section{Relationships Between Information-Theoretic Quantities}
\label{app:info_theory_identities}

In this section, we summarize the fundamental relationships between entropy, cross-entropy, Kullback-Leibler (KL) divergence, and mutual information. These identities underpin the theoretical analysis of VJEPA and BJEPA. For further info, readers are encouraged to refer to standard texts such as \cite{mackay2002information}. Here, let's denote $p(x)$ as the true (or target) data distribution and $q(x)$ the approximating (or predicted) distribution.

\subsection{Cross-Entropy, Entropy, and KL Divergence}
Cross-entropy is the sum of the true entropy and the KL divergence. Intuitively, it represents the average number of bits needed to encode data from $p$ using a code optimized for $q$, which is the optimal bits (entropy of $p$) plus the penalty for using the wrong code (KL divergence):
\begin{equation}
\label{eq:cross_entropy_identity}
H(p, q) = H(p) + D_{\text{KL}}(p \,\|\, q),
\end{equation}
where the terms are defined as:
\begin{itemize}
    \item \textbf{Cross-Entropy:} $H(p, q) = -\sum_{x} p(x) \log q(x)$
    \item \textbf{Entropy:} $H(p) = -\sum_{x} p(x) \log p(x)$
    \item \textbf{KL Divergence:} $D_{\text{KL}}(p \,\|\, q) = \sum_{x} p(x) \log \frac{p(x)}{q(x)}$
\end{itemize}

\subsection{Mutual Information and Entropy}
Mutual information (MI) measures the reduction in uncertainty (entropy) of one variable given knowledge of another. It satisfies the following additive identities:
\begin{align}
I(X; Y) &= H(X) - H(X \mid Y) \\
I(X; Y) &= H(Y) - H(Y \mid X) \\
I(X; Y) &= H(X) + H(Y) - H(X, Y)
\end{align}
where $H(X \mid Y)$ is the conditional entropy. We used the first formula later in our proof of the non-collapsed global optimum associated with the VJEPA objective (Theorem~\ref{thm:collapse_avoidance}, and its detailed proof in Appendix.\ref{app:collapse_avoidance_proof}).

\subsection{Chain Rule for Mutual Information}
The chain rule for mutual information describes how the information provided by a pair of variables $(X, Y)$ about a third variable $Z$ can be decomposed into two distinct stages. The step-by-step expansion is as follows:

\begin{enumerate}
    \item \textit{Definition of Mutual Information:}
    The mutual information between $(X, Y)$ and $Z$ is the difference between the marginal entropy of $Z$ and the conditional entropy of $Z$ given $(X, Y)$:
    \begin{equation}
    I(X, Y; Z) = H(Z) - H(Z \mid X, Y)
    \end{equation}
    
    \item \textit{Decomposition of Conditional Entropy:}
    Using the identity for conditional entropy, we can express $H(Z \mid X, Y)$ by first conditioning on $X$:
    \begin{equation}
    H(Z \mid X, Y) = H(Z \mid X) - I(Y; Z \mid X)
    \end{equation}
    Rearranging this reveals that $H(Z \mid X, Y)$ represents the remaining uncertainty in $Z$ after observing both $X$ and $Y$.

    \item \textit{Substitution and Re-grouping:}
    Substituting the decomposition into the original definition:
    \begin{align}
    I(X, Y; Z) &= H(Z) - [H(Z \mid X) - I(Y; Z \mid X)] \\
    I(X, Y; Z) &= [H(Z) - H(Z \mid X)] + I(Y; Z \mid X)
    \end{align}

    \item \textit{Final Chain Rule Form:}
    Recognizing that $H(Z) - H(Z \mid X)$ is the definition of $I(X; Z)$, we arrive at the chain rule:
    \begin{equation}
    I(X, Y; Z) = I(X; Z) + I(Y; Z \mid X)
    \end{equation}
\end{enumerate}

\subsection{Symmetry Properties}
A crucial distinction between these metrics is their behavior under argument swapping:

\begin{itemize}
    \item \textbf{Mutual Information is Symmetric:}
    \[ I(X; Y) = I(Y; X). \]
    Intuitively, the amount of information $X$ provides about $Y$ is identical to the information $Y$ provides about $X$. This symmetry arises from the identity $I(X;Y) = H(X) + H(Y) - H(X,Y)$, where the joint entropy $H(X,Y)$ is invariant to ordering.

    \item \textbf{KL Divergence is Asymmetric:}
    \[ D_{\text{KL}}(p \,\|\, q) \neq D_{\text{KL}}(q \,\|\, p) \quad (\text{generally}). \]
    Minimizing $D_{\text{KL}}(p \,\|\, q)$ (forward KL) encourages the approximation $q$ to cover the support of $p$ (mode-covering), whereas minimizing $D_{\text{KL}}(q \,\|\, p)$ (reverse KL) encourages $q$ to seek a single mode of $p$ (mode-seeking).

    \item \textbf{Cross-Entropy is Asymmetric:}
    \[ H(p, q) \neq H(q, p). \]
    Since $H(p, q) = H(p) + D_{\text{KL}}(p \,\|\, q)$, the asymmetry of the KL divergence imparts asymmetry to the cross-entropy.
\end{itemize}

\paragraph{Common Confusion: Cross-Entropy vs. Joint Entropy.}
It is easy to confuse Cross-Entropy $H(p,q)$ with Joint Entropy $H(X,Y)$.
\begin{itemize}
    \item \textbf{Cross-Entropy} $H(p,q)$ is defined between two \emph{probability distributions} $p$ and $q$:
    \[ H(p, q) = -\sum_{x} p(x) \log q(x). \]
    It is \textbf{NOT} symmetric ($H(p,q) \neq H(q,p)$) because the roles of the weighting distribution $p(x)$ and the log-term $q(x)$ are distinct.
    \item \textbf{Joint Entropy} $H(X,Y)$ is defined over two \emph{random variables} $X$ and $Y$:
    \[ H(X, Y) = - \sum_{x} \sum_{y} p(x, y) \log p(x, y). \]
    It \textbf{IS} symmetric ($H(X,Y)=H(Y,X)$) because the joint probability $p(x,y)$ is invariant to ordering.
\end{itemize}

\subsection{Evidence Lower Bound (ELBO)}
In variational inference, we often aim to approximate an intractable posterior $p(z \mid x)$ using a simpler distribution $q(z \mid x)$. The Evidence Lower Bound (ELBO) provides a tractable lower bound on the log-likelihood of the data (the evidence) $\log p(x)$. It relates the log-evidence to the KL divergence between the approximate and true posteriors:
\begin{equation}
\log p(x) = \mathcal{L}_{\text{ELBO}}(q) + D_{\text{KL}}(q(z \mid x) \,\|\, p(z \mid x)),
\end{equation}
where the ELBO is defined as:
\begin{equation}
\mathcal{L}_{\text{ELBO}}(q) = \mathbb{E}_{q(z \mid x)}[\log p(x, z) - \log q(z \mid x)].
\end{equation}
Since the KL divergence is non-negative ($D_{\text{KL}} \ge 0$), the ELBO serves as a valid lower bound on the evidence:
\begin{equation}
\log p(x) \ge \mathcal{L}_{\text{ELBO}}(q).
\end{equation}
Maximizing the ELBO is therefore equivalent to minimizing the KL divergence between the approximate posterior $q(z \mid x)$ and the true posterior $p(z \mid x)$, effectively pushing the approximation towards the true distribution.

\subsection{Non-Negativity Properties}
An important property of these quantities is their non-negativity (for discrete variables), which ensures they act as valid objective functions or distance/divergence metrics:

\begin{enumerate}
    \item \textbf{Entropy:} $H(X) \ge 0$. It is zero if and only if $X$ is deterministic.
    \item \textbf{KL Divergence:} $D_{\text{KL}}(p \| q) \ge 0$. This is known as Gibbs' inequality. It is zero if and only if $p = q$ almost everywhere.
    \item \textbf{Mutual Information:} $I(X; Y) \ge 0$. It is zero if and only if $X$ and $Y$ are independent.
    \item \textbf{Cross-Entropy:} $H(p, q) \ge H(p)$. Since $D_{\text{KL}} \ge 0$, the cross-entropy is always lower-bounded by the entropy of the true distribution.
\end{enumerate}

\subsection{Summary of Relationships}
Table.\ref{tab:info_theory_summary} provides a quick reference for these relationships.

\begin{table}[H]
\centering
\caption{Summary of relationships and properties of information-theoretic quantities.}
\label{tab:info_theory_summary}
\begin{tabular}{llcc}
\toprule
\textbf{Quantity} & \textbf{Relation to others} & \textbf{Non-negative?} & \textbf{Symmetric?} \\
\midrule
Cross-Entropy ($H(p,q)$) & $H(p) + D_{\text{KL}}(p \| q)$ & Yes ($\ge H(p)$) & No \\[0.5em]
KL Divergence ($D_{\text{KL}}$) & $H(p,q) - H(p)$ & Yes & No \\[0.5em]
Mutual Information ($I(X;Y)$) & $H(X) - H(X \mid Y)$ & Yes & \textbf{Yes} \\[0.5em]
Mutual Information ($I(X;Y)$) & $D_{\text{KL}}( p(x, y) \| p(x)p(y) )$ & Yes & \textbf{Yes} \\
\bottomrule
\end{tabular}
\end{table}

\section{Compare VJEPA with Hidden Markov Models (HMMs)}
\label{app:hmmlink}

VJEPA shares a structural resemblance to a Hidden Markov Model (HMM) in that both posit a latent state sequence to explain temporal data. However, the resemblance is only partial: classical HMMs are \emph{generative probabilistic models of observations}, whereas VJEPA is a \emph{representation-predictive} model whose training objective operates entirely in the latent space. Here we clarify the similarities and differences.

\subsection{Classical HMMs}

An HMM defines a discrete-time latent Markov process $\{s_t\}_{t\ge 1}$ and a sequence of observations $\{x_t\}_{t\ge 1}$ via the joint factorization:
\begin{equation}
\label{eq:hmm_factorization}
p(s_{1:T}, x_{1:T})
=
p(s_1)\prod_{t=2}^T p(s_t \mid s_{t-1}) \prod_{t=1}^T p(x_t \mid s_t),
\end{equation}
where $p(s_t \mid s_{t-1})$ is the transition model and $p(x_t \mid s_t)$ is the emission (observation) model.
Learning\footnote{
A (discrete) HMM assumes an unobserved state sequence
$s_{1:T}$ where each $s_t \in \{1,\dots,N\}$ (the state space) and an observed sequence
$x_{1:T}$ where each $x_t \in \{v_1,\dots,v_M\}$ (the observation symbols).
It is parameterized by $\lambda=(A,B,\pi)$, where:
$A=[a_{ij}]$ with $a_{ij}=\Pr(s_t=j \mid s_{t-1}=i)$ (transitions);
$B=[b_j(k)]$ with $b_j(k)=\Pr(x_t=v_k \mid s_t=j)$ (emissions); and
$\pi=[\pi_i]$ with $\pi_i=\Pr(s_1=i)$ (initial distribution).
The three canonical tasks are:
(i) \emph{evaluation}: compute $\Pr(x_{1:T}\mid\lambda)$ (e.g. via the Forward algorithm);
(ii) \emph{decoding}: find the most likely hidden state sequence $\arg\max_{s_{1:T}} \Pr(s_{1:T} \mid x_{1:T}, \lambda)$ (e.g. via the Viterbi algorithm);
and (iii) \emph{learning}: estimate $\lambda$ from data (e.g. via the Baum-Welch/EM algorithm).
Details can be found in standard course materials, e.g. \textit{Computational Intelligence}, which the author taught at University of Aberdeen.}
typically proceeds by maximum likelihood (e.g. Expectation-Maximization), and inference computes the filtering posterior $p(s_t \mid x_{\le t})$ or smoothing posterior $p(s_t \mid x_{1:T})$.

In modern state-space model variants (SSMs), $s_t$ may be continuous and the dynamics nonlinear. However, the defining property remains: the model is trained to explain \emph{observations} through an explicit hidden-to-observation likelihood $p(x_t \mid s_t)$.

\subsection{VJEPA as ``Latent-Only'' Predictive Modeling}

VJEPA is not an observation-likelihood model. As detailed in Section~\ref{sec:vjepa}, VJEPA introduces a latent representation for the target region and learns a predictive distribution (Eq.~\ref{eq:vjepa_predictive_model}):
\[
p_\phi(Z_T \mid Z_C, \xi_T),
\]
matched against an amortized target-encoder ``inference'' distribution (Eq.~\ref{eq:vjepa_inference_target}):
\begin{equation}
q_{\theta'}(Z_T \mid x_T).
\end{equation}
Training minimizes a representation-space negative log-likelihood plus regularization (Eq.~\ref{eq:vjepa_objective}), strictly avoiding the computation of a sensory emission likelihood $p(x_T \mid Z_T)$. While one \emph{could} optionally introduce a decoder $p_\psi(x_T \mid Z_T)$, VJEPA does not optimize it; the learning signal flows solely through latent prediction.

\subsection{Similarities: Latent-State Semantics, Dynamics, and Beliefs}

Both HMMs and VJEPA posit latent variables that summarize information needed for prediction.
\begin{itemize}
    \item \textbf{Latent State:} In an HMM, $s_t$ is the latent state; in VJEPA, $Z_t$ plays an analogous role as a predictive latent embedding.
    \item \textbf{Transitions:} The VJEPA predictive distribution $p_\phi(Z_T \mid Z_C,\xi_T)$ acts as a \emph{transition model} in latent space, mapping a ``current'' latent summary (context) to a distribution over a ``future'' latent representation.
    \item \textbf{Filtering:} If $q_{\theta'}(Z_T\mid x_T)$ has nontrivial variance, VJEPA supports belief-state style reasoning (means, covariances, sampling), similar to probabilistic filtering in HMMs.
\end{itemize}

\subsection{Differences}

\paragraph{(1) Observations are not modeled.}
HMMs are trained to maximize $\log p(x_{1:T})$ via the emission likelihood $p(x_t \mid s_t)$.
In contrast, VJEPA focuses exclusively on predictive structure in representation space. It learns a \emph{predictive latent process} without committing to a generative model of sensory observations, thereby avoiding the modeling of high-entropy noise.

\paragraph{(2) The ``inference model'' is defined by the encoder, not the generator.}
In an HMM, the posterior $p(s_t \mid x_{\le t})$ is derived by inverting the generative model Eq.~\eqref{eq:hmm_factorization}. In VJEPA, $q_{\theta'}(Z_T \mid x_T)$ is an amortized distribution defined by the target encoder (the ``teacher''). It serves as the ground truth for the predictor, rather than being an approximation of the predictor's posterior.

\paragraph{(3) The Markov property is optional.}
The HMM assumes a strict first-order Markov property $p(s_t \mid s_{<t}) = p(s_t \mid s_{t-1})$. VJEPA does not require a particular temporal factorization; it can predict multi-step futures directly ($Z_t \to Z_{t+\Delta}$) without iterating a 1-step transition, and without autoregressive observation factorization (see Section~\ref{sec:JEPA_as_dynamics}).

\paragraph{(4) Representation geometry is learned, not prescribed.}
Classical HMMs typically impose strong inductive biases on the latent space (e.g. discrete states or Gaussian dynamics). VJEPA learns the embedding geometry via neural encoders, making it closer in spirit to continuous Predictive State Representations (PSRs) than to classical HMM fitting.

\subsection{Sequential VJEPA as an ``HMM Without Emissions''}

Since an HMM is defined by the tuple (Transition, Emission), one can interpret \textit{sequential} VJEPA as learning a \emph{latent predictive process} (the Transition component) while deliberately omitting the Emission optimization. This design choice retains probabilistic semantics and uncertainty propagation while discarding the burden of reconstructing irrelevant observation details.

\section{Kalman-Filter-Like Interpretation of VJEPA}
\label{sec:kalman_interpretation}

Although VJEPA does not explicitly model observation likelihoods, its operation in latent space closely mirrors the prediction step of a Bayesian filter. This analogy becomes exact in the limit of linear-Gaussian dynamics, providing a rigorous grounding for VJEPA's uncertainty estimates and justifying its use in belief-space planning.

\subsection{The Linear-Gaussian Benchmark}

Consider a classical linear dynamical system with Gaussian noise. The latent state $s_t$ evolves according to:
\begin{equation}
\label{eq:linear_dynamics}
s_{t+1} = A s_t + B u_t + w_t, \quad w_t \sim \mathcal{N}(0, Q),
\end{equation}
where $A$ is the state transition matrix, $B$ is the control matrix, and $Q$ is the process noise covariance.

A Kalman Filter maintains the belief state as a Gaussian distribution $b_t(s_t) = \mathcal{N}(\mu_t, P_t)$. The \textbf{prediction step} (or time update) propagates this belief forward:
\begin{align}
\label{eq:kf_mean_update}
\mu_{t+1|t} &= A \mu_t + B u_t, \\
\label{eq:kf_cov_update}
P_{t+1|t} &= A P_t A^\top + Q.
\end{align}
Here, the mean is shifted deterministically by the dynamics, while the uncertainty (covariance) grows due to the additive process noise $Q$.

\subsection{VJEPA as Generalized Latent Prediction}

VJEPA generalizes this structure to non-linear, amortized inference in representation space. If we interpret the context embedding $Z_t$ as a sufficient statistic for the current belief history (i.e. $Z_t \cong \mu_t$), then the probabilistic predictor $p_\phi$ implements a generalized transition update.

Recall the VJEPA predictive distribution\footnote{Unlike Theorem.\ref{sec:map_optimal} (MAP/mean predictive control is optimal under quadratic costs), in which we assumed the Gaussian covariance $\Sigma$ of the predictive distribution does \emph{not} depend on $u_t$, i.e. $p_\phi(Z_{t+1}\mid Z_t,u_t)
=\mathcal{N} \big(Z_{t+1}\mid \mu_\phi(Z_t,u_t),\,\Sigma\big)$, here we make it a more general case.}:
\[
p_\phi(Z_{t+1} \mid Z_t, u_t) = \mathcal{N}\big(\mu_\phi(Z_t, u_t), \Sigma_\phi(Z_t, u_t)\big).
\]
Comparing this to the Kalman equations \eqref{eq:kf_mean_update} and \eqref{eq:kf_cov_update}:

\begin{enumerate}
    \item \textbf{Learned Transition ($\mu_\phi \approx A, B$):} The network $\mu_\phi(Z_t, u_t)$ approximates the non-linear flow of the system. Unlike the fixed matrix $A$, it can model state-dependent transitions and complex control interactions.
    
    \item \textbf{State-Dependent Uncertainty ($\Sigma_\phi \approx Q$):} The predicted covariance $\Sigma_\phi(Z_t, u_t)$ generalizes the process noise $Q$. Importantly, VJEPA allows this noise to be \emph{heteroscedastic} (state-dependent). For example, the model can predict high uncertainty (large $\Sigma_\phi$) in chaotic regions of the state space and low uncertainty (small $\Sigma_\phi$) in stable regions.
\end{enumerate}

\paragraph{Uncertainty Propagation.}
In the standard Kalman Filter, the predictive uncertainty $P_{t+1|t}$ depends on the previous uncertainty $P_t$ via $A P_t A^\top$. In standard JEPA, $Z_t$ is typically a point estimate ($P_t \to 0$). However, if we treat the input $Z_t$ as a sample from a belief distribution, the VJEPA predictor naturally propagates this uncertainty via the sampling mechanism discussed in Section~\ref{sec:vjepa} (Eq.~\ref{eq:VJEPA_multi_step_prediction}):
\[
P_{t+1} \approx \underbrace{\mathbb{E}[\Sigma_\phi(Z_t, u_t)]}_{\text{Process Noise } Q} + \underbrace{\mathrm{Var}(\mu_\phi(Z_t, u_t))}_{\text{Propagated Uncertainty } A P_t A^\top}.
\]
Thus, VJEPA implicitly captures the full Kalman prediction logic: aleatoric uncertainty is output by $\Sigma_\phi$, while epistemic/belief uncertainty is propagated by passing a distribution of $Z_t$ through the non-linear mean $\mu_\phi$.

\subsection{Amortized Correction}

A full Bayesian filter also requires a \textbf{correction step} (measurement update) to incorporate new observations $x_{t+1}$:
\[
\mu_{t+1} = \mu_{t+1|t} + K_t (x_{t+1} - H \mu_{t+1|t}).
\]
VJEPA handles this differently. Instead of an explicit update equation involving observation likelihoods and Kalman gains $K_t$, VJEPA \textbf{amortizes} the correction into the context encoder $f_\theta$. 
\[
Z_{t+1} = f_\theta(x_{\le t+1}).
\]
The encoder learns to map the updated history $x_{\le t+1}$ directly to the posterior latent state $Z_{t+1}$. The VJEPA training objective (minimizing the divergence between the \emph{predicted} belief $p_\phi(\cdot|Z_t)$ and the \emph{encoded} belief $q_{\theta'}(\cdot|x_{t+1})$) effectively forces the predictor to be consistent with the encoder's implicit Bayesian updates.

\section{A Particle-Filter Interpretation of VJEPA}
\label{sec:pf_vjepa}

VJEPA admits a natural \emph{particle filter} interpretation when used sequentially. This allows us to track multimodal beliefs about the latent state without restricting ourselves to Gaussian assumptions, using a method known as Sequential Importance Resampling (SIR). The complete algorithmic flow, illustrating how latent predictions are fused with surrogate measurements, is visualized in Fig.\ref{fig:pf_diagram}.

Let the latent belief at time $t$ be represented by a set of $K$ weighted particles $\{Z_t^{(k)}, w_t^{(k)}\}_{k=1}^K$ approximating the posterior distribution $b_t(Z_t) \approx p(Z_t \mid x_{1:t})$:
\[
b_t(Z_t) \approx \sum_{k=1}^K w_t^{(k)}\,\delta(Z_t - Z_t^{(k)}),
\qquad
\sum_{k=1}^K w_t^{(k)} = 1.
\]
Given an action $u_t$ (or target specification $\xi_{t+1}$), the filter proceeds in three main stages:

\paragraph{1. Prediction (Proposal).} We propagate each particle forward through the learned VJEPA dynamics. We employ a \emph{bootstrap filter} design, using the dynamics model itself as the importance sampling proposal distribution:
\[
Z_{t+1}^{(k)} \sim p_\phi(Z_{t+1}\mid Z_t^{(k)}, \xi_{t+1}),
\qquad k=1,\dots,K.
\]
This step draws samples from the predictive belief $p(Z_{t+1} \mid x_{1:t})$.

\paragraph{2. Update (Weighting).} To incorporate the new information from observation $x_{t+1}$, we update the importance weights. Since the proposal distribution is the transition prior, the unnormalized weight $\tilde w_{t+1}^{(k)}$ is simply proportional to the likelihood of the observation given the particle's state. VJEPA provides a representation-space inference distribution $q_{\theta'}(Z_{t+1}\mid x_{t+1})$ (Eq.~\ref{eq:vjepa_inference_target}), which serves as a \emph{surrogate measurement model}. We consider two cases for defining this likelihood:

\paragraph{Option A: Explicit Observation Model.}
If one instantiates the optional decoder $p_\psi(x_{t+1}\mid Z_{t+1})$ (Eq.~\ref{eq:vjepa_obs_model_optional}), the weight update is the standard observation likelihood:
\[
\tilde w_{t+1}^{(k)} \propto w_t^{(k)}\, p_\psi(x_{t+1}\mid Z_{t+1}^{(k)}).
\]

\paragraph{Option B: Latent-Space Update (Likelihood-Free).}
Without a decoder, we can derive a \emph{pseudo-likelihood} in latent space. By inverting the inference encoder using Bayes' rule, we have $p(x \mid Z) \propto p(Z \mid x) / p(Z)$. Treating the target encoder $q_{\theta'}$ as an approximation to the true posterior $p(Z \mid x)$, the weight update becomes:
\[
\tilde w_{t+1}^{(k)} \propto w_t^{(k)}\,
\frac{q_{\theta'}(Z_{t+1}^{(k)}\mid x_{t+1})}{p_{\text{ref}}(Z_{t+1}^{(k)})},
\]
where $p_{\text{ref}}(Z)$ is the fixed reference prior (e.g. $\mathcal{N}(0,I)$) used in the VJEPA KL regularization term (Eq.~\ref{eq:vjepa_objective}). Intuitively, this ratio measures how much more probable the particle $Z_{t+1}^{(k)}$ is according to the data-informed encoder compared to the uninformed prior.

\paragraph{3. Normalize and Resample.}
The weights are normalized so they sum to one: $w_{t+1}^{(k)} = \tilde w_{t+1}^{(k)} / \sum_{j=1}^K \tilde w_{t+1}^{(j)}$. To prevent particle degeneracy (where one particle accumulates all the weight), a resampling step is performed, drawing $K$ new particles with replacement from the current set with probabilities proportional to their weights.

This interpretation establishes sequential VJEPA as a full \emph{belief-propagation mechanism} over latent predictive states, where $p_\phi$ provides the dynamics and $q_{\theta'}$ injects measurement information directly in the latent space.

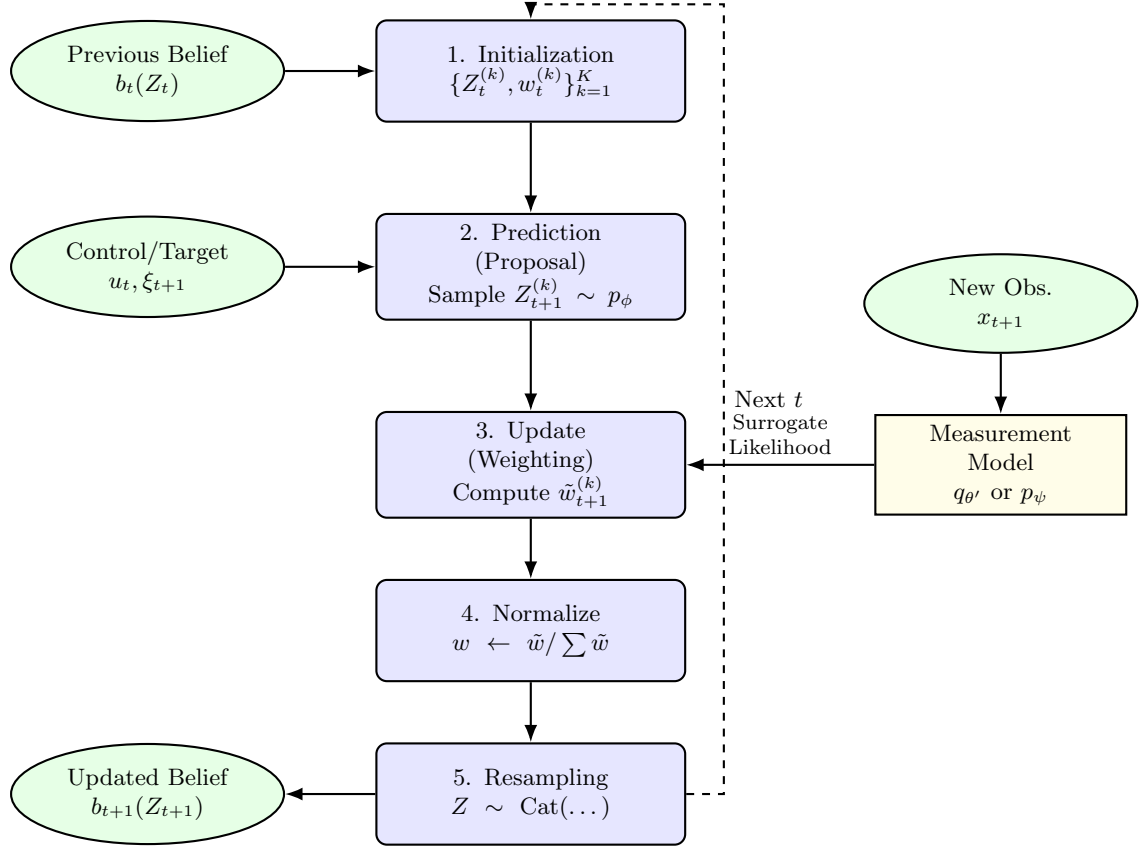
\begin{figure}[H]
\centering
\begin{tikzpicture}[
    auto, 
    >=Latex, 
    thick, 
    font=\footnotesize,
    node distance=1.0cm and 1.5cm, 
    block/.style ={draw, rectangle, fill=blue!10, text width=10em, text centered, rounded corners, minimum height=3.5em},
    line/.style ={draw, ->, thick},
    cloud/.style ={draw, ellipse, fill=green!10, minimum height=2.5em, text width=6em, text centered},
    model/.style ={draw, rectangle, fill=yellow!10, text width=8em, text centered, minimum height=3em}
]

    \node [block] (init) {1. Initialization \\ $\{Z_t^{(k)}, w_t^{(k)}\}_{k=1}^K$};
    \node [block, below=1.2cm of init] (predict) {2. Prediction \\ (Proposal) \\ Sample $Z_{t+1}^{(k)} \sim p_\phi$};
    \node [block, below=1.2cm of predict] (update) {3. Update \\ (Weighting) \\ Compute $\tilde{w}_{t+1}^{(k)}$};
    \node [block, below=0.8cm of update] (normalize) {4. Normalize \\ $w \leftarrow \tilde{w} / \sum \tilde{w}$};
    \node [block, below=0.8cm of normalize] (resample) {5. Resampling \\ $Z \sim \text{Cat}(\dots)$};

    \node [cloud, left=1.2cm of init] (prev_belief) {Previous Belief \\ $b_t(Z_t)$};
    \node [cloud, left=1.2cm of predict] (control) {Control/Target \\ $u_t, \xi_{t+1}$};
    \node [cloud, left=1.2cm of resample] (next_belief) {Updated Belief \\ $b_{t+1}(Z_{t+1})$};

    \node [model, right=2.5cm of update] (measurement) {Measurement \\ Model \\ $q_{\theta'}$ or $p_\psi$};
    \node [cloud, above=0.8cm of measurement] (obs) {New Obs. \\ $x_{t+1}$};

    \path [line] (prev_belief) -- (init);
    \path [line] (control) -- (predict);
    \path [line] (resample) -- (next_belief);

    \path [line] (init) -- (predict);
    \path [line] (predict) -- (update);
    \path [line] (update) -- (normalize);
    \path [line] (normalize) -- (resample);

    \path [line] (obs) -- (measurement);
    \path [line] (measurement) -- node[midway, above, align=center, font=\scriptsize] {Surrogate \\ Likelihood} (update);

    \draw [line, dashed] (resample.east) -- ++(0.5,0) |- node[near start, right] {Next $t$} ($(init.north east) + (0.5, 0.2)$) -- ($(init.north) + (0, 0.2)$) -- (init.north);

\end{tikzpicture}
\caption{Schematic overview of a Particle Filter (Sequential Monte Carlo) algorithm applied within the VJEPA framework. The process begins with a set of weighted particles representing the current belief. The \textbf{Prediction} step uses the learned probabilistic dynamics model $p_\phi$ as a proposal distribution to propagate particles forward. The \textbf{Update} step re-weights these particles based on new observation information, which is incorporated via a surrogate likelihood derived from the target encoder $q_{\theta'}$ (or an optional decoder $p_\psi$). Finally, \textbf{Resampling} is performed to avoid particle degeneracy, resulting in an updated particle set representing the posterior belief for the next time step.}
\label{fig:pf_diagram}
\end{figure}

\section{Detailed Proof of Theorem~\ref{thm:collapse_avoidance}}
\label{app:collapse_avoidance_proof}

We provide a detailed proof that the VJEPA objective admits no collapsed global optimum under the stated assumptions in Theorem~\ref{thm:collapse_avoidance}.
\vspace{0.5cm}

\noindent\fbox{%
    \begin{minipage}{\dimexpr\linewidth-2\fboxsep-2\fboxrule}
        \textbf{Theorem~\ref{thm:collapse_avoidance}: No Collapsed Global Optimum under Target Diversity}
        \vspace{0.5em}

        Consider the VJEPA objective
        \[
        \mathcal{L}_{\text{VJEPA}}
        =
        \mathbb{E}_{x}
        \mathbb{E}_{Z_T \sim q_{\theta'}(\cdot\mid x_T)}
        \big[-\log p_\phi(Z_T \mid Z_C, \xi_T)\big]
        +
        \beta\,\mathbb{E}_x
        \mathrm{KL}\!\left(q_{\theta'}(Z_T\mid x_T)\,\|\,p(Z_T)\right).
        \]
        Assume:
        \begin{enumerate}[label=(\roman*)]
            \item (\textbf{Target diversity}) There exist $x_T,x_T'$ such that
            $q_{\theta'}(\cdot\mid x_T)\neq q_{\theta'}(\cdot\mid x_T')$.
            \item (\textbf{Nontrivial conditioning}) The predictive family
            $\{p_\phi(\cdot\mid Z_C,\xi_T)\}$ can represent different distributions for different
            $Z_C$ (for fixed $\xi_T$).
        \end{enumerate}
        Then no global minimizer of $\mathcal{L}_{\text{VJEPA}}$ satisfies
        $f_\theta(x_C)\equiv c$ for all $x_C$.
    \end{minipage}%
}

\begin{proof} The proof follows three steps: we first analyze the structural consequences of a collapsed context encoder, then derive the minimal achievable prediction loss in this regime, and finally demonstrate that non-collapsed representations achieve strictly lower VJEPA objective values.

\paragraph{Step 1: Consequence of a collapsed context encoder.}
Assume, for contradiction, that the context encoder collapses:
\[
f_\theta(x_C) = c \qquad \forall x_C.
\]
Then the predictive model cannot depend on the context content and reduces to
\[
p_\phi(Z_T \mid Z_C,\xi_T) = p_\phi(Z_T \mid c,\xi_T),
\]
i.e.\ an \emph{unconditional} distribution for each $\xi_T$. Under this restriction, the first term of the VJEPA objective becomes
\begin{equation}
\label{eq:collapsed_loss}
\mathcal{L}_{\text{pred}}^{\text{coll}}
=
\mathbb{E}_{\xi_T}
\mathbb{E}_{x_T\mid \xi_T}
\mathbb{E}_{Z_T\sim q_{\theta'}(\cdot\mid x_T)}
\big[-\log p_\phi(Z_T\mid c,\xi_T)\big].
\end{equation}

\paragraph{Step 2: Optimal unconditional predictor.}
For a fixed $\xi_T$, define the \emph{aggregated target distribution}
\[
\bar q(Z_T\mid \xi_T)
:=
\mathbb{E}_{x_T\mid \xi_T}\big[q_{\theta'}(Z_T\mid x_T)\big].
\]
It is a standard result that the distribution minimizing cross-entropy with respect to $\bar q$ is $\bar q$ itself. Hence,
\[
\inf_{p(\cdot\mid c,\xi_T)}
\mathbb{E}_{x_T\mid\xi_T}
\mathbb{E}_{Z_T\sim q_{\theta'}(\cdot\mid x_T)}
[-\log p(Z_T)]
=
H(\bar q(\cdot\mid \xi_T)).
\]
Therefore, the best achievable collapsed prediction loss equals
\begin{equation}
\label{eq:collapsed_entropy}
\mathcal{L}_{\text{pred}}^{\text{coll}}
=
\mathbb{E}_{\xi_T}\, H(Z_T\mid \xi_T).
\end{equation}

\paragraph{Step 3: Decomposition via conditional mutual information.}
By standard entropy identities (see Appendix.\ref{app:info_theory_identities}),
\[
H(Z_T\mid \xi_T)
=
H(Z_T\mid X_T,\xi_T)
+
I(Z_T;X_T\mid \xi_T).
\]
Assumption (i) (\emph{target diversity}) implies that
\[
I(Z_T;X_T\mid \xi_T) > 0
\]
for at least one $\xi_T$ with nonzero probability mass. Hence the collapsed predictor incurs a strictly positive irreducible excess loss.

\paragraph{Step 4: Advantage of non-collapsed representations.}
Consider any non-collapsed encoder producing distinct $Z_C$ values correlated with $x_T$. By Assumption (ii), the predictive family can represent different $p_\phi(Z_T\mid Z_C,\xi_T)$ and therefore approximate $q_{\theta'}(Z_T\mid x_T)$ conditionally. In the realizable limit,
\[
\mathcal{L}_{\text{pred}}^{\text{non-coll}}
=
\mathbb{E}_{\xi_T}
\mathbb{E}_{x_T\mid \xi_T}
H(Z_T\mid X_T=x_T,\xi_T)
<
\mathcal{L}_{\text{pred}}^{\text{coll}}.
\]

\paragraph{Step 5: Role of the KL regularizer.}
The KL term
\[
\mathrm{KL}(q_{\theta'}(Z_T\mid x_T)\,\|\,p(Z_T))
\]
is independent of $Z_C$ and therefore does not eliminate the gap identified above.
Its role is to prevent pathological collapse of the target encoder, not to create
the strict separation between collapsed and non-collapsed optima.

\paragraph{Conclusion}
A collapsed context encoder forces the predictive model to fit a single unconditional
distribution to multiple distinct target distributions, incurring an irreducible
mutual-information gap. Since a non-collapsed solution can achieve strictly lower loss,
no collapsed representation can be globally optimal.

\paragraph{Interpretation.}
Collapse is avoided because predictive uncertainty depends on the target input.
VJEPA therefore prevents collapse through \emph{information mismatch}, not architectural heuristics.
Theorem~\ref{thm:collapse_avoidance} formalizes the intuitive point that if target embeddings vary with the target input, then a collapsed context representation forces the predictor to fit a single unconditional distribution to multiple distinct targets, incurring an irreducible prediction loss. Consequently, any solution satisfying the collapsed mode hypothesis cannot be globally optimal for the VJEPA objective. Under the stated assumptions, preventing collapse is therefore an intrinsic property of the objective itself, rather than a consequence of architectural heuristics or training asymmetries.

\end{proof}

\section{Detailed Proof of Theorem~\ref{thm:map_optimal}}
\label{app:map_optimal_proof}

We restate Theorem~\ref{thm:map_optimal} in Section.\ref{sec:map_optimal} for convenience. 
\vspace{0.5cm}

\noindent\fbox{%
    \begin{minipage}{\dimexpr\linewidth-2\fboxsep-2\fboxrule}
        \textbf{Theorem \ref{thm:map_optimal}: Optimality of MAP (Mean) Control under Quadratic Costs}
        \vspace{0.5em}

        Consider a one-step control problem with latent predictive model
        $p_\phi(Z_{t+1}\mid Z_t,u_t)$ and stage cost $c(Z_{t+1},u_t)$.
        Assume that for each $(Z_t,u_t)$,
        \[
        p_\phi(Z_{t+1}\mid Z_t,u_t) = \mathcal{N} \big(Z_{t+1}\mid \mu_\phi(Z_t,u_t),\,\Sigma\big),
        \]
        where the covariance $\Sigma$ does \emph{not} depend on $u_t$.
        Let the stage cost be quadratic in $Z_{t+1}$:
        \[
        c(Z_{t+1},u_t) = (Z_{t+1}-z^\star)^\top Q_c (Z_{t+1}-z^\star) + r(u_t),
        \]
        with $Q_c \succeq 0$ and arbitrary $r(\cdot)$.
        Then the action minimizing expected cost,
        \[
        u_t^\star \in \arg\min_{u_t}  \mathbb{E} \left[c(Z_{t+1},u_t)\mid Z_t,u_t\right],
        \]
        is equivalently obtained by minimizing the cost at the predictive mean:
        \[
        u_t^\star \in \arg\min_{u_t}  c(\mu_\phi(Z_t,u_t),u_t).
        \]
        Moreover, since a Gaussian has $\mu_\phi(Z_t,u_t)$ as its unique MAP point, this is
        equivalently \emph{MAP control}.
    \end{minipage}%
}

\begin{proof}
Fix a time index $t$ and a given latent state $Z_t$. For each candidate action $u_t$,
define the random variable
\[
Z \equiv Z_{t+1} \sim \mathcal{N}(\mu,\Sigma),
\qquad
\mu \equiv \mu_\phi(Z_t,u_t),
\]
where $\Sigma$ is constant with respect to $u_t$ by assumption.
We want to analyze this objective
\begin{equation} \label{eq:quadratic_cost_objective}
J(u_t)
:=
\mathbb{E} \left[c(Z,u_t)\mid Z_t,u_t\right]
=
\mathbb{E} \left[(Z-z^\star)^\top Q_c (Z-z^\star)\right] + r(u_t).
\end{equation}
in which the only nontrivial term is the expectation of the quadratic form.

\paragraph{Step 1: Expand the quadratic form.}
Let $d := Z - z^\star$. Then
\[
d^\top Q_c d
=
(Z-z^\star)^\top Q_c (Z-z^\star).
\]
Insert and subtract the mean $\mu$:
\[
Z - z^\star
=
(Z-\mu) + (\mu - z^\star).
\]
Therefore,
\[
(Z-z^\star)^\top Q_c (Z-z^\star)
=
\big((Z-\mu)+(\mu-z^\star)\big)^\top Q_c \big((Z-\mu)+(\mu-z^\star)\big).
\]
Expand the product into four terms:
\begin{align*}
&\big((Z-\mu)+(\mu-z^\star)\big)^\top Q_c \big((Z-\mu)+(\mu-z^\star)\big) \\
&\quad=
\boldsymbol{(Z-\mu)^\top Q_c (Z-\mu)}
+ (Z-\mu)^\top Q_c (\mu-z^\star)
+ (\mu-z^\star)^\top Q_c (Z-\mu)
+ \boldsymbol{(\mu-z^\star)^\top Q_c (\mu-z^\star)}.
\end{align*}

\paragraph{Step 2: Take expectations term-by-term.}
Take $\mathbb{E}[\cdot]$ under $Z\sim\mathcal{N}(\mu,\Sigma)$.

\emph{(i) Cross terms vanish.}
Because $\mathbb{E}[Z-\mu]=0$, we have
\[
\mathbb{E}\big[(Z-\mu)^\top Q_c (\mu-z^\star)\big]
=
\mathbb{E}[Z-\mu]^\top Q_c (\mu-z^\star)
=
0.
\]
Similarly,
\[
\mathbb{E}\big[(\mu-z^\star)^\top Q_c (Z-\mu)\big]
=
(\mu-z^\star)^\top Q_c \mathbb{E}[Z-\mu]
=
0.
\]

\emph{(ii) The constant term remains.}
Since $(\mu-z^\star)^\top Q_c (\mu-z^\star)$ is deterministic given $u_t$,
\[
\mathbb{E}\big[(\mu-z^\star)^\top Q_c (\mu-z^\star)\big]
=
(\mu-z^\star)^\top Q_c (\mu-z^\star).
\]

\emph{(iii) The centered quadratic term becomes a trace.}
Let $\varepsilon := Z-\mu$. Then $\mathbb{E}[\varepsilon]=0$ and
$\mathbb{E}[\varepsilon\varepsilon^\top]=\Sigma$.
We claim
\[
\mathbb{E}\big[\varepsilon^\top Q_c \varepsilon\big] = \mathrm{tr}(Q_c\Sigma).
\]
To see this, use the identity for any vector $\varepsilon$ and matrix $Q_c$:
\[
\varepsilon^\top Q_c \varepsilon
=
\mathrm{tr} \left(\varepsilon^\top Q_c \varepsilon\right)
=
\mathrm{tr} \left(Q_c \varepsilon\varepsilon^\top\right),
\]
where we used $\mathrm{tr}(a)=a$ for a scalar $a$, and cyclicity of trace \footnote{The cyclic property of the trace says that $\text{tr}(XY) = \text{tr}(YX)$; in our case $X=\varepsilon^T$ and $Y=Q_c\varepsilon$.}. Taking expectations,
\[
\mathbb{E}\big[\varepsilon^\top Q_c \varepsilon\big]
=
\mathbb{E} \left[\mathrm{tr} \left(Q_c \varepsilon\varepsilon^\top\right)\right]
=
\mathrm{tr} \left(Q_c \,\mathbb{E}[\varepsilon\varepsilon^\top]\right)
=
\mathrm{tr}(Q_c\Sigma),
\]
where linearity of expectation and linearity of trace justify exchanging $\mathbb{E}$ and $\mathrm{tr}$.

\paragraph{Step 3: Assemble the expected quadratic cost.}
Combining the above,
\[
\mathbb{E} \left[(Z-z^\star)^\top Q_c (Z-z^\star)\right]
=
(\mu-z^\star)^\top Q_c (\mu-z^\star) + \mathrm{tr}(Q_c\Sigma).
\]
Hence the full objective in Eq.~\eqref{eq:quadratic_cost_objective} becomes
\[
J(u_t)
=
(\mu_\phi(Z_t,u_t)-z^\star)^\top Q_c (\mu_\phi(Z_t,u_t)-z^\star)
+
\mathrm{tr}(Q_c\Sigma)
+
r(u_t).
\]

\paragraph{Step 4: Use the covariance-independence assumption to reduce the argmin.}
By assumption, $\Sigma$ does not depend on $u_t$, and therefore $\mathrm{tr}(Q_c\Sigma)$ is a constant with respect to $u_t$. Adding or subtracting a constant does not change the minimizer, so
\begin{align*}
\arg\min_{u_t} J(u_t)
&=
\arg\min_{u_t}
\left[
(\mu_\phi(Z_t,u_t)-z^\star)^\top Q_c (\mu_\phi(Z_t,u_t)-z^\star)
+ r(u_t)
\right] \\
&=
\arg\min_{u_t}  c(\mu_\phi(Z_t,u_t),u_t),
\end{align*}
which proves the mean-planning claim.

\paragraph{Step 5: Mean equals MAP for a Gaussian.}
For $Z\sim\mathcal{N}(\mu,\Sigma)$ with $\Sigma\succ 0$, the log-density is
\[
\log p(Z) = -\tfrac{1}{2}(Z-\mu)^\top\Sigma^{-1}(Z-\mu) + \text{const},
\]
which is strictly concave in $Z$ and uniquely maximized at $Z=\mu$.
Thus the MAP point equals the mean:
\[
\arg\max_Z p(Z) = \mu.
\]
Therefore minimizing cost at the predictive mean is equivalently minimizing cost at the predictive MAP point, i.e.\ MAP control.
\end{proof}

\begin{remark}[Assumptions matter]
If $\Sigma=\Sigma(u_t)$ depends on the action, then the term $\mathrm{tr}(Q_c\Sigma(u_t))$ is no longer constant and the optimal action generally depends on predictive uncertainty. Likewise, for non-quadratic costs or non-Gaussian/multimodal $p_\phi$, the expected cost typically cannot be reduced to evaluating a single point estimate.
\end{remark}

\begin{remark}[Post-action cost indexing]
We index the stage cost as $c(Z_{t+1},u_t)$ to emphasize that cost is incurred \emph{after} applying the action $u_t$ and observing its effect on the system. This convention is standard in stochastic MPC and belief-space control, where actions are optimized based on predicted future states rather than current ones
\cite{mesbah2016stochasticmpc,rawlings2017MPC}. 

Some reinforcement learning and dynamic programming formulations instead write the stage cost as $c(Z_t,u_t)$ \cite{sutton2018rl,puterman2014mdp}. The two conventions are equivalent under expectation for Markov dynamics, since
\[
\mathbb{E}\!\left[c(Z_{t+1},u_t)\mid Z_t,u_t\right]
\]
can be absorbed into a redefined cost function of $(Z_t,u_t)$. We adopt the post-action form to align with prediction-based control and MPC.
\end{remark}

\section{Proof of Proposition~\ref{prop:nuisance_invariance} (Invariance to Nuisance)}
\label{app:proof_nuisance}

We provide a formal derivation showing that generative objectives force the representation to encode nuisance information, whereas VJEPA objectives do not.

\paragraph{Setup.}
Let the observation at time $t$ be a tuple $x_t = (s_t, n_t)$, where:
\begin{itemize}
    \item $s_t$ is the \textbf{signal}: it contains information relevant for predicting the future target $Z_{t+\Delta}$.
    \item $n_t$ is the \textbf{nuisance}: it is independent of the future target, i.e. $I(n_t; Z_{t+\Delta} \mid s_t) = 0$.
\end{itemize}
Let $Z_t$ be the representation of $x_t$.

\subsection{Case 1: Generative / Reconstruction Objective}
Consider a generative world model (e.g. VAE or autoregressive model) that maximizes the log-likelihood of the observation $x_t$ given the representation $Z_t$:
\[
\mathcal{L}_{\text{gen}} = \mathbb{E}_{x_t} \left[ \log p(x_t \mid Z_t) \right].
\]
Using the decomposition $x_t = (s_t, n_t)$, the likelihood factorizes as (using chain rule of probability):
\[
\log p(s_t, n_t \mid Z_t) = \log p(s_t \mid Z_t) + \log p(n_t \mid s_t, Z_t).
\]
To maximize this objective (globally), the representation $Z_t$ must maximize both terms.
\begin{itemize}
    \item Maximizing $\log p(s_t \mid Z_t)$ requires $Z_t$ to encode the signal $s_t$.
    \item Maximizing $\log p(n_t \mid s_t, Z_t)$ requires $Z_t$ to encode the nuisance $n_t$ (so that the posterior entropy $H(n_t \mid Z_t)$ is minimized).
\end{itemize}
\textbf{Conclusion:} A solution $Z_t$ that discards $n_t$ (i.e. $Z_t$ depends only on $s_t$) is \emph{suboptimal} for the generative objective because it fails to reconstruct the nuisance component $n_t$. Thus, generative models \emph{force} the representation to be maximal (retaining $N$).

\subsection{Case 2: VJEPA Prediction Objective}
Consider the VJEPA objective, which maximizes the likelihood of the \emph{future target} $Z_{t+\Delta}$ given the current representation $Z_t$:
\[
\mathcal{L}_{\text{VJEPA}} = \mathbb{E} \left[ \log p(Z_{t+\Delta} \mid Z_t) \right].
\]
By the definition of nuisance variables, $Z_{t+\Delta}$ is conditionally independent of $n_t$ given $s_t$. Therefore, the true predictive distribution satisfies:
\[
p(Z_{t+\Delta} \mid s_t, n_t) = p(Z_{t+\Delta} \mid s_t).
\]
Now compare two candidate representations:
\begin{enumerate}
    \item \textbf{Maximal representation:} $Z_t^{\text{max}} = (s_t, n_t)$.
    \item \textbf{Minimal representation:} $Z_t^{\text{min}} = s_t$ (discards $n_t$).
\end{enumerate}
Substituting these into the objective:
\[
\mathcal{L}_{\text{VJEPA}}(Z_t^{\text{max}})
= \mathbb{E} [\log p(Z_{t+\Delta} \mid s_t, n_t)]
= \mathbb{E} [\log p(Z_{t+\Delta} \mid s_t)]
= \mathcal{L}_{\text{VJEPA}}(Z_t^{\text{min}}).
\]
\textbf{Conclusion:} The minimal representation $Z_t^{\text{min}}$ achieves the exact same optimal loss value as the maximal representation. Since the objective does not penalize the absence of $n_t$, VJEPA \emph{admits} (is invariant to) minimal solutions that discard nuisance variability.

\section{Derivation of the BJEPA Predictive Factorization}
\label{app:poe_derivation}

In this section, we derive the relationship between the joint posterior $p(Z_T \mid Z_C, \eta)$ and the individual predictive factors used in the BJEPA architecture (Eq.~\ref{eq:bjepa_posterior}). We aim to express the posterior probability of the target latent state $Z_T$, given the historical context $Z_C$ and auxiliary information $\eta$, in terms of the individual conditional probabilities $p(Z_T \mid Z_C)$ and $p(Z_T \mid \eta)$.

\paragraph{Assumption: Conditional Independence.}
We assume that the history encoder and the auxiliary encoder provide independent information about the target state. Formally, we assume that $Z_C$ and $\eta$ are conditionally independent given the true target $Z_T$:
\begin{equation}
\label{eq:cond_indep}
p(Z_C, \eta \mid Z_T) = p(Z_C \mid Z_T) \, p(\eta \mid Z_T).
\end{equation}
This implies that if the true future state is known, the past history and the auxiliary constraints (e.g. goals) do not provide additional information about each other.

\paragraph{Derivation.}
Starting with Bayes' theorem for the posterior:
\[
p(Z_T \mid Z_C, \eta) = \frac{p(Z_C, \eta \mid Z_T)\, p(Z_T)}{p(Z_C, \eta)}.
\]
Substituting the conditional independence assumption from Eq.~\eqref{eq:cond_indep}:
\[
p(Z_T \mid Z_C, \eta) = \frac{p(Z_C \mid Z_T)\, p(\eta \mid Z_T)\, p(Z_T)}{p(Z_C, \eta)}.
\]
Next, we invert the individual likelihood terms using Bayes' theorem to express them in terms of the predictive distributions (the "experts"):
\[
p(Z_C \mid Z_T) = \frac{p(Z_T \mid Z_C)\, p(Z_C)}{p(Z_T)} 
\quad \text{and} \quad
p(\eta \mid Z_T) = \frac{p(Z_T \mid \eta)\, p(\eta)}{p(Z_T)}.
\]
Substituting these back into the posterior equation:
\[
p(Z_T \mid Z_C, \eta) = \frac{\left( \frac{p(Z_T \mid Z_C)\, p(Z_C)}{p(Z_T)} \right) \left( \frac{p(Z_T \mid \eta)\, p(\eta)}{p(Z_T)} \right) p(Z_T)}{p(Z_C, \eta)}.
\]
Simplifying the terms:
\[
p(Z_T \mid Z_C, \eta) = \frac{p(Z_T \mid Z_C)\, p(Z_T \mid \eta) \, p(Z_C) \, p(\eta)}{p(Z_T) \, p(Z_C, \eta)}.
\]
Since $Z_C$ and $\eta$ are observed inputs (constants with respect to the optimization of $Z_T$), the terms $p(Z_C)$, $p(\eta)$, and $p(Z_C, \eta)$ can be absorbed into a normalization constant $\mathcal{Z}$. The relation reduces to:
\[
p(Z_T \mid Z_C, \eta) \propto \frac{p(Z_T \mid Z_C)\, p(Z_T \mid \eta)}{p(Z_T)}.
\]
In the BJEPA formulation, we approximate the marginal prior over targets $p(Z_T)$ as uniform (or absorbed into the learned bias of the experts), yielding the standard Product of Experts form:
\[
p(Z_T \mid Z_C, \eta) \propto p(Z_T \mid Z_C)\, p(Z_T \mid \eta).
\]

\section{Toy Experiment Details}
\label{app:toy_experiment_details}

Here we provide the specific implementation details, hyperparameters, and architectural and training specifications used to generate the results in Section~\ref{sec:toy_experiment}.

\subsection{Environment and Data Generation}
The synthetic environment is a Linear-Gaussian system designed to simulate the ``Noisy TV'' scenario.
\begin{itemize}
    \item \text{Dimensions:} Observation dimension $D_x = 20$, Signal dimension $D_s = 4$, Distractor dimension $D_d = 4$.
    \item \textbf{Dynamics:}
    \begin{itemize}
        \item \emph{Signal:} $s_{t+1} = A_{\text{rot}} s_t + w_t$, where $A_{\text{rot}}$ is an orthogonal matrix generated via QR decomposition of a random Gaussian matrix. Process noise $w_t \sim \mathcal{N}(0, 0.1^2 I)$.
        \item \emph{Distractor:} $d_{t+1} = 0.9 d_t + v_t$. Process noise $v_t \sim \mathcal{N}(0, 0.3^2 I)$.
    \end{itemize}
    \item \textbf{Observations:} $x_t = C s_t + D (\sigma d_t) + \epsilon_t$. Matrices $C$ and $D$ are random matrices with column-wise unit norm normalization. Sensor noise $\epsilon_t \sim \mathcal{N}(0, 0.01^2 I)$.
    \item \textbf{Dataset:} We generate a continuous training trajectory of $T_{\text{train}}=6000$ steps and a testing trajectory of $T_{\text{test}}=2000$ steps.
    \item \textbf{Noise Scales:} The scaling factor $\sigma$ is swept from $0.0$ to $8.0$ in 9 equidistant steps.
\end{itemize}

\subsection{Model Architectures and Training Strategy}
\label{subsec:architectures_app}

We implement five models to map the current observation $x_t \in \mathbb{R}^{D_x}$ to a latent representation $z_t \in \mathbb{R}^{D_z}$. To impose a hard trade-off between signal and distractor encoding, we set the latent dimension equal to the true signal dimension ($D_z = D_s = 4$). All models utilize strictly linear transformations ($f(x) = Wx$) without bias terms (`bias=False`), matching the linear nature of the data generation process.

We compare two modeling paradigms: \textbf{Generative Reconstruction} (VAE, AR) operating in pixel space, and \textbf{Joint-Embedding Prediction} (JEPA, VJEPA, BJEPA) operating via latent-space autoregression.

\subsubsection{Generative Baselines (Pixel Reconstruction)}
These models must reconstruct the full observation $x$, including high-variance distractors.

\paragraph{Linear Variational Autoencoder (VAE).}
The VAE treats the problem as a static modeling task. It maps $x_t$ to a latent distribution $q(z_t|x_t)$ and reconstructs $x_t$.
\begin{itemize}
    \item \textit{Architecture:} A probabilistic encoder with two linear heads $\mu_\phi, \log\sigma^2_\phi$ mapping $D_x \to D_z$, and a deterministic decoder $\hat{x}_t = W_{dec} z_t$.
    \item \textit{Objective:} The Evidence Lower Bound (ELBO)\footnote{Under the assumption of a Gaussian likelihood with identity covariance, $p(x_t|z_t) = \mathcal{N}(\hat{x}_t, I)$, minimizing the negative log-likelihood is equivalent to minimizing the squared Euclidean distance $\| x_t - \hat{x}_t \|^2$ plus a constant.}:
    \[ \mathcal{L}_{\text{VAE}} = \| x_t - \hat{x}_t \|^2 + \beta D_{KL}(q(z_t|x_t) \| \mathcal{N}(0, I)). \]
    \item \textit{Hypothesis:} To minimize reconstruction error, the VAE must prioritize features with the highest variance. We predict it will encode distractors and ignore the signal.
\end{itemize}

\paragraph{Linear Pixel-Autoregressive (AR).}
The AR model is a predictive autoencoder operating in pixel space. It compresses $x_t$ into a bottleneck to predict the \emph{next} observation $x_{t+1}$.
\begin{itemize}
    \item \textit{Architecture:} A deterministic encoder $z_t = W_{enc} x_t$ mapping $D_x \to D_z$ and a predictive decoder $\hat{x}_{t+1} = W_{pred} z_t$ mapping $D_z \to D_x$.
    \item \textit{Objective:} Mean Squared Error in pixel space: $\mathcal{L}_{\text{AR}} = \| x_{t+1} - \hat{x}_{t+1} \|^2$.
    \item \textit{Hypothesis:} Since distractors are predictable ($d_{t+1} \approx 0.9 d_t$), the AR model must track them to minimize prediction error, wasting capacity on distractor texture.
\end{itemize}

\subsubsection{Joint-Embedding Architectures (Latent Autoregression)}
These models are time-indexed, order-1 Markovian in latent space, trained to predict $z_{t+1}$ without reconstructing pixels.

\paragraph{Linear JEPA (Deterministic).}
JEPA aligns the predicted next state with the encoded future state.
\begin{itemize}
    \item \textit{Architecture:} 
    \begin{enumerate}
        \item Context Encoder: $z_t = W_{enc} x_t$.
        \item Predictor: A linear layer $\hat{z}_{t+1} = W_{pred} z_t$ mapping $D_z \to D_z$.
        \item Target Encoder (EMA): Same architecture as Context Encoder.
    \end{enumerate}
    \item \textit{Objective:} VICReg loss, minimizing prediction error $\| \hat{z}_{t+1} - z'_{t+1} \|^2$ with variance/covariance regularization to prevent collapse.
\end{itemize}

\paragraph{Linear Probabilistic VJEPA.}
VJEPA extends JEPA by treating the prediction as a probabilistic distribution.
\begin{itemize}
    \item \textit{Architecture:} 
    \begin{enumerate}
        \item Deterministic Encoder: $z_t = W_{enc} x_t$.
        \item Probabilistic Predictor: Two linear heads outputting $\mu_{pred}, \log\sigma^2_{pred} = f_\phi(z_t)$ mapping $D_z \to D_z$.
        \item Target Encoder (EMA): Two heads outputting $\mu_{\theta'}, \log\sigma^2_{\theta'}$ mapping $D_x \to D_z$.
    \end{enumerate}
    \item \textit{Objective:} The regularized negative log-likelihood (Eq.~\ref{eq:vjepa_objective}):
    \[ \mathcal{L}_{\text{VJEPA}} = \mathbb{E}_{\tilde{z}_{t+1} \sim q_{\theta'}} [-\log p_\phi(\tilde{z}_{t+1} | z_t)] + \beta D_{KL}(q_{\theta'}(z'_{t+1}|x_{t+1}) \| \mathcal{N}(0, I)). \]
\end{itemize}

\paragraph{Linear Bayesian JEPA (BJEPA).}
BJEPA incorporates a static Bayesian prior to filter out unpredictable dynamics by fusing a \emph{Predictive Likelihood} (Dynamics Expert) with a \emph{Latent-Space Prior} (Static Expert).
\begin{itemize}
    \item \textit{Architecture:} Identical to VJEPA, but adds learnable static parameters $\mu_{prior}, \log\sigma^2_{prior}$ (initialized to zero) to define a preferred latent manifold.
    \item \textit{Inference:} BJEPA performs \emph{hard fusion} via a Product of Experts (PoE):
    \[ \mu_{post}, \log\sigma^2_{post} = \text{PoE}((\mu_{dyn}, \sigma^2_{dyn}), (\mu_{prior}, \sigma^2_{prior})). \]
    \item \textit{Objective:} $\mathcal{L}_{\text{BJEPA}} = \mathcal{L}_{\text{VJEPA}} + \gamma D_{KL}(p_{\text{like}}(z_{t+1}|z_t) \| p_{\text{prior}}(z_{t+1}))$.
\end{itemize}

\paragraph{Stochastic Training Strategy.}
In our implementation, we parameterize the context encoder $f_\theta(x_C)$ as a deterministic map, delegating uncertainty modeling entirely to the predictive transition $p_\phi(Z_T \mid Z_C)$ and the target inference $q_{\theta'}(Z_T \mid x_T)$. For both VJEPA and BJEPA, optimization of the variational objective (Eq.~\ref{eq:vjepa_objective} and Eq.~\ref{eq:bjepa_training_objective}) requires estimating an expectation over the target distribution $\mathbb{E}_{Z_T \sim q_{\theta'}}$. Following standard practice in stochastic gradient variational inference (SGVI), we approximate this expectation using a single Monte Carlo sample ($K=1$) drawn from the target encoder's distribution $q_{\theta'}(Z_T \mid x_T)$ via the reparameterization trick at each training step. While this single-sample estimate is noisy for any individual data point, the stochasticity of mini-batch gradient descent (SGD) smooths the objective over the course of training, providing an unbiased gradient estimator.

Note that the sampling operation is applied \emph{only} to the target encoder $q_{\theta'}$ to obtain a concrete regression target $Z_T$. We do \emph{not} sample from the predictor $p_\phi$. Instead, the predictor outputs the parameters of its belief distribution (e.g.\ $\mu_\phi, \sigma^2_\phi$), and we analytically evaluate the negative log-likelihood density (Gaussian in our implementation) of the sampled target $Z_T$ under these parameters. Mathematically, this evaluates how well the predictor's belief explains the specific realization drawn from the target encoder's belief.

\paragraph{Necessity of Target Sampling.}
Sampling $Z_T$ from the target distribution $q_{\theta'}$ (rather than simply using its mean $\mu_{\theta'}$) is strictly necessary to prevent \emph{variance collapse} in the predictor. Consider the Gaussian negative log-likelihood term minimized by the predictor:
\[
\mathcal{L}_{\text{pred}} \approx \frac{(Z_T - \mu_\phi)^2}{2\sigma^2_\phi} + \frac{1}{2}\log(\sigma^2_\phi).
\]
If we were to use the target mean as the training target (i.e.\ set $Z_T = \mu_{\theta'}$), the predictor could learn to output $\mu_\phi \approx \mu_{\theta'}$ perfectly, driving the squared error term to zero. This would encourage the variance $\sigma^2_\phi$ to collapse to zero (or $\log\sigma^2_\phi \to -\infty$) to minimize the second term, resulting in a deterministic, overconfident model. By sampling $Z_T \sim \mathcal{N}(\mu_{\theta'}, \sigma^2_{\theta'})$, the target becomes a ``moving target'' that fluctuates around the mean according to the target's uncertainty. The squared error term $(Z_T - \mu_\phi)^2$ remains non-zero, forcing the predictor to maintain a non-zero variance $\sigma^2_\phi$ to cover the distribution of potential futures.

\paragraph{Implementation and Limitations.}
In our experimental implementation, we parameterize both the target encoder $q_{\theta'}$ and the predictor $p_\phi$ as independent diagonal Gaussian distributions:
\[
q_{\theta'}(Z \mid x_T) = \mathcal{N}(Z; \mu_{\theta'}(x_T), \text{diag}(\sigma^2_{\theta'}(x_T))),
\quad
p_\phi(Z \mid Z_C) = \mathcal{N}(Z; \mu_\phi(Z_C), \text{diag}(\sigma^2_\phi(Z_C))).
\]
While this choice simplifies optimization and admits closed-form KL terms, the unimodal Gaussian assumption represents a limitation. In environments with complex, multi-modal bifurcations (e.g.\ distinct ``go left'' vs.\ ``go right'' futures), a single Gaussian may average the modes, resulting in a high-variance belief centered on an invalid state. Future work could address this by implementing more expressive heads, such as Gaussian Mixture Models or normalizing flows.

\subsubsection{Temporal Alignment in Training and Evaluation}
\label{subsubsec:temporal_alignment}

All models receive observation pairs $(x_t, x_{t+1})$ during training, where $x_t$ denotes the current observation and $x_{t+1}$ denotes the next-step observation. The temporal alignment differs between model families based on their predictive structure:

\paragraph{Training Alignment.}
\begin{itemize}
    \item \textit{VAE:} trained on single observations $x_t$ with reconstruction target $x_t$ (no temporal prediction).
    \item \textit{AR:} trained on pairs $(x_t, x_{t+1})$ with pixel-space prediction target $x_{t+1}$.
    \item \textit{JEPA/VJEPA/BJEPA:} trained on pairs $(x_t, x_{t+1})$ with latent-space prediction target $z_{t+1} = f_{\theta'}(x_{t+1})$.
\end{itemize}

\paragraph{Evaluation Alignment.}
As predictive models (AR, JEPA, VJEPA, BJEPA) are trained to predict future states, their latent representations should be evaluated against the \emph{future} ground-truth signal $s_{t+1}$, not the current signal $s_t$. Concretely:
\begin{itemize}
    \item \textit{VAE:} The encoder output $z_t = \mu_\phi(x_t)$ represents the current state and is evaluated against $s_t$.
    \item \textit{AR:} The encoder output $z_t = W_{enc} x_t$ is designed to predict $x_{t+1}$, hence is evaluated against $s_{t+1}$.
    \item \textit{JEPA:} The predictor output $\hat{z}_{t+1} = W_{pred} z_t$ explicitly predicts the next latent state and is evaluated against $s_{t+1}$.
    \item \textit{VJEPA:} The predictor mean $\mu_{pred}(z_t)$ represents the expected next latent state and is evaluated against $s_{t+1}$.
    \item \textit{BJEPA:} At inference, the hard-fused posterior mean $\mu_{post} = \text{PoE}(\mu_{dyn}, \mu_{prior})$ represents the constrained prediction of the next latent state and is evaluated against $s_{t+1}$.
\end{itemize}

This temporal alignment ensures that each model's representation is evaluated against the appropriate ground-truth signal consistent with its predictive objective.

\subsubsection{Linear Probe Evaluation Protocol}
\label{subsubsec:linear_probe}

To assess how well each model's latent representation captures the true underlying signal $s_t$ (while ignoring the distractor $d_t$), we employ a linear probe methodology. For each model, we extract latent representations and fit a linear regression to predict the ground-truth signal, then measure the coefficient of determination ($R^2$) on held-out test data.

\paragraph{Latent Extraction.}
The latent representation used for probing differs by model, reflecting their distinct architectures:
\begin{itemize}
    \item \textit{VAE:} $z_t^{\text{probe}} = \mu_\phi(x_t)$ --- the encoder mean, representing the current state.
    \item \textit{AR:} $z_t^{\text{probe}} = W_{enc} x_t$ --- the encoder output, which implicitly predicts the next state.
    \item \textit{JEPA:} $z_t^{\text{probe}} = W_{pred} (W_{enc} x_t) = \hat{z}_{t+1}$ --- the predictor output, explicitly predicting the next latent state.
    \item \textit{VJEPA:} $z_t^{\text{probe}} = \mu_{pred}(W_{enc} x_t)$ --- the predictor mean, representing the expected next latent state.
    \item \textit{BJEPA:} $z_t^{\text{probe}} = \mu_{post}$ --- the hard-fused posterior mean computed via Product of Experts:
    \[
    \mu_{post} = \frac{\sigma^2_{prior} \cdot \mu_{dyn} + \sigma^2_{dyn} \cdot \mu_{prior}}{\sigma^2_{dyn} + \sigma^2_{prior}},
    \]
    where $(\mu_{dyn}, \sigma^2_{dyn})$ are the dynamics predictor parameters and $(\mu_{prior}, \sigma^2_{prior})$ are the learned static prior parameters.
\end{itemize}

\paragraph{Probe Fitting and Evaluation.}
Given extracted latents $\{z_i^{\text{probe}}\}_{i=1}^{N_{\text{train}}}$ from the training set and corresponding ground-truth signals $\{s_i^{\text{target}}\}_{i=1}^{N_{\text{train}}}$, we fit a linear regression:
\[
\hat{s} = W_{\text{probe}} z^{\text{probe}} + b_{\text{probe}},
\]
where $(W_{\text{probe}}, b_{\text{probe}})$ are learned via ordinary least squares. The target signal $s^{\text{target}}$ is chosen according to the temporal alignment described above: $s_t$ for VAE, and $s_{t+1}$ for all predictive models (AR, JEPA, VJEPA, BJEPA).

We then evaluate on the test set by computing the $R^2$ score:
\[
R^2 = 1 - \frac{\sum_i \| s_i^{\text{target}} - \hat{s}_i \|^2}{\sum_i \| s_i^{\text{target}} - \bar{s} \|^2},
\]
where $\bar{s}$ is the mean of the test targets. An $R^2$ close to 1 indicates that the latent representation successfully captures the true signal, while an $R^2$ close to 0 indicates that the representation has failed to disentangle signal from distractor.

\paragraph{Interpretation.}
This evaluation protocol directly tests the core hypothesis: predictive models operating in latent space (JEPA, VJEPA, BJEPA) should learn representations that capture the predictable signal while discarding unpredictable distractors, whereas generative models (VAE, AR) operating in pixel space are forced to allocate capacity to high-variance distractors regardless of their task relevance. The use of predictor outputs (rather than encoder outputs) for JEPA-family models ensures we evaluate the representation that the model is actually trained to produce, i.e. the predicted future state.

\subsection{Optimization and Hyperparameters}
All models are trained using the \textit{Adam} optimizer \cite{kingma2017haha_dont_ask_me_who_is_Adam_my_name_is_Yong}. We use full-batch training (processing the continuous trajectory as a single batch) for stability in this linear setting.

\begin{table}[H]
\centering
\caption{Optimization settings and loss function hyperparameters.}
\label{tab:hyperparams}
\begin{tabular}{l c l}
\toprule
\textbf{Parameter} & \textbf{Value} & \textbf{Description} \\
\midrule
\multicolumn{3}{l}{\textit{Global Optimization}} \\
Random Seed & 111 & For reproducibility \\
Optimizer & Adam \cite{kingma2017haha_dont_ask_me_who_is_Adam_my_name_is_Yong} & Standard implementation \\
Learning Rate & $1 \times 10^{-3}$ & Constant throughout training \\
Training Steps & 6000 & Sufficient for convergence \\
\midrule
\multicolumn{3}{l}{\textit{JEPA Family Common Settings}} \\
EMA Decay ($\tau$) & 0.99 & Target encoder update rate \\
\midrule
\multicolumn{3}{l}{\textit{Model-Specific Coefficients}} \\
\textbf{JEPA} (VICReg Loss) & & \\
\quad Invariance Coeff & 25.0 & MSE term weight \\
\quad Variance Coeff & 25.0 & Hinge loss on standard deviation \\
\quad Covariance Coeff & 1.0 & Off-diagonal decorrelation weight \\
\textbf{VJEPA} & & \\
\quad $\beta$ (Eq.~\ref{eq:vjepa_objective}) & 0.01 & Target encoder KL regularization weight \\
\textbf{BJEPA} & & \\
\quad $\beta$ (Eq.~\ref{eq:vjepa_objective}) & 0.01 & Target encoder KL regularization weight \\
\quad $\gamma$ (Eq.~\ref{eq:bjepa_training_objective}) & 0.1 & \textbf{Structural Prior} KL weight \\
\bottomrule
\end{tabular}
\end{table}

\paragraph{Note on BJEPA Priors.} In the Python code implementation for the ``Noisy TV'' experiment, the Structural Prior used for the $\gamma$ term is \emph{static} (time-invariant) but \emph{learnable}. It is parameterized as a Gaussian with parameters $\mu_{\text{prior}}, \sigma^2_{\text{prior}}$ that are optimized alongside the model weights to find the stable manifold of the signal. The reference prior used for the $\beta$ term (target regularization) is a fixed unit Gaussian $\mathcal{N}(0, I)$.

\subsection{Results}

We measure performance using the \emph{coefficient of determination} ($R^2$) between the learned latent $Z_t$ and the true signal $s_t$ via a linear probe. The quantitative results \footnote{We do not compare training loss values across methods because their objectives define optimality in fundamentally different metric spaces. Generative models (VAE, AR) minimize pixel-level reconstruction error, whereas JEPA-based models minimize feature distance or maximize likelihood in an abstract latent space with arbitrary scaling. Consequently, the raw loss magnitudes are not comparable.} are summarized in Table.\ref{tab:toyExperimentResults} and visualized across all 9 noise scales in Fig.\ref{fig:results_grid}. Additionally, Fig.\ref{fig:reconstructions} provides a qualitative view of the signal reconstruction at low, medium, and high noise levels.

\paragraph{Failure of Generative Models.}
As shown in Table.\ref{tab:toyExperimentResults_full} and Fig.\ref{fig:results_grid}, the reconstructive baselines (VAE, AR) suffer catastrophic performance degradation as the distractor noise scale increases. At the highest noise scale ($\sigma=8.0$), VAE and AR signal recovery drops to $R^2 \approx 0.50$ and $0.58$ respectively, while their noise recovery remains high ($R^2 = 0.62$ and $R^2 = 0.45$). This confirms that likelihood-based objectives force the models to encode high-variance distractors, leading to representation collapse on the target task.

\paragraph{Robustness of Joint-Embedding Architectures.}
In contrast, all joint-embedding architectures (JEPA, VJEPA, BJEPA) demonstrate robustness to the ``Noisy TV'' distractor. They maintain high signal recovery ($R^2 > 0.85$) even at $\sigma=8.0$, effectively filtering out the distractor dynamics. In this specific trial, the deterministic JEPA achieved the highest final accuracy ($R^2=0.93$), though it exhibited higher variance in training stability compared to the probabilistic variants\footnote{While the deterministic JEPA achieved high performance in the reported run (Seed 111), we observed in auxiliary experiments with different random seeds that it is prone to occasional training collapse (sudden performance drops to $R^2 < 0.5$) at intermediate noise scales. We attribute this to the difficulty of optimizing a deterministic predictor on highly stochastic targets. The probabilistic methods (VJEPA and BJEPA) consistently exhibited greater training stability across seeds by accounting for aleatoric uncertainty via their variance outputs.}.

\paragraph{Stability and Optimization.}
The high-variance noise makes the target embeddings ($z_{t+1}$) jittery and unpredictable. As standard JEPA is deterministic, it struggles to average out this jitter, leading to potential training instability. Even in this successful run, JEPA showed a momentary instability at Noise Scale 3.0, dropping to $0.841$ (Test) while VJEPA/BJEPA remained at $> 0.90$. This collapse at Scale 3.0 likely indicates optimization difficulties where the variance regularization term fought against the predictive loss. This hints at the optimization difficulty deterministic models face when the noise magnitude creates complex gradients, even if the model recovered at Scale 4.0.

\begin{table}[H]
\centering
\footnotesize
\setlength{\tabcolsep}{5pt}
\renewcommand{\arraystretch}{1.05}
\begin{tabular}{@{} l l c c c @{}}
\toprule
Scale (SNR) & Model & $\uparrow$Signal $R^2$ (Tr/Te) & $\downarrow$Noise $R^2$ (Tr/Te) & $\downarrow$Time (Tr/Te) \\
\midrule
0.0 (inf dB)  & VAE   & \underline{\textbf{1.000}} / \underline{\textbf{1.000}} & NA / NA & 12.6s / 0.01s \\
0.0 (inf dB)  & AR    & 0.999 / 0.999      & NA / NA    & \underline{\textbf{6.4s}} / 0.01s \\
0.0 (inf dB)  & JEPA  & 0.947 / 0.930      & NA / NA    & 16.9s / 0.01s \\
0.0 (inf dB)  & VJEPA & 0.999 / 0.999      & NA / NA    & 13.9s / 0.01s \\
0.0 (inf dB)  & BJEPA & 0.987 / 0.981      & NA / NA    & 23.5s / 0.01s \\
\midrule
1.0 (15.8 dB) & VAE   & 0.978 / 0.973      & 0.065 / 0.002    & 12.6s / 0.01s \\
1.0 (15.8 dB) & AR    & 0.983 / 0.980      & 0.041 / -0.007   & \underline{\textbf{6.6s}} / 0.01s \\
1.0 (15.8 dB) & JEPA  & 0.947 / 0.930      & 0.230 / 0.178    & 16.9s / 0.01s \\
1.0 (15.8 dB) & VJEPA & \underline{\textbf{0.999}} / \underline{\textbf{0.998}} & \underline{\textbf{0.006}} / \underline{\textbf{-0.012}} & 13.8s / 0.01s \\
1.0 (15.8 dB) & BJEPA & 0.975 / 0.976      & 0.029 / -0.003   & 23.5s / 0.01s \\
\midrule
2.0 (9.8 dB)  & VAE   & 0.903 / 0.853      & 0.283 / 0.194    & 12.8s / 0.01s \\
2.0 (9.8 dB)  & AR    & 0.923 / 0.889      & 0.179 / 0.094    & \underline{\textbf{6.4s}} / 0.01s \\
2.0 (9.8 dB)  & JEPA  & 0.947 / 0.930      & 0.213 / 0.172    & 17.0s / 0.01s \\
2.0 (9.8 dB)  & VJEPA & \underline{\textbf{0.996}} / \underline{\textbf{0.995}} & \underline{\textbf{0.020}} / -0.012 & 14.2s / 0.01s \\
2.0 (9.8 dB)  & BJEPA & 0.966 / 0.967      & 0.042 / \underline{\textbf{-0.023}} & 23.3s / 0.01s \\
\midrule
3.0 (6.3 dB)  & VAE   & 0.852 / 0.770      & 0.458 / 0.393    & 12.5s / 0.01s \\
3.0 (6.3 dB)  & AR    & 0.870 / 0.801      & 0.328 / 0.256    & \underline{\textbf{6.5s}} / 0.01s \\
3.0 (6.3 dB)  & JEPA  & 0.916 / 0.841      & \underline{\textbf{0.000}} / \underline{\textbf{-0.003}}   & 16.8s / 0.01s \\
3.0 (6.3 dB)  & VJEPA & \underline{\textbf{0.947}} / \underline{\textbf{0.930}}      & 0.224 / 0.176    & 13.5s / 0.01s \\
3.0 (6.3 dB)  & BJEPA & 0.946 / 0.919      & 0.106 / 0.078    & 23.0s / 0.01s \\
\midrule
4.0 (3.8 dB)  & VAE   & 0.822 / 0.730      & 0.512 / 0.458    & 12.5s / 0.01s \\
4.0 (3.8 dB)  & AR    & 0.839 / 0.756      & 0.394 / 0.338    & \underline{\textbf{6.5s}} / 0.01s \\
4.0 (3.8 dB)  & JEPA  & \underline{\textbf{0.999}} / \underline{\textbf{0.999}}      & \underline{\textbf{0.004}} / \underline{\textbf{-0.010}}   & 16.6s / 0.02s \\
4.0 (3.8 dB)  & VJEPA & 0.994 / 0.993      & 0.025 / -0.007   & 13.8s / 0.01s \\
4.0 (3.8 dB)  & BJEPA & 0.920 / 0.899      & 0.213 / 0.156    & 23.2s / 0.01s \\
\midrule
5.0 (1.8 dB)  & VAE   & 0.777 / 0.681      & 0.543 / 0.492    & 12.5s / 0.01s \\
5.0 (1.8 dB)  & AR    & 0.806 / 0.717      & 0.418 / 0.364    & \underline{\textbf{6.4s}} / 0.01s \\
5.0 (1.8 dB)  & JEPA  & \underline{\textbf{0.947}} / \underline{\textbf{0.930}}      & \underline{\textbf{0.200}} / \underline{\textbf{0.183}}    & 16.5s / 0.01s \\
5.0 (1.8 dB)  & VJEPA & 0.945 / 0.928      & 0.227 / 0.190    & 14.1s / 0.01s \\
5.0 (1.8 dB)  & BJEPA & 0.884 / 0.841      & 0.317 / 0.240    & 22.3s / 0.01s \\
\midrule
6.0 (0.3 dB)  & VAE   & 0.729 / 0.627      & 0.577 / 0.531    & 12.3s / 0.01s \\
6.0 (0.3 dB)  & AR    & 0.770 / 0.677      & 0.437 / 0.386    & \underline{\textbf{7.1s}} / 0.01s \\
6.0 (0.3 dB)  & JEPA  & \underline{\textbf{0.999}} / \underline{\textbf{0.999}}      & \underline{\textbf{0.007}} / \underline{\textbf{-0.011}}   & 16.0s / 0.01s \\
6.0 (0.3 dB)  & VJEPA & 0.947 / 0.930      & 0.197 / 0.202    & 13.7s / 0.01s \\
6.0 (0.3 dB)  & BJEPA & 0.935 / 0.914      & 0.205 / 0.171    & 22.8s / 0.01s \\
\midrule
7.0 (-1.1 dB) & VAE   & 0.675 / 0.566      & 0.615 / 0.574    & 12.7s / 0.01s \\
7.0 (-1.1 dB) & AR    & 0.729 / 0.631      & 0.462 / 0.416    & \underline{\textbf{7.1s}} / 0.01s \\
7.0 (-1.1 dB) & JEPA  & \underline{\textbf{0.946}} / \underline{\textbf{0.929}}      & \underline{\textbf{0.209}} / 0.191    & 16.1s / 0.01s \\
7.0 (-1.1 dB) & VJEPA & 0.939 / 0.918      & 0.264 / \underline{\textbf{0.182}}    & 13.4s / 0.01s \\
7.0 (-1.1 dB) & BJEPA & 0.894 / 0.851      & 0.291 / 0.237    & 23.1s / 0.01s \\
\midrule
8.0 (-2.2 dB) & VAE   & 0.613 / 0.499      & 0.656 / 0.620    & 12.3s / 0.01s \\
8.0 (-2.2 dB) & AR    & 0.680 / 0.578      & 0.491 / 0.449    & \underline{\textbf{7.1s}} / 0.01s \\
8.0 (-2.2 dB) & JEPA  & \underline{\textbf{0.947}} / \underline{\textbf{0.930}} & \underline{\textbf{0.226}} / \underline{\textbf{0.183}}    & 16.1s / 0.01s \\
8.0 (-2.2 dB) & VJEPA & 0.905 / 0.870      & 0.299 / 0.251    & 13.4s / 0.01s \\
8.0 (-2.2 dB) & BJEPA & 0.896 / 0.841      & 0.292 / 0.238    & 23.0s / 0.01s \\
\bottomrule
\end{tabular}
\caption{Performance metrics ($R^2$) across models at all noise scales. $Tr$ denotes training set, $Te$ denotes test set. Note the robustness of JEPA family models at high noise levels (e.g. Scale 8.0) compared to generative baselines.}
\label{tab:toyExperimentResults_full}
\end{table}

\subsection{Computational Resources}
All experiments were conducted using the standard CPU runtime provided by Google Colab \cite{Edwards2024Colab}. The key hardware specifications are summarized below:
\begin{itemize}
    \item \textbf{Platform:} Google Colab (Linux x86\_64)
    \item \textbf{CPU:} Intel(R) Xeon(R) CPU @ 2.20GHz
    \item \textbf{CPU Topology:} 2 Logical CPUs (1 Physical Core, 2 Threads per core)
    \item \textbf{Memory (RAM):} 13.61 GB
    \item \textbf{Disk Space:} 242.49 GB
    \item \textbf{Instruction Set:} x86\_64 with AVX2 and FMA support
\end{itemize}
Given the lightweight nature of the linear toy experiment (latent dimension $D_z=4$), training was performed entirely on the CPU, with individual model runs completing in approximately 10-35 seconds (see Table~\ref{tab:toyExperimentResults_full}).

\subsection{Code Availability}
The complete source code and data generation procedures used for the experiments in this paper are available at: \url{https://github.com/YongchaoHuang/VJEPA}.

\end{document}